\documentclass[twoside]{article}
\usepackage[accepted]{aistats2018}
\usepackage{times}
\usepackage{helvet}
\usepackage{courier}
\usepackage{graphicx}
\usepackage[T1]{fontenc}
 \usepackage[english]{babel}
\usepackage{algorithm}
\usepackage{algorithmic}
\usepackage{amsmath}
\usepackage{amssymb}
\usepackage{url}
\usepackage{color}
\usepackage{booktabs}
\usepackage{wrapfig}
\usepackage{lipsum}
\usepackage{pdfpages}
\usepackage{threeparttable}
\usepackage{array}
\usepackage{multirow}

\newtheorem{proposition}{Proposition}

\newtheorem{proof}{Proof}
\newtheorem{lemma}{Lemma}

\usepackage[accepted]{aistats2018}
\begin{document}

% If your paper is accepted and the title of your paper is very long,
% the style will print as headings an error message. Use the following
% command to supply a shorter title of your paper so that it can be
% used as headings.
%
%\runningtitle{I use this title instead because the last one was very long}

% If your paper is accepted and the number of authors is large, the
% style will print as headings an error message. Use the following
% command to supply a shorter version of the authors names so that
% they can be used as headings (for example, use only the surnames)
%
%\runningauthor{Surname 1, Surname 2, Surname 3, ...., Surname n}

\twocolumn[

\aistatstitle{The Binary Space Partitioning-Tree Process}

\aistatsauthor{ Xuhui Fan \And Bin Li \And  Scott A. Sisson }

\aistatsaddress{School of Mathematics and Statistics\\University of New South Wales\\xuhui.fan@unsw.edu.au \And  School of Computer Science\\Fudan University\\libin@fudan.edu.cn \And School of Mathematics and Statistics\\University of New South Wales\\scott.sisson@unsw.edu.au} ]

\begin{abstract}%   <- trailing '%' for backward compatibility of .sty file
    The Mondrian process represents an elegant and powerful approach for space partition modelling. However, as it restricts the partitions to be axis-aligned, its modelling flexibility is limited. In this work, we propose a self-consistent Binary Space Partitioning (BSP)-Tree process to generalize the Mondrian process. The BSP-Tree process is an almost surely right continuous Markov jump process that allows uniformly distributed oblique cuts in a two-dimensional convex polygon. The BSP-Tree process can also be extended using a non-uniform probability measure to generate direction differentiated cuts. The process is also self-consistent, maintaining distributional invariance under a restricted subdomain. We use Conditional-Sequential Monte Carlo for inference using the tree structure as the high-dimensional variable. The BSP-Tree process's performance on synthetic data partitioning and relational modelling demonstrates clear inferential improvements over the standard Mondrian process and other related methods.
\end{abstract}

%==========================================================

\section{Introduction}

%==========================================================

    In machine learning, some tasks such as constructing decision trees or relational modelling may be interpreted as a space partitioning strategy to identify regular regions in a product space. Models may then be fitted to the data in each regular ``block'', whereby the data within each block will exhibit certain types of homogeneity. While applications range from relational 
    modeling~\cite{kemp2006learning,airoldi2009mixed}, community detection~\cite{nowicki2001estimation,karrer2011stochastic}, collaborative filtering~\cite{porteous2008multi, Li_transfer_2009}, and random 
    forests~\cite{LakRoyTeh2014a}, most of the work only focuses on regular grids (rectangular blocks). While some recent work~\cite{roy2009mondrian, xuhui2016OstomachionProcess} has introduced more flexibility in the partitioning 
    strategy, there are still substantial limitations in the available modelling and inferential capabilities. 
    
    An inability to capture complex dependencies between different dimensions is one such limitation. Axis-aligned cuts, as extended from rectangular blocks, consider each division on one dimension only. This is an over-simplified and invalid assumption in many scenarios. For example, in decision tree classification,  classifying data through a combination of features (dimensions in our case) is usually more efficient than a recursive single features division. Similarly, in relational data modelling, where the nodes (individuals) correspond to the dimension and the blocks correspond to the communities the nodes belong to. An individual may have asymmetric relations in the communities (e.g., a football match organizer would pay more attention to the community of football while a random participant might be less focused). While the recently proposed Ostomachion Process (OP)~\cite{xuhui2016OstomachionProcess} also tries to allow oblique cuts, several important properties (e.g. self-consistency, uniform generalization) are missing, thereby limiting its appeal.

%     Space partitioning presents as an interesting while still meaningful approach in the current data modelling strategies. Some applications can be widely seen in areas such as . By tailoring the product space into regular regions, the partition model aims at using these regular ``blocks'' to fit the data, such that the data within each block would exhibit certain types of homogeneity. As one can choose an arbitrarily fine resolution of partition, the data can be fitted reasonably well.
% 
%     Among the various approaches, the Mondrian Process (MP) \cite{roy2009mondrian}\cite{roy2011thesis}\cite{LakRoyTeh2014a} presents to be the typical model. Generally, MP uses either horizontal or vertical cut to partition the space, where these cuts are organized in a tree structure. Many applications have been involved with this strategy, including the relational modeling~\cite{roy2007learning}, Decision tree problem~\cite{LakRoyTeh2014a,LakOryTeh2015ParticleGibbs}, Bayesian co-clustering problem~\cite{wang2011nonparametric}, the musical signal processing problem~\cite{mhmm2014nakano} and the ecological network reconstruction problem~\cite{aderhold2013reconstructing}.
% 
%     However, this approach still contains some issues in preventing its wider applications: (1), the axis-aligned cut is insufficient to give flexible partitions over the space; (2), there is difficulty in incorporating some prior information, especially about the direction of the cut. 

    To systematically address these issues, we propose a Binary Space Partitioning (BSP)-Tree process to hierarchically partition the space. The BSP-Tree process is an almost surely right continuous Markov jump process in a budget line. Except for some isolated points in the line, any realization of the BSP-Tree process maintains a constant partition between these consecutive points. Instead of axis-aligned cuts, the partition is formed by a series of consecutive oblique cuts. In this sense, the BSP-Tree process can simultaneously capture more complex partition structures with inter-dimensional dependence. The proposed cut described by the generative process can be proved to be uniformly distributed and moreover, the measure over the cut lines is fixed to a scaled sum of the blocks' perimeters. 
    
    Further, a particular form of non-uniformly distributed oblique cuts can be obtained by imposing weight functions for the cut directions. This variant can be well suited to cases where the directions of the cut are differentially favored. This variant might be particularly suitable for certain modelling scenarios. For example, while ensuring the existence of oblique cuts, axis-aligned cuts would be favored more in social networks since some communities would need to have rectangular blocks). All variants of the BSP-Tree process can be proved to be self-consistent, which ensures distributional invariance while restricting the process from a larger domain to a smaller one.
    
    The partition of the BSP-Tree process is inferred through the Conditional-Sequential Monte Carlo~(C-SMC) sampler~\cite{LakOryTeh2013TopDownPF,LakOryTeh2015ParticleGibbs}. In particular, we use a tree structure for the blocks to mimic the high-dimensional variables in the C-SMC sampler, where each dimension corresponds to one cut on the existing partition. Its performance is validated in space partition applications on toy data and in relational modelling. The BSP-Tree process provides clear performance gains over all competing methods. 
%     Particularly, we have tried to tune the budget parameter through a Metropolis-Hastings algorithm. The results on several dataset has validated the merit of our process.

\section{Related Work}

%==========================================================
    Stochastic partition processes aim to  divide a product space into meaningful blocks. A popular application of such processes is modelling relational data whereby the interactions within each block tend to be homogeneous. For state-of-the-art stochastic partition processes, partitioning strategies vary, including regular-grids ~\cite{kemp2006learning}, hierarchical partitions~\cite{roy2009mondrian, roy2007learning} and entry-by-entry strategies~\cite{nakano2014rectangular}. 
    
    A regular-grid stochastic partition process constitutes separate partition processes on each dimension of the multi-dimensional array. The resulting orthogonal interactions between two dimensions will exhibit regular grids, which can represent interacting intensities. Typical regular-grid partition models include the infinite relational model (IRM)~\cite{kemp2006learning} and the overlapping communities extension of mixed-membership stochastic blockmodels~\cite{airoldi2009mixed}. Regular-grid partition models are widely used in real-world applications for modeling graph data~\cite{ishiguro2010dynamic,nonpa2013schmidt, Li_transfer_2009}.

    The Mondrian process (MP)~\cite{roy2009mondrian,roy2011thesis} and its  variant the Ostomachion process (OP)~\cite{xuhui2016OstomachionProcess}, can produce hierarchical partitions on a product space. The MP recursively generates axis-aligned cuts on a unit hypercube and partitions the space in a hierarchical fashion known as the $k$d-tree (\cite{roy2007learning} also considers a tree-consistent partition model, but it is not a Bayesian nonparametric model). While using similar hierarchical partition structures, the OP additionally allows oblique cuts for flexible partitions, however it does not guarantee the important self-consistency property.

%==========================================================

\section{The BSP-Tree Process}

%==========================================================

	In the Binary Space Partitioning (BSP)-Tree process, we aim to generate partitions $\boxplus$ on an arbitrary two-dimensional convex polygon $\Box$. The partitioning result 
	$\boxplus$ can be represented as a set of blocks $\boxplus=\{\{\Box^{(k)}\}^{k\in\mathbb{N}^+}:\cup_k \Box^{(k)}=\Box, \Box^{(k')}\cap \Box^{(k'')}=\emptyset, \forall k'\neq k''\}$. These blocks are generated by a series of cuts, which recursively divide one of the existing blocks into two new blocks. As a result, these recursive divisions organize the blocks in the manner as a Binary Space Partitioning tree structure, from which the process name is derived.

%	Let $(\Omega, \mathcal{F}, \mathbb{P})$ denoting a probability space and $(\mathcal{S}, \mu)$ denoting a measurable state space (each element of $\mathcal{S}$ is the partitioning result $\boxplus$ here), our BSP-Tree process is defined as a collection of partitioning result $\boxplus$-valued random variables, which can be written as:
%	\begin{align}
%	    \{\boxplus(t, \omega), t\in\}
%	\end{align}
%	where $(0, \tau]$ refers to the index line and $\tau$ is a pre-fixed budget. 
	
	We first use a pre-fixed budget $\tau$ as a termination point in an index line of $(0, \tau]$. The BSP-Tree process is defined as an almost surely right continuous Markov jump process in $(0, \tau]$. Except for some isolated time points (corresponding to the locations of cuts in $(0, \tau]$), the values (partitions) taken in any realization of the BSP-Tree process are constant between these consecutive points (cuts). Let $\{\tau_l\}^l$ denote the locations of these time points in $(0, \tau]$ and $\boxplus_t$ denote the partition at time $t$. We then have $\boxplus_{t}=\boxplus_{\tau_l}, \forall t: \tau_{l}\le t<\tau_{l+1}$. More precisely, $\boxplus_{t}$ lies in a measurable state space $(\mathcal{S}, \Sigma)$, where $\mathcal{S}$ refers to the set of all potential Binary Space Partitions and $\Sigma$ refers to the $\sigma$-algebra over the elements of $\mathcal{S}$. Thus, the BSP-Tree process can be interpreted as a $\mathcal{S}^{(0, \tau]}$-valued random variable, where $\mathcal{S}^{(0, \tau]}$ is the space of all possible partitions of $t\in(0, \tau]$ that map from the index line $(0, \tau]$ into the space $\mathcal{S}$. Figure~\ref{fig:generativeprocesspartition} displays a sample function in $\mathcal{S}^{(0, \tau]}$.
    \begin{figure}[t]
    \centering
    \includegraphics[width =  0.48 \textwidth]{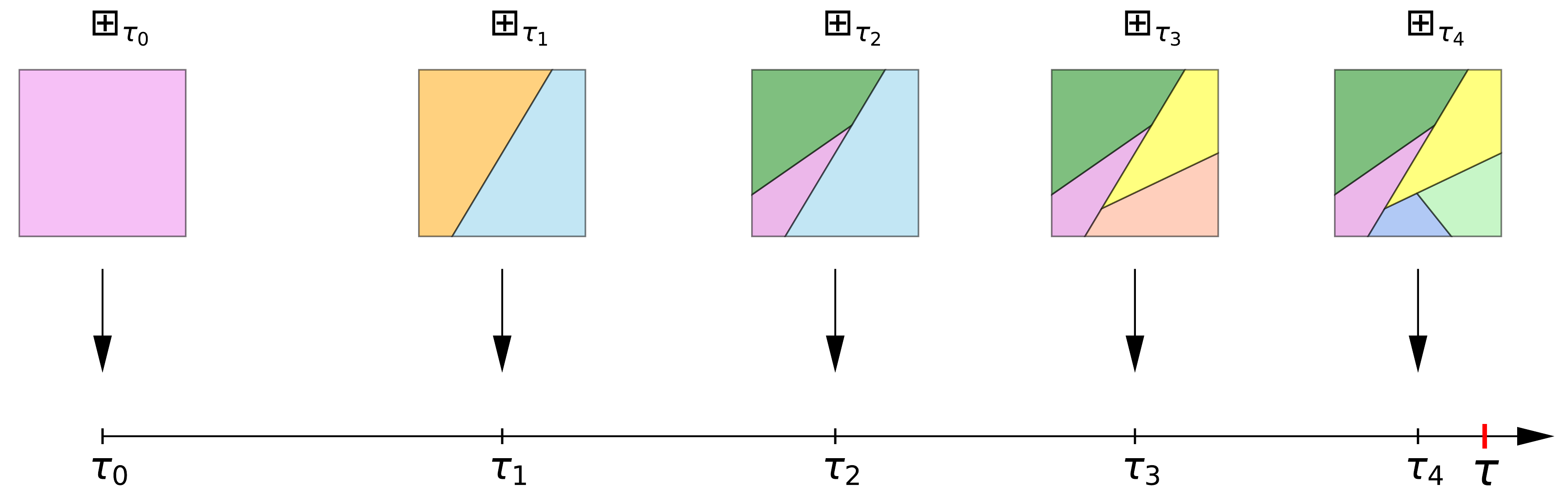}
    \caption{A realization of the BSP-Tree process in $(0, \tau]$.}
    \label{fig:generativeprocesspartition}
    \end{figure}
	
    \begin{figure*}[t]
    \centering
    \includegraphics[width =  0.7 \textwidth]{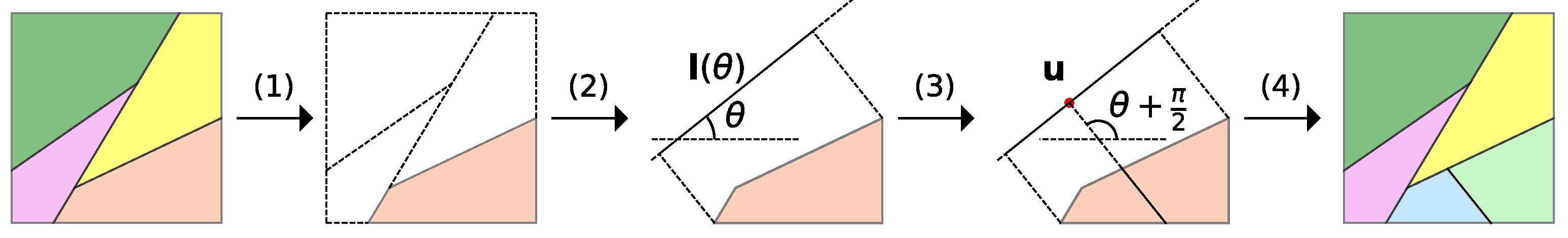}
\caption{Given the partition $\boxplus_{\tau_3}$, the generative process for the cut line $L(\theta, \pmb{u})$ at the location of $\tau_4$. }
    \label{fig:stagesofpartition}
    \end{figure*}
% \includepdf[pages=-]{images/BSP_generative_process.pdf}       
	As well as the partition at $\tau_{l-1}$, the incremental time for the $l$-th cut also conditions only on the previous partitions at $\tau_{l-1}$. Given an existing partition at $\tau_{l-1}$, the time to the next cut, $\tau_{l}-\tau_{l-1}$, follows an Exponential distribution:
	\begin{align} \label{elaping_time}
	    (\tau_{l}-\tau_{l-1})|\boxplus_{\tau_{l-1}}\sim \mbox{Exp}(\sum_{k=1}^lPE(\Box_{\tau_{l-1}}^{(k)}))
	\end{align}
    where $PE(\Box_{\tau_{l-1}}^{(k)})$ denotes the perimeter of the $k$-th block in partition $\boxplus_{\tau_{l-1}}$. Each cut divides one block (the block is chosen with probabilities in proportion to their perimeters) into two new blocks and forms a new partition. If the index of the location $\tau_l$  of the new cut exceeds the budget $\tau$, the BSP-Tree process terminates and returns the current partition $\boxplus_{\tau_{l-1}}$ as the final realization.

    \subsection{Generation of the cut line} \label{sec_cut_line_generation}

%    To keep simplicity, we let $X=[0, 1]^2$ and assume the elastic Mondrian process is proceeded in the domain of a unit square $[0, 1]^2$. In the generative process, the cut $\eta$ belongs to the set of potential oblique line cuts crossing into one of the existing blocks $\{\Box_k\}_k$ in $\boxplus_{\Box}$.

    {In each block $\Box^{(k)}$ in the partition, the BSP-Tree process defines cuts as straight lines cross through $\Box^{(k)}$. This is achieved by generating a random auxiliary line $\pmb{l}(\theta)$ with direction $\theta$, onto which the block $\Box^{(k)}$ is projected, uniformly selecting a cut position $\pmb{u}$ on the projected segment, and then constructing the cut as the line orthogonal to $\pmb{l}(\theta)$ which passes through $\pmb{u}$. } In detail,  given the partition $\boxplus_{\tau_{l-1}}$ at time $\tau_{l-1}$, the generative process of the line for the $l$-th cut is: \begin{description} \label{partitiondescription}
    \item{(1)} Sample a candidate block $\Box^*$ from the existing blocks $\{\Box^{(k)}_{\tau_{l-1}}\}^{k=1,\ldots,l}$, with probabilities in proportion to the perimeters (refer to Proposition \ref{measure_definition_proofs}) of these existing blocks ({Figure \ref{fig:stagesofpartition}} -(1));
    \item{(2)} Sample a direction $\theta$ from $(0, \pi]$, where the probability density function is in proportion to the length of the line segment $\pmb{l}(\theta)$, onto which $\Box^*$ is projected in the direction of $\theta${\footnote{This can be achieved using rejection sampling. We use the uniform distribution $g(\theta)=1/\pi$ over $(0, \pi]$. The scaling parameter is determined as $M=\pi/2$, which guarantees a tight upper bound such that $(l(\theta)/PE(\Box))/(M*g(\theta))<= 1$. The expected number of sampling times is then $\pi/2$}.} ({Figure \ref{fig:stagesofpartition}} -(2));
    \item{(3)} Sample the cutting position $\pmb{u}$ uniformly on the line segment $\pmb{l}(\theta)$. The proposed cut is formed as the straight line passing through $\pmb{u}$ and crossing through the block $\Box^*$, orthogonal to $\pmb{l}(\theta)$ ({Figure \ref{fig:stagesofpartition}} -(3)).
    \item{(4)} Sample the incremental time for the proposed cut according to Eq. (\ref{elaping_time}). If $\tau_{l}>\tau$, reject the cut and return $\{\Box^{(k)}_{\tau_{l-1}}\}^{k=1,\ldots,l}$ as the final partition structure; otherwise accept the proposed cut, increment $l$ to $l+1$ and go back to Step (1) ({Figure \ref{fig:stagesofpartition}} -(4)).
    \end{description}
    It is clear that the blocks generated in this process are convex polygons.    
    Step (4) determines that the whole process terminates only if the accumulated cut cost exceeds the budget, $\tau$. However, notice that $\tau_{l}\to\infty$ as $l\to\infty$ almost surely~(justification in Supplementary Material A). This means an infinite number of cuts would require an infinite budget almost surely. I.e., this block splitting process terminates to any finite budget $\tau$ with probability one. 

    In the following, we analyze the generative process for proposing the cut. For reading convenience, we first consider the case of cutting on a sample block $\Box$ (a convex polygon). Its extension to the whole partition $\boxplus$ (i.e., a set of blocks $\{\Box^{(k)}\}$) is then straightforward.

%
%    This joint probability can be decomposed as:
%	\begin{eqnarray}
%    P(\eta^{(t)}(\theta, u)|\eta^{(1), \cdots, (t-1)})=\nu\circ\eta(\theta, \vec{u})= \nu\circ\eta(\theta, \cdot)\cdot\frac{\nu\circ\eta(\theta, \vec{u})}{\nu\circ\eta(\theta, \cdot)}
%    \end{eqnarray}
%    where $f_{\eta}(\theta) = \nu\circ\eta(\theta, \cdot)$ denotes the probability of the set of cut operations that fixed at the direction of $\theta$ ; $f_{\eta}(\vec{u}|\theta) = \frac{\nu\circ\eta(\theta, \vec{u})}{\nu\circ\eta(\theta, \cdot)}$ represents the conditional probability of $\vec{u}$ given $\theta$.
%

\subsection{Cut definitions and notations}

    {From step (3) in the generative process, the cut line  is defined as a set of points in the block $L(\theta, \pmb{u}):=\{\pmb{x}\in\Box|(\pmb{x}-\pmb{u})^{\top}\cdot(1;\tan\theta)=0\}$.} The set of cut lines for all the potential cuts crossing into block $\Box$ can be subsequently denoted as $C_{\Box}=\{L(\theta,\pmb{u})|\theta\in[0, \pi), \pmb{u}\mbox{ lies on the line segment } \pmb{l}(\theta)\}$. 
    
    % {The cut line  is defined as a set of points in the block $L(\theta, \pmb{u}):=\{\pmb{x}\in\Box|(\pmb{x}-\pmb{u})^{\top}\cdot(1;\tan\theta)=0\}$. }{The measures over $C_{\Box}$ and $T_{\Box}$ are described as: $\overline{\nu}_{\Box}(T_{\Box}):=\overline{\nu}_{\Box}\circ \phi(C_{\Box})=\lambda_{\Box}(C_{\Box})$, where $\overline{\nu}_{\Box}(\cdot)$ denotes the normalized probability measure on $T_{\Box}$ and $\lambda_{\Box}(C_{\Box})$ denotes the probability measure on $C_{\Box}$.} $C_{\Box}^{\theta}=\{L(\theta,\pmb{u})|\theta\mbox{ is fixed}, \pmb{u}\mbox{ lies on the line segment}\}$ is used to the set of all cut lines with fixed direction $\theta$ and by $\lambda_{\Box}(C_{\Box}^{\theta})$ the associated probability measure. $\lambda_{\Box}(L(\theta, \pmb{u})|C_{\Box}^{\theta})$ is used to denote the conditional probability measure on the line $L(\theta, \pmb{u})$ given direction $\theta$.  

     {Each of the element in $C_{\Box}$ corresponds to a partition on the block $\Box$. This is described by a one-to-one mapping $\phi$: $C_{\Box}\to T_{\Box}$, where $T_{\Box}$ denotes the set of one-cut partitions on the block $\Box$. The measures over $C_{\Box}$ and $T_{\Box}$ are described as: $\overline{\nu}_{\Box}(T_{\Box}):=\overline{\nu}_{\Box}\circ \phi(C_{\Box})=\lambda_{\Box}(C_{\Box})$, where $\overline{\nu}_{\Box}(\cdot)$ denotes the normalized probability measure on $T_{\Box}$ and $\lambda_{\Box}(C_{\Box})$ denotes the probability measure on $C_{\Box}$.}

    The direction $\theta$ and the cut position $\pmb{u}$ are sequentially sampled in steps (2) and (3) in the generative process, where $\pmb{u}$ is located on the  image of the polygon projected onto the line $\pmb{l}(\theta)$. For step (2), we denote by $C_{\Box}^{\theta}=\{L(\theta,\pmb{u})|\theta\mbox{ is fixed}, \pmb{u}\mbox{ lies on the line segment}\}$ the set of all cut lines with fixed direction $\theta$ and by $\lambda_{\Box}(C_{\Box}^{\theta})$ the associated probability measure. In Step (3), we use $\lambda_{\Box}(L(\theta, \pmb{u})|C_{\Box}^{\theta})$ to denote the conditional probability measure on the line $L(\theta, \pmb{u})$ given direction $\theta$.  

 \subsection{Uniformly distributed cut lines $L(\theta, \pmb{u})$}

% The definition of $\nu$ can be analysed through a conditional probability:
% \begin{eqnarray}
% \nu(\boxplus^*)=\lambda_{\boxplus}\{(\theta, u|\Box)\} = \lambda_{\boxplus}\{\theta, u\in l(\theta)|\Box\}\cdot \frac{\lambda_{\boxplus}\{\theta, u|\Box\}}{\lambda_{\boxplus}\{\theta, u\in l(\theta)|\Box\}}, \forall \boxplus^* \in R_{\boxplus}
%\end{eqnarray}
%where $\lambda_{\boxplus}\{\theta, u\in l(\theta)|\Box\} = \int_{u\in l(\theta)} \lambda_{\boxplus}\{(\theta, u|\Box)\} d\mu(u)$ is the marginal probability on $\theta$ and $\frac{\lambda_{\boxplus}\{\theta, u|\Box\}}{\lambda_{\boxplus}\{\theta, u\in l(\theta)|\Box\}}$ is the conditional probability on $u$ given $\theta$.
    % While using this strategy in generating the line for cut, we can show it is Uniformly distributed. 
    It is easy to demonstrate that the above strategy produces a cut line $L(\theta, \pmb{u})$ that is uniformly distributed on $C_{\Box}$. 
    \paragraph{Marginal probability measure $\lambda_{\Box}(C_{\Box}^{\theta})$:} We restrict the marginal probability measure $\lambda_{\Box}(C_{\Box}^{\theta})$ of $C_{\Box}^{\theta}$ to the family of functions that remains invariant under the following three operations on the block $\Box$ (their mathematical definitions are provided in Supplementary Material B):
    \begin{enumerate}
    \item {Translation $t$}: $\lambda_{\Box}(C_{\Box}^{\theta}) = \lambda_{t_{\pmb{v}}\Box}\circ t_{\pmb{v}}(C_{\Box}^{\theta})$, where $t_{\pmb{v}}(\cdot)$ denotes incrementing the set of points by a vector $\pmb{v}, \forall \pmb{v}\in \mathbb{R}^2$; 
    \item {Rotation $r$}: ${\lambda_{\Box}(C_{\Box}^{\theta})} = {\lambda_{r_{\theta'}\Box}\circ r_{\theta'}(C_{\Box}^{\theta})}$, where $r_{\theta'}(\cdot)$ denotes rotating the set of points by an angle $\theta', \forall \theta'\in [0, \pi)$; 
    \item {Restriction $\psi$}: $\lambda_{\Box}(C_{\triangle}^{\theta}) = \lambda_{\psi_{\triangle}\Box}\circ\psi_{\triangle}(C_{\triangle}^{\theta})$, where ${\triangle}\subseteq \Box$ refers to a sub-region of $\Box$; $\psi_{\triangle}(\cdot)$ retains the set of points in ${\triangle}$, and {$C_{\triangle}^{\theta}$ refers to the set of cut lines (orthogonal to $\pmb{l}(\theta)$) that cross through ${\triangle}$}.
    \end{enumerate}
    \begin{figure}[t]
    \centering
    \includegraphics[width =  0.35\textwidth]{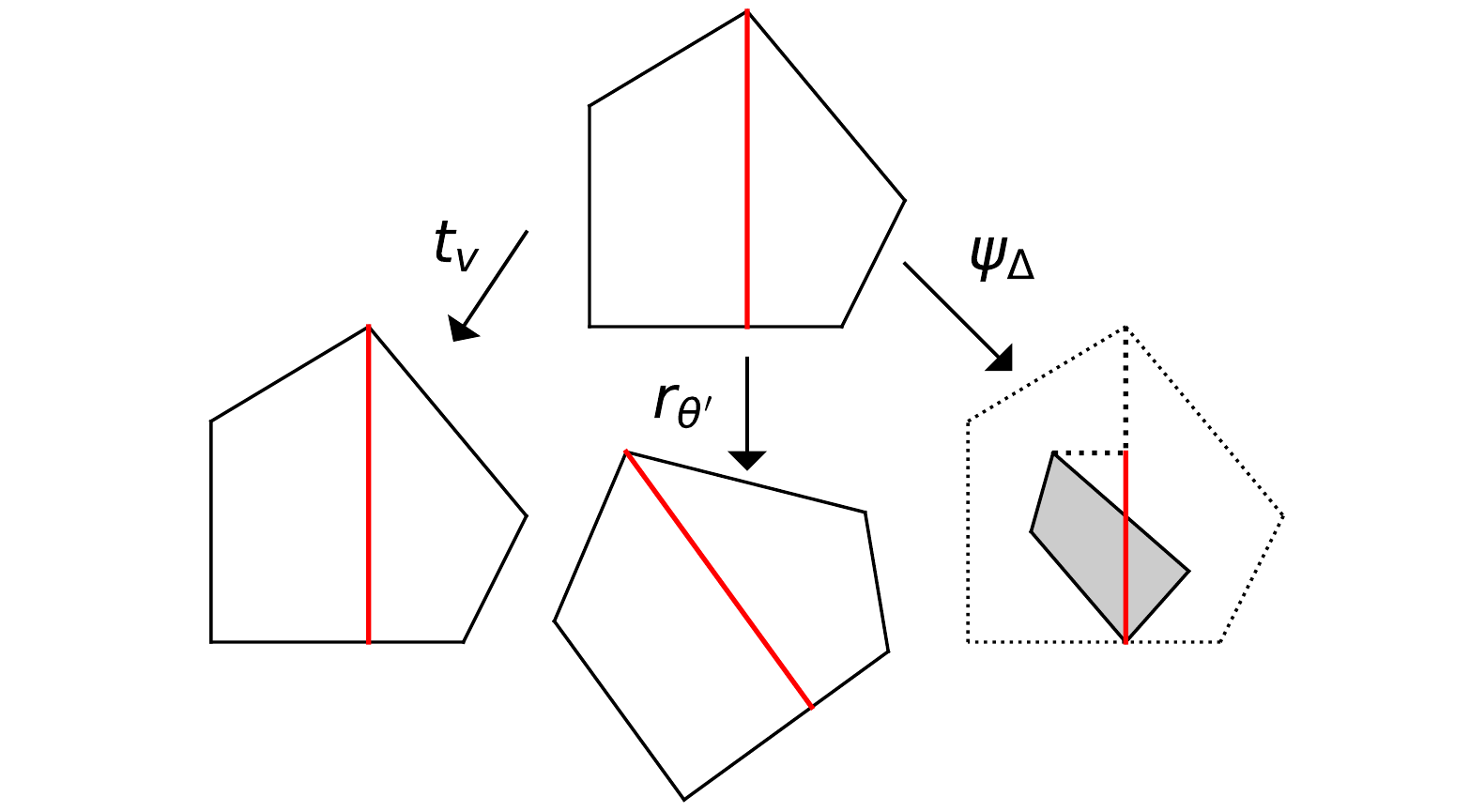}
    \caption{$|\pmb{l}(\theta)|$ remains invariant under the operations of translation ($t_{\pmb{v}}$), rotation ($r_{\theta'}$) and restriction ($\psi_{\triangle}$).}
    \label{fig:invariant_under_3_operations}
    \end{figure}

    Define a set function as $f_{\Box}(C_{\Box}^{\theta}) = |\pmb{l}_{\Box}(\theta)|$, where $|\pmb{l}_{\Box}(\theta)|$ is the length of an image of the polygon $\Box$ projected onto the line with direction $\theta$. It is clear that $f_{\Box}(C_{\Box}^{\theta})$ remains invariant under the first two operations. For the restriction $\psi$, we have $f_{\Box}(C_{\triangle}^{\theta})=f_{\Box}(\{L(\theta, \pmb{u}|\pmb{u}\mbox{ lies on $\pmb{l}_{\psi_{\triangle}(\Box)}(\theta)$})\})=|\pmb{l}_{\triangle}(\theta)|=f_{\triangle}(\{L(\theta, \pmb{u}|\pmb{u}\mbox{ lies on $\pmb{l}_{\psi_{\triangle}(\Box)}(\theta)$})\})$. A visualization can be seen in Figure~\ref{fig:invariant_under_3_operations}. Further, the following result shows $\lambda_{\Box}\left(C_{\Box}^{\theta}\right)$ is fixed to a scaled form of $f_{\Box}(C_{\Box}^{\theta})$~(proof in Supplementary Material C):
    
    \begin{proposition} \label{measure_definition_proofs}
    The family of functions $\lambda_{\Box}\left(C_{\Box}^{\theta}\right)$ remains invariant under the operations of translation, rotation and restriction if and only if there is a constant $C$ such that $\lambda_{\Box}\left(C_{\Box}^{\theta}\right)=C\cdot f_{\Box}(C_{\Box}^{\theta})=C\cdot |\pmb{l}(\theta)|, \forall C\in \mathbb{R}^+$.
    \end{proposition}

    Following Proposition~\ref{measure_definition_proofs}, the measure of all cut lines on the whole block $\Box$ may be obtained by integrating out the direction $\theta$. This produces the result that the measure over the cuts on the block is fixed to the scale of the perimeter of the block (proof in  Supplementary Material D):
    \begin{proposition} \label{cut_measure}
    Assuming the direction $\theta$ has a distribution of $|\pmb{l}(\theta)|/\int_{0}^{\pi}|\pmb{l}(\theta)|d\theta$, the BSP-Tree process has the partition measure $\lambda_{\Box}(C_{\Box})=C\cdot\int_{0}^{\pi}|\pmb{l}(\theta)| d\theta = C\cdot\mbox{PE}(\Box)$.
    \end{proposition}
    
    Combining Proposition \ref{measure_definition_proofs} and Proposition \ref{cut_measure}, we obtain the probability measure as $\lambda_{\Box}\left(C_{\Box}^{\theta}\right) = |\pmb{l}(\theta)|/PE(\Box)$. 

    \paragraph{Conditional probability $\lambda_{\Box}\left(L(\theta, \pmb{u})|C_{\Box}^{\theta}\right)$:} Given the direction $\theta$, define the conditional probability $\lambda_{\Box}\left(L(\theta, \pmb{u})|C_{\Box}^{\theta}\right)$ as a 
    uniform distribution over $\pmb{l}(\theta)$. That is: $\lambda_{\Box}\left(L(\theta, \pmb{u})|C_{\Box}^{\theta}\right)=\frac{1}{|\pmb{l}(\theta)|}$. 

    As a result, the uniform distribution of $L(\theta, \pmb{u})$ can be obtained as $\lambda_{\Box}\left(L(\theta, \pmb{u})\right)=\lambda_{\Box}\left(C_{\Box}^{\theta}\right)\cdot \lambda_{\Box}\left(L(\theta, \pmb{u})|C_{\Box}^{\theta}\right)=1/PE(\Box)$. Further, $\lambda_{\Box}\left(L(\theta, \pmb{u})\right) = \overline{\nu}_{\Box}\circ\phi\left(L(\theta, \pmb{u})\right)=\overline{\nu}_{\Box}\left(\boxplus^*\right)$, where $\boxplus^*\in T_{\Box}$. That is to say, the new partition over $T_{\Box}$ is also uniformly distributed.
    
    This uniform distribution does not require the use of particular partition priors. Without additional knowledge about the cut, all potential partitions are equally favored.  

    \begin{figure}[t]
    \centering
    \includegraphics[width = 0.4 \textwidth]{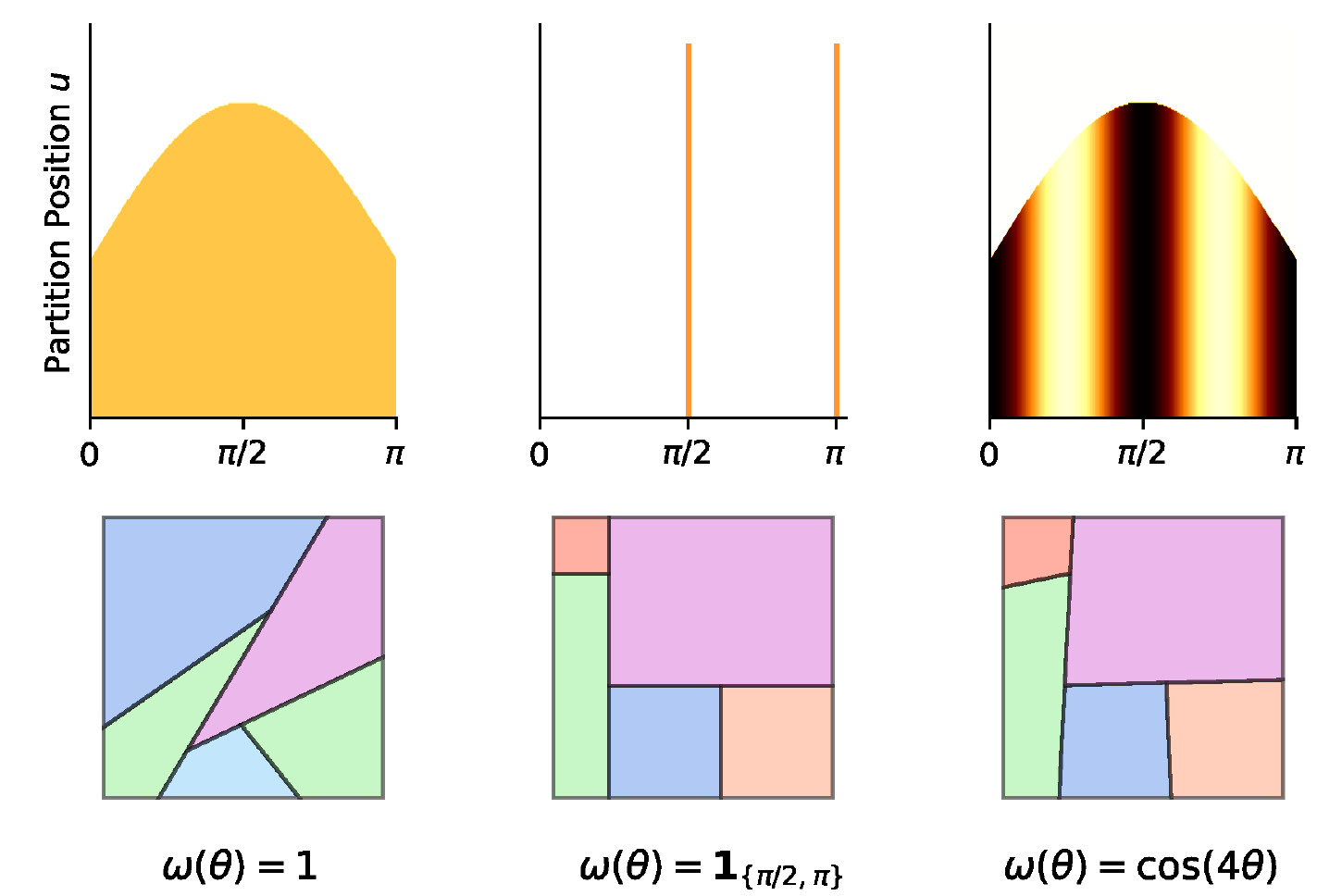}
    \caption{Various settings of $\omega(\theta)$ and example partitions.}
    \label{fig:NonUnifrom}
    \end{figure}
\paragraph{Extension to the partition $\boxplus$:}
    The extension from the single block $\Box$ to the whole partition $\boxplus=\{\Box^{(k)}\}$ is completed by the exponentially distributed incremental time of $\{\tau^{(k)}_l-\tau^{(k)}_{l-1}\}^l$ for each block $\Box^{(k)}$, where with an abuse of notation $\tau_l^{(k)}$ refers to the location of the $l$-th cut in block $\Box^{(k)}$. In terms of partition $\boxplus$, the minimum incremental time for all the blocks is distributed as as $\min\{\tau^{(k)}_{l}-\tau^{(k)}_{l-1}\}^{k}\sim \mbox{Exp}(\sum_k PE(\Box^{(k)}))$ since $P(\min\{\tau^{(k)}_{l}-\tau^{(k)}_{l-1}\}^{k}>t)=\prod_{k}P(\tau^{(k)}_{l}-\tau^{(k)}_{l-1}>t)=\exp(-t\sum_k PE(\Box^{(k)}))$, which is the complementary CDF of the exponential distribution. Further, we have that $P(k^*=\arg\min_k\{\tau^{(k)}_{l}-\tau^{(k)}_{l-1}\}^{k})=PE(\Box^{(k^*)})/\sum_k PE(\Box^{(k)})$, which is the probability of selecting the block to be cut. These results correspond to Step (1) and Step (4) in the generative process.

%
%     This can be verified through the properties of the exponential distribution. Since each cut is assigned to one of the existing blocks with proportions as the block's perimeter, the selected block's intensity rate is calculated as $\tau_k=\frac{c_k}{\sum_{k'}c_{k'}}\cdot \sum_{k'}c_{k'}=c_k$. Also, due to memoryless of the Poisson Process, we have $P(c>t+s|c>s) = P(c>t), \forall s, t>0$. Given the current occurred cut, all the blocks' cuts ``start as new''. In this way, we can control the number of blocks by setting appropriate value for the budget. This corresponds to Step (1) in the generative process.
%
%    Eq. (\ref{eq_cut_measure}) indicates that the measure of a single block equals to its perimeter. Due to the independence of block evolving in the generative process, we can easily conclude the total measure upon the current partition equals to the sum of the existing blocks' perimeters. This corresponds to Step (4) in the generative process.

\subsection{Non-uniformly distributed cut lines $L(\theta, \pmb{u})$}
    The BSP-Tree process does not require that the cut line is uniformly distributed over $C_{\Box}$. In some cases (e.g. in social network partitioning, where the blocks refer to communities), some blocks may tend to have regular shapes (rectangular-shaped communities with equal contributions from the nodes). In these scenarios, axis-aligned cuts would have larger contributions than any others. In the following, 
    we generalize the uniformly distributed cuts to a particular form of non-uniformly distributed cuts, placing arbitrary weights on choosing the directions of the cut line.
        
    \paragraph{Non-uniform measure} This generalization is achieved by relaxing the invariance restriction on the measure of $C_{\Box}^{\theta}$, under the rotation operation. We use an arbitrary non-negative finite function 
    $\omega(\theta)$ as prior weight on the direction $\theta$.  Under the rotation {$r_{\theta'}$}, the probability measure $C_{\Box}^{\theta}$ requires to have the form of {${\omega(\theta+\theta')}\cdot{\lambda_{\Box}(C_{\Box}^{\theta})} = {\omega(\theta)}\cdot{\lambda_{r_{\theta'}\Box}\circ r_{\theta'}(C_{\Box}^{\theta})}$}. Following similar arguments as proposition~\ref{measure_definition_proofs}, this implies that $\lambda_{\Box}(C_{\Box}^{\theta}) = C\cdot\omega(\theta)\cdot|\pmb{l}(\theta)|$. Given the 
    uniform conditional distribution of $\lambda_{\Box}(L(\theta, \pmb{u})|C_{\Box}^{\theta})$, the probability measure over 
    $L(\theta, \pmb{u})$ is $\lambda_{\Box}\left(L(\theta, \pmb{u})\right)\propto\omega(\theta)$. 
    Clearly, $\lambda_{\Box}\left(L(\theta, \pmb{u})\right)$ is a non-uniform distribution and it is determined by the weight
    function $\omega(\theta)$.

%    the likelihood function for sampling the cut direction 
%    $\theta$ is thus proportional to $f(\theta)\propto \omega(\theta)\cdot |\vec{l}(\theta)|$. While keeping the cut position $\vec{u}$ the same as uniformly generating from the width 
%    $|\vec{l}(\theta)|$, the cost for the cut is generated from the exponential distribution with the parameter $\int_{[0, \pi)}\omega(\theta)\cdot |\vec{l}(\theta)|d\theta$. 

%    This generalization can be well explained by the definition and assumption of the cut distribution. The flexible prior over the direction $\theta$ can be achieved by adjusting 
%    the invariance assumption over the operator of rotation ($r_{\theta'}$). Particularly, 

%    , where $\omega(\theta)$ is a non-negative finite weight function. 
    
%    , by using similar . 

    {Figure~\ref{fig:NonUnifrom}} displays some examples of the probability function 
    of $\lambda_{\Box}\left(L(\theta, \pmb{u})\right)$ and corresponding partition visualizations under different settings of $\omega(\theta)$. The color indicates the value of the PDF 
    at the orthogonal slope $\theta$, with darker colors indicating larger values. While $\omega(\theta) = 1$, the cuts follows the uniform distribution; 
    while with $\omega(\theta) = 1_{\{\pi/2, \pi\}}$, the BSP-Tree process is reduced to the Mondrian process, whereby only axis-aligned partitions are allowed; 
    when $\omega(\theta)=\cos(4\theta)$, the weighted probability distribution adjusts the original one into shapes of stripes along the $\theta$-axis. In each example, the colour on the same direction $\theta$ is constant. This illustrates the uniform distribution of the partition position $\pmb{u}$ given the direction $\theta$.

    \paragraph{Mixed measure} Naturally, we may use mixed distributions on $\theta$ for greater modelling specificity. If we have two different non-negative weight functions 
    $\omega_{1}(\theta)$ and $\omega_2(\theta)$, the partition measure over these two sets can be written as 
    $\lambda_{\Box}^{1}(C_{\Box}^{\theta})=C_1\cdot\int_{[0, \pi)}\omega_{1}(\theta)\cdot|l(\theta)|d\theta, \lambda_{\Box}^{2}(C_{\Box}^{\theta})=C_2\cdot\int_{[0, \pi)}\omega_{2}(\theta)\cdot|l(\theta)|d\theta$, where $C_1$ and $C_2$ are non-negative constants. For instance, let $\omega_1(\theta) = \pmb{1}_{\{\pi/2, \pi\}}, \omega_2(\theta) = \pmb{1}_{(0, \pi]}$. The direction $\theta$ is then 
    sampled as: 
    { \begin{align} \label{mixed_disribution}
    \theta\sim\left\{ \begin{array}{ll}
    \sum_{\frac{\pi}{2}, \pi}\mathbf{1}_{\theta}\frac{|\pmb{l}(\theta)|}{|\pmb{l}(0)|+|\pmb{l}(\frac{\pi}{2})|}, & z=1; \\
    \frac{|\pmb{l}(\theta)|}{PE(\Box)}, & z= 0. 
    \end{array} \right.
    \end{align} }
    where $z\sim \mbox{Bernoulli}(\frac{C^1}{C^1+C^2})$, indicating which distribution for $\theta$ samples from. That is, $\theta$ is sampled either from the discrete set of $\{\pi/2, \pi\}$ or the continuous set of $(0, \pi]$. Figure \ref{fig:mixed_measure} shows example partition visualizations 
    based on this particular mixed measure.
    \begin{figure}[t]
    \centering
    \includegraphics[width =  0.08\textwidth]{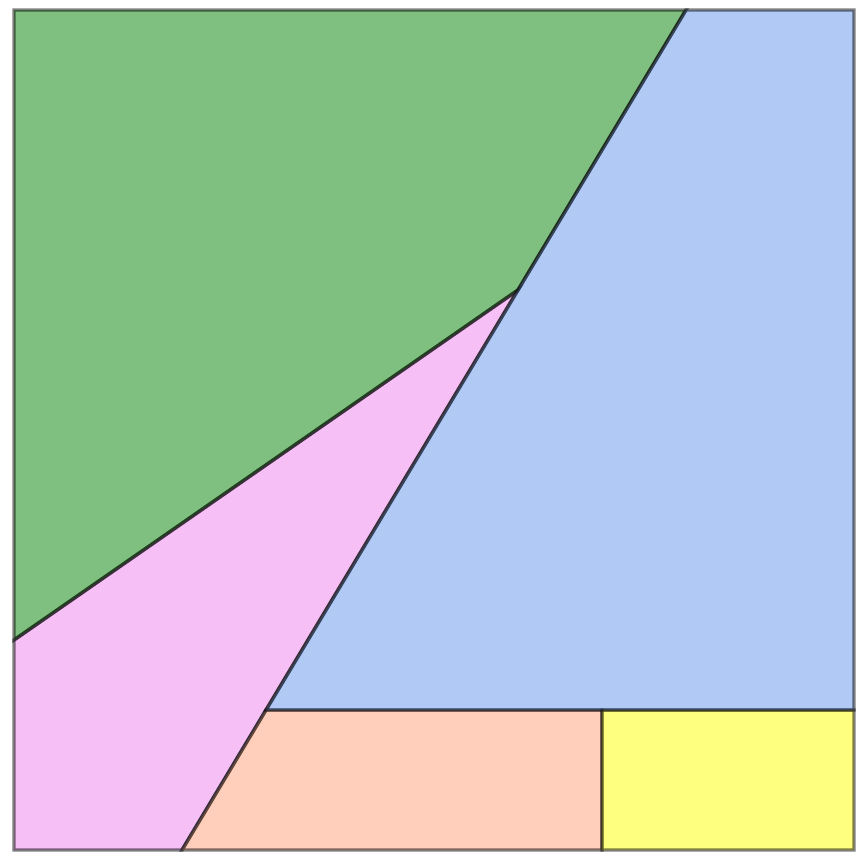} \qquad
    \includegraphics[width =  0.08\textwidth]{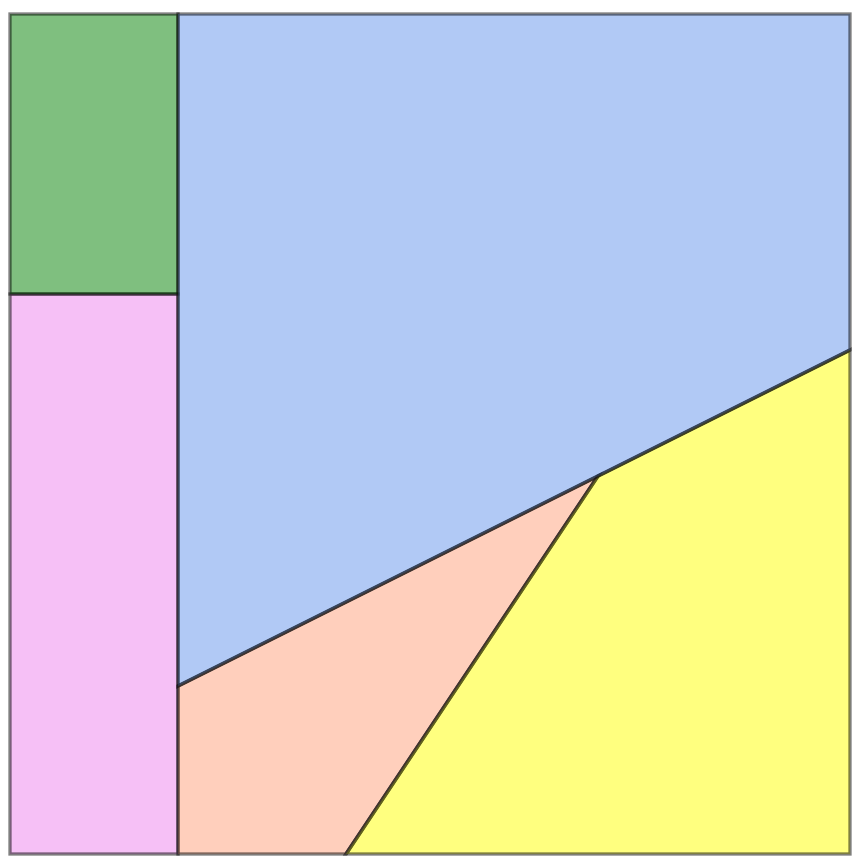}
    \caption{Example partitions on mixed measures.}
    \label{fig:mixed_measure}
    \end{figure}

    \begin{figure*}[t]
    \centering
    \includegraphics[width =  0.75\textwidth]{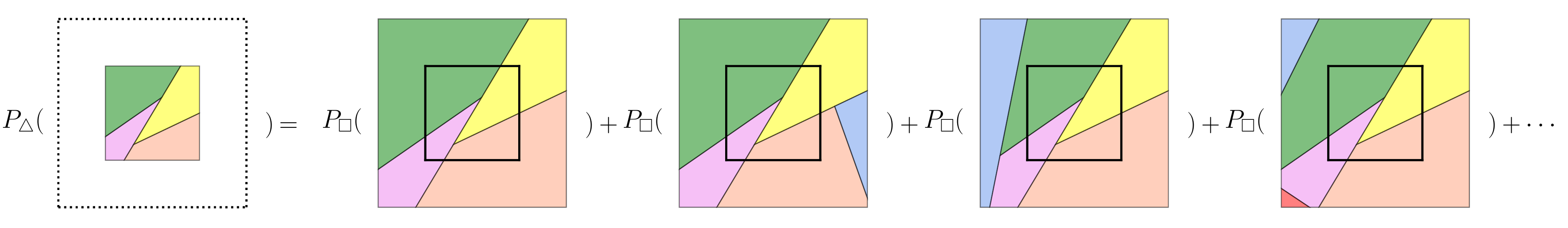}
    \caption{Self-Consistency of the BSP-Tree process.}
    \label{fig:Consistency}
    \end{figure*}

\subsection{Self-consistency}

	The BSP-Tree process is defined on a finite convex polygon. To further extend the BSP-Tree process to the infinite two-dimensional space $\mathbb{R}^2$, an essential property is self-consistency. That is, while restricting the BSP-Tree process on a finite two-dimensional convex polygon $\Box$, to its sub-region $\triangle, \triangle\subseteq\Box$, the resulting partitions restricted to $\triangle$ are distributed as if they are directly generated on $\triangle$ through the generative process (see Figure~\ref{fig:Consistency}).
	
    For both of the uniform-distributed and non-uniform distributed cut lines in the BSP-Tree process, we have the following result~(justification in Supplementary Material E):
    \begin{proposition}
    The BSP-Tree process (including the uniformly distributed and non-uniformly distributed cut lines described above) is self-consistent, which maintains distributional invariance when restricting its domain from a larger convex polygon to its belonging smaller one.
    \end{proposition}
    The Kolmogorov Extension Theorem~\cite{chung2001course} ensures that the self-consistency property enables the BSP-Tree process to be defined on the infinite two-dimensional space. This property can be suited to some cases, such as online learning, domain expansion.

%    Some notations are firstly defined for convenient reference. We use $X$ and $A$ to denote a domain and its subdomain, which is $A\subseteq X$. $M_t$ and $N_t$ are individually defined as the Mondrian processes on $X$ and $A$ respectively. The restriction is denoted as $\Pi$. Thus, we could have $\Pi_A M_t=N_t$. $X$'s Perimeter is denoted as $PE_{\Box}$.
%
%%    We have similar recursive construction property as that of the Mondrian Process
%%    \begin{proposition}
%%    $\forall \pi\in\mathcal{F}_{\Box}$ and let $\{N^A\}_{A\in\pi}$ be independent eMPes on $A\in\pi$. Then their union partition $\cup_{A\in\pi} N^A$ is an eMP.
%%    \end{proposition}
%
%    Two cases need to be considered in proving this theorem: 1), the equivalence of extending the eMP from domain $A$ to $X$; 2), the equivalence of restricting it from domain $X$ to $A$. As provided the elaboration in the appendix, the discussion on case (1) mainly follows the idea of \cite{roy2011thesis}, however, we have given our own understanding for case (2).

    \begin{algorithm}[h]
    \caption{C-SMC for inferring $\boxplus$ at $(t+1)$-th iteration} \label{particlegibbs}
    \begin{algorithmic}[1]
        \REQUIRE Training data $X$, Budget $\tau>0$, Number of particles
        $C$, Partition sequence $\{\tau_l^*, \boxplus_{\tau_l}^*\}^l$ at the $t$-th iteration , Likelihood function $P(X|\boxplus)$.
        \STATE Initialize partitions for each particle $P^{c}=\Box$ and weight $\omega_0^{c}=1$, stage $l=1$, $\tau_l^c=0$
        \WHILE{$\exists c, s.t. $ $\tau^c_l<\tau$}
        \FOR{$c = 1:C$}
        \IF{$\tau_l^c<\tau$}
        \IF{$c=1$}
        \STATE Set $P_{l}^{1}=\boxplus_{\tau_l}^*, \tau_l^1=\tau_l^*$
        \ELSE
        \STATE Sample $P^c_{l}, \tau^c_l$ according to the generative process of the cut line in Section \ref{sec_cut_line_generation}
        \ENDIF
        \STATE Set $\omega^c_{l} = \omega^c_{l-1}\cdot\frac{P(X|P^c_l)}{P(X|P^c_{l-1})}$
        \ELSE
        \STATE Set $P^c_{l} = P^c_{l-1}, \omega^c_{l} = \omega^c_{l-1}, \tau^c_{l} = \tau^c_{l-1}$
        \ENDIF
        \ENDFOR
        \STATE Normalize weights $\overline{\omega}^c_{l} = \frac{\omega^c_l}{W_l}$, $W_l = \sum_c\omega^{c}_l$
        \STATE Set $j_1=1$. 
        \STATE For $c=2:C$, resample indices $j_c$ from $\sum_c\overline{\omega}^c_{l}\delta_{P^c_l}$
        \STATE $\forall c, P^c_{l} = P_l^{j_c}$; $\omega^c_l = W_l/C$; $l=l+1$
        \ENDWHILE
        \STATE Return partition at $(t+1)$-th iteration: $\boxplus\sim \sum_c\overline{\omega}^c_l\delta_{P^c_l}$
    \end{algorithmic}
    \end{algorithm}
\section{C-SMC Sampler for the BSP-Tree Process}
    Inference for the BSP-Tree type hierarchical partition process is hard, since early partitions would heavily influence subsequent blocks and consequently their partitions. That is, all later blocks must be considered while making inference about these early stage-cuts (or intra-nodes cuts from the perspective of the tree structure). Previous approaches have typically used local structure movement strategies to slowly reach the posterior structure, including the Metropolis-Hastings algorithm~\cite{roy2009mondrian}, Reversible-Jump Markov chain Monte Carlo~\cite{wang2011nonparametric, pratola2016}. These algorithms suffer either slow mixing rates or a heavy computational cost. 
    
    To avoid this issue, we use a Conditional-Sequential Monte Carlo~(C-SMC)  sampler~\cite{andrieu2010particle} to infer the structure of the BSP-Tree process. Here the tree structure is taken as the high-dimensional variable in the sampler. The cut line can be taken as the dimension dependence between the consecutive partition states of the particles. In this way, a completely new tree can be sampled in each iteration, rather than the local movements in the previous approaches. 
    
    Algorithm \ref{particlegibbs} displays the detailed strategy to infer the tree structure in the $(t+1)$-th iteration, given the generative partition sequence $\{\tau_l^*, \boxplus_{\tau_l}^*\}^l$ at the $t$-th iteration. We should note the likelihood $P(X|\boxplus)$ varies in different applications. For example, $P(X|\boxplus)$ would be a Multinomial probability in Section \ref{sec:toydata} and a Bernoulli probability in Section \ref{sec:relationaldata}. Line $2$ ensures that algorithm continues until the generative process of all particles terminates. Line $6$ fixes the $1$-st particle as the partition sequence in the $t$-th iteration in all stages. Lines $15-18$ refer to the resampling step.  

%==========================================================

\section{Applications}
%==========================================================

\subsection{Toy data partition visualization} \label{sec:toydata}
Three different cases of toy data are generated as: {\bf Case 1:} Dense uniform data on $[0, 1]^2$. The labels for these points are determined by a realization of BSP-Tree process on $[0, 1]^2$; {\bf Case 2:} $1,000$ data points generated from $5$ bivariate Gaussian distributions, with an additional $100$ uniformly distributed noisy points; {\bf Case 3:} same setting with case 2, however with an additional $500$ uniformly distributed noisy points.

Using $\{(\xi_i, \eta_i)\}_{i=1}^N$to denote the coordinates of the data points, the data are generated as: (1) Generate an realization from the BSP-Tree process $\boxplus= \{\Box^{(k)}\}^{k\in\mathbb{N}^+}$; (2) Generate the parameter of the multinomial distribution in $\Box^{(k)}$, $  \pmb{\phi}_k\sim\mbox{Dirichlet}(\pmb{\alpha}), \forall k\in \mathbb{N}^+$; (3) Generate the labels for the data points $z_{i}\sim\mbox{Multinomial}(\pmb{\phi}_{h(\xi_i, \eta_i|\{\Box^{(k)}\})}), \forall i\in\{1, \cdots, N\}$, where $h(\xi_i, \eta_i)$ is a function mapping $i$-th data point to the corresponding block.     
    
For Case 1, we set $\tau=10$; for Case 2 and 3, we set $\tau = 3$. In all these cases, the weight function $\omega(\theta)$ is set as the uniform distribution $\omega(\theta) =1/\pi, \forall \theta\in(0, \pi]$.
Figure \ref{fig:toy_data_partition} shows example partitions drawn from the posterior distribution on each dataset (we put the visualization of Case 1 in the Supplementary Material). For the dense uniform data, the partitions can identify the subtle structure, while it might be accused of more cutting budget. In the other two cases, the nominated Gaussian distributed data are well classified, regardless of the noise contamination.

\begin{table*}[t]
\centering \small
\caption{Relational modeling (link prediction) comparison results (AUC$\pm$std)}
\begin{tabular}{l|ccccc}
  \hline
  Data Sets  &{IRM} & {LFRM} & {MP-RM} & {MTA-RM} & {BSP-RM} \\
  \hline
  Digg     & $0.792 \pm 0.011 $  & $0.801 \pm 0.031 $  & $0.784 \pm 0.020 $  & $0.793 \pm 0.005 $  & $\textbf{0.820} \pm 0.016 $ \\
  Flickr   & $0.870 \pm 0.003 $  & $0.881 \pm 0.006 $   & $0.868 \pm 0.011 $ & $0.872 \pm 0.004 $  & $\textbf{0.929} \pm 0.015 $ \\
  Gplus     & $0.857 \pm 0.002 $  & $0.860 \pm 0.008 $   & $0.855 \pm 0.007 $ & $0.857 \pm 0.002 $  & $\textbf{0.885} \pm 0.017 $ \\
  Facebook  & $0.872 \pm 0.013 $  & $0.881 \pm 0.023 $  & $0.876 \pm 0.028 $  & $0.885 \pm 0.010 $  & $\textbf{0.931} \pm 0.020 $ \\
  Twitter  & $0.860 \pm 0.003 $  & $0.868 \pm 0.021 $  & $0.815 \pm 0.055 $ & $0.870 \pm 0.006 $   & $\textbf{0.896} \pm 0.008 $ \\
  \hline
\end{tabular}
\label{table:rm_dataset}
\end{table*}
    \begin{figure*}[t]
    \centering
    \includegraphics[width =  0.12 \textwidth]{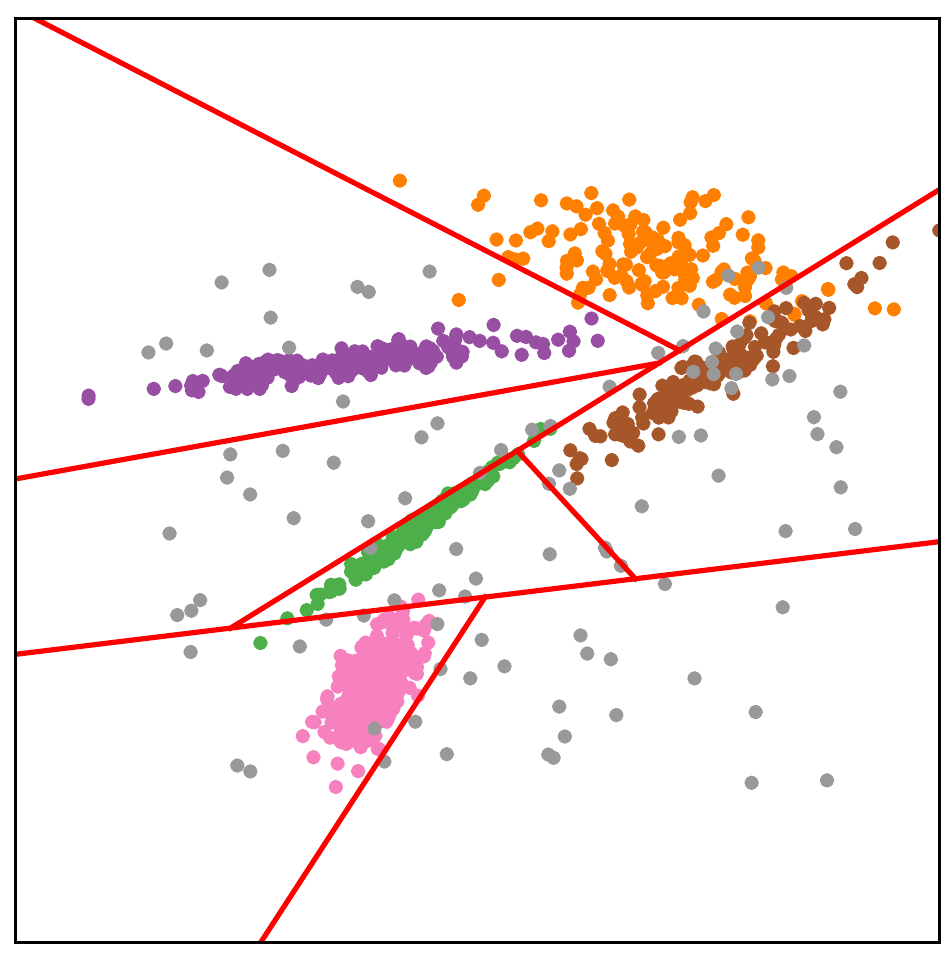}\qquad
    \includegraphics[width =  0.12 \textwidth]{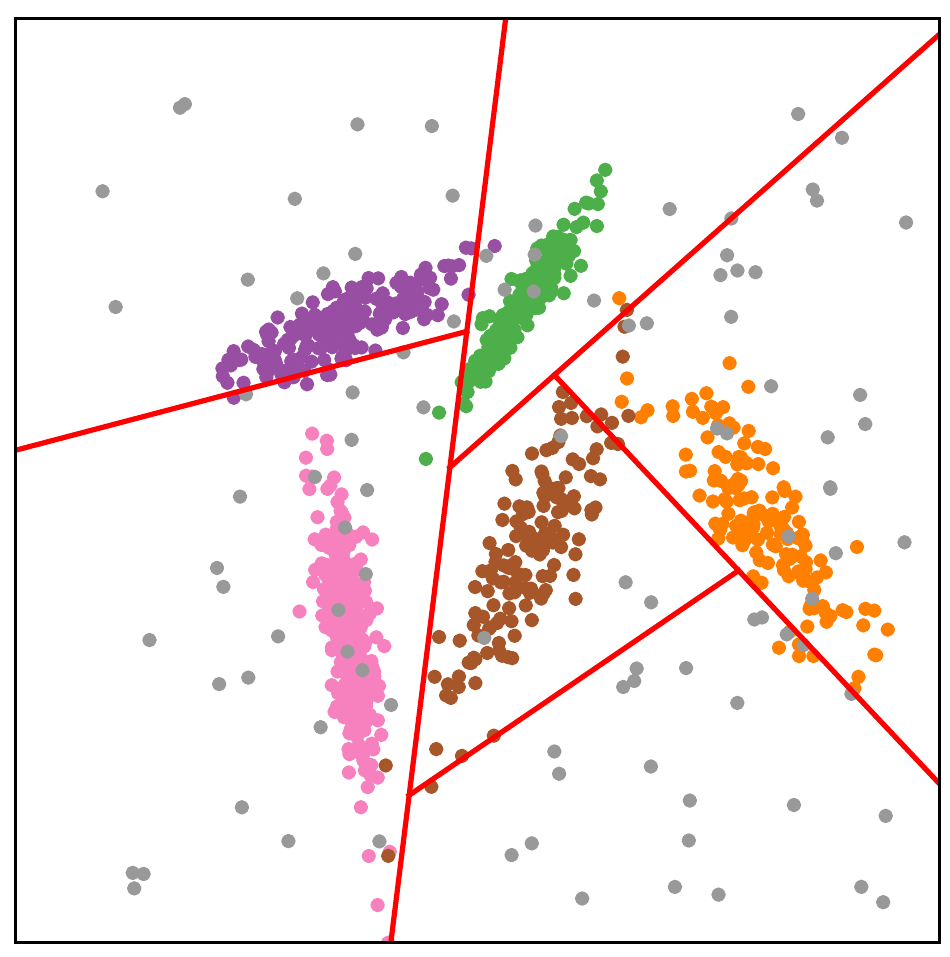}\qquad
    \includegraphics[width =  0.12 \textwidth]{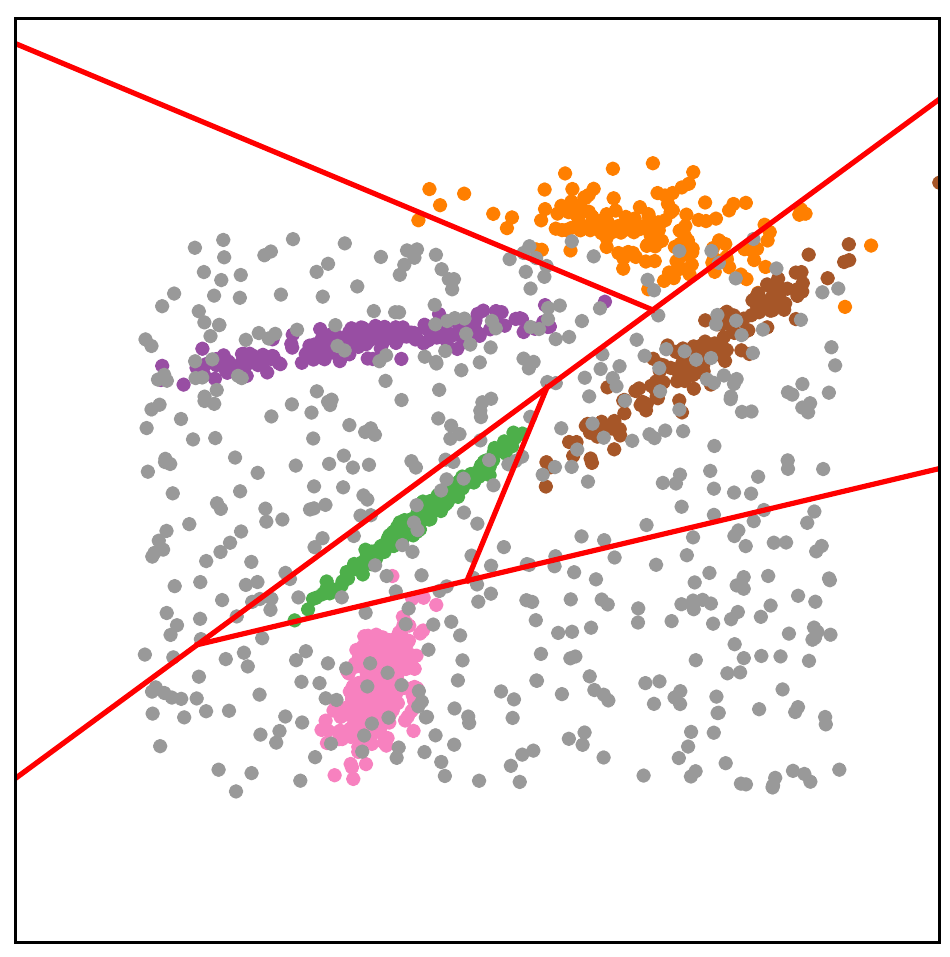}\qquad
    \includegraphics[width =  0.12 \textwidth]{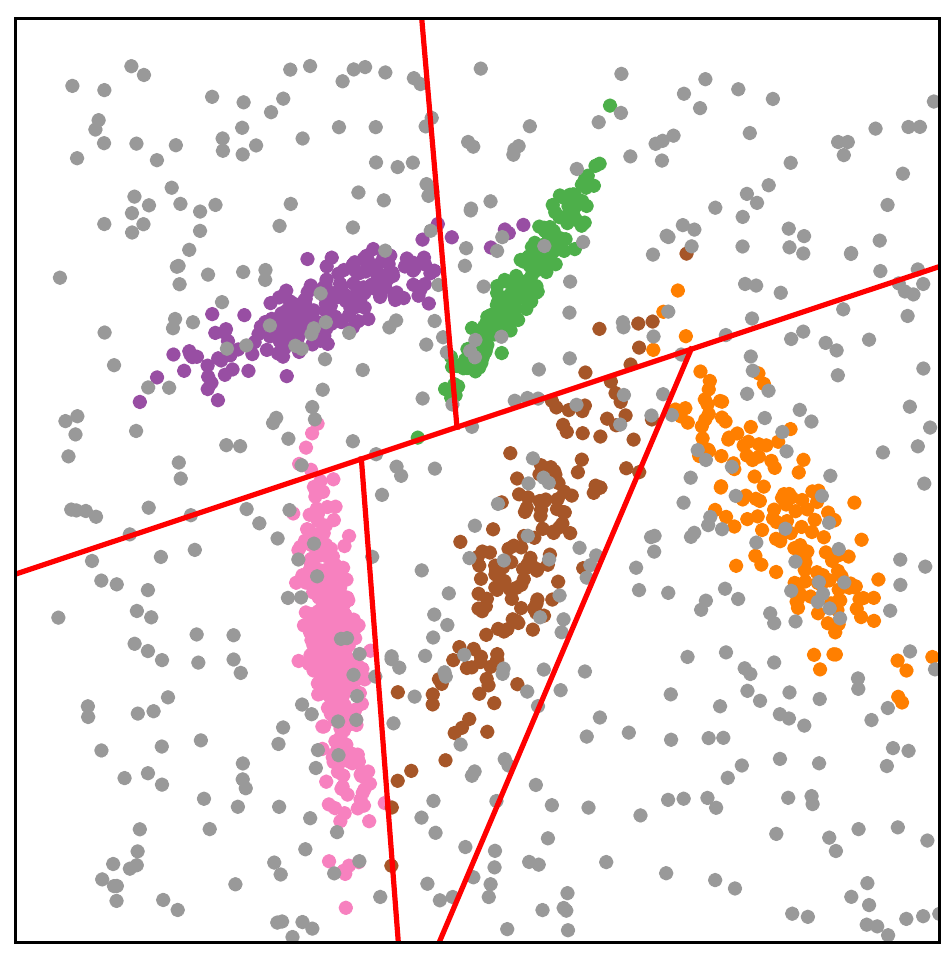}\qquad
    \caption{Toy Data Partition Visualization (left two: Case 2; right two: Case 3).}
    \label{fig:toy_data_partition}
    \end{figure*}
    
\subsection{Relational Modelling} \label{sec:relationaldata}
\subsubsection{Model Construction}
    A typical application for the BSP-Tree process is relational modelling \cite{kemp2006learning}. The observed data in relational modelling is an asymmetric matrix $R \in \{0,1\}^{N \times N}$. The rows and columns in the matrix represent the nodes in the interaction network and an entry $R_{ij}=1$ indicates a relation from node $i$ to node $j$. Partitions over the nodes in rows and columns will jointly form the blocks in this observed matrix. It is expected that the relations within each block  share homogeneity, compared to the relations between blocks. 
    
    The Aldous-Hoover representation theorem \cite{TPAMI2014peterdaniel} indicates that these types of exchangeable relational data can be modelled by a function on the unit space $[0, 1]^2$ and coordinates in the unit interval $[0, 1]$. In particular, the coordinates in the unit interval $[0, 1]$ represent the nodes and we use the block-wise function to denote the function on $[0, 1]^2$. Here, the blocks generated in the BSP-Tree process are used to infer the communities for these relations. In this way, the generative process is constructed as: (1) Generate an realization from the BSP-Tree process $\boxplus= \{\Box^{(k)}\}^{k\in\mathbb{N}^+}$; (2) Generate the parameter of the Bernoulli distribution $\phi_k\sim\mbox{Beta}(\alpha_0, \beta_0), \forall k\in \mathbb{N}^+$; (3) Generate the coordinates for the data points $\xi_i, \eta_j\sim\mbox{uniform}[0, 1], \forall i,j\in\{1, \cdots, N\}$;(4) Generate the relations $R_{ij}\sim\mbox{Bernoulli}(\phi_{h(\xi_i, \eta_j)}), \forall i, j\in\{1, \cdots, N\}$.
    
    The details of the inference procedure can be found in the Supplementary Material F. 
%    Due to the Beta-Bernoulli conjugacy, we integrate the likelihood parameter $\phi_k$ and thus get the likelihood as $\Pr(e_{ij}|e,) = \frac{\Gamma(N_{kl}^1+\alpha_0)\Gamma(N_{kl}^0+\beta_0)}{\Gamma(N_{kl}+\alpha_0+\beta_0)}$, where $N_{kl}^1$ denotes the number of relational in the whole base.

\subsubsection{Experiments}
    We compare the BSP-Tree process in relational modelling (BSP-RM) with one ``flat''-partition method and one latent feature method: (1) IRM~\cite{kemp2006learning}, which produces a regular-grid partition structure; (2) Latent Feature Relational Model (LFRM)~\cite{miller2009nonparametric}, which uses latent features to represent the nodes and these features' interactions to describe the relations; (3), the Mondrian Process (MP-RM), which uses the Mondrian Process to infer the structure of communities; (4),the Matrix Tile Analysis~\cite{Givoni06} Relational Model (MTA-RM). For IRM and LFRM, we adopt the collapsed Gibbs sampling algorithms for inference~\cite{kemp2006learning}; we implement RJ-MCMC~\cite{ReversibleJump_PJGreen,wang2011nonparametric} for the Mondrian Process and the Iterative Condition Modes algorithm~\cite{Givoni06} for MTA-RM.

    \textbf{Datasets}. We use 5 social network datasets, Digg, Flickr~\cite{Zafarani+Liu:2009}, Gplus~\cite{facebook_mcauley2012learning}, and Facebook, Twitter~\cite{leskovec2010predicting}. We extract a subset of nodes from each network dataset: We select the top $1,000$ active nodes based on their interactions with others; then randomly sample $500$ from these $1,000$ nodes for constructing the relational data matrix.
  
    \textbf{Experimental Setting}. In IRM, we let $\alpha$ be sampled from a gamma prior $\Gamma(1, 1)$ and the row and column partitions be sampled from two independent Dirichlet processes; In LFRM, we let $\alpha$ be sampled from a gamma prior $\Gamma(2, 1)$. As the budget parameter of MP-RM is hard to sample~\cite{balaji2016aistats}, we set it to $3$, which suggests that around $(3+1)\times(3+1)$ blocks would be generated. For parametric model MTA-RM, we simply set the number of tiles to 16. For the BSP-Tree process, we set the the budget $\tau$ as $8$, which is the same as MP, and $\omega(\theta)$ to be the form of the mixed distribution Eq. (\ref{mixed_disribution}). We compare the results in terms of training Log-likelihood (community detection) and testing AUC (link prediction). The reported performance is averaged over 10 randomly selected hold-out test sets ($\text{Train}:\text{Test} = 9:1$).

    \textbf{Experimental Results}. Table~\ref{table:rm_dataset} presents the performance comparison results on these datasets. As can be seen, the BSP-Tree process with the C-SMC strategy 
    clearly perform better than the comparison models. The AUC link prediction score is improved by around $3\%\sim 5\%$. 
    
  Figure~\ref{PartiionGraph} (rows 1-5) illustrates the sample partitions drawn from the resulting posteriors. The partition from the BSP-Tree process looks to capture dense irregular blocks and smaller numbers of sparse regions, showing the efficiency of the oblique cuts. While the two representative cutting-based methods, IRM and MP-RM, cut sparse regions into bigger number of blocks. Another observation is that regular and irregular blocks co-exist in Flickr and Facebook under BSP-RM. Thus, in addition to improved performance, BSP-RM also produces a more efficient partition strategy.

  Figure~\ref{PartiionGraph} (rows 6-7) plots the average performance versus wall-clock time for two measures of performance. IRM and LFRM converge fastest because of efficient collapsed Gibbs sampling. MTA-RM also converges fast because it is trained using a simple iterative algorithm. Although BSP-RM takes a bit longer time to converge, it is clear that it ultimately produces the best performance against other methods in terms of both AUC value and training log-likelihood. 

\section{Conclusions}
The BSP-Tree process is a generalization of the Mondrian process that allows oblique cuts within space partition models and also provably maintains the important self-consistency property. It incorporates a general (non-uniform) weighting function $\omega(\theta)$, allowing for flexible modelling of the direction of the cuts, with the axis-aligned only cuts of the Mondrian process given as a special case. Experimental results on both toy data and real-world relational data shows clear inferential improvements over the Mondrian process and other related methods.

Two aspects worth further investigation: (1) learning the budget $\tau$ to make the number of partitions better fit to the data; (2) learning the weight function $\omega(\theta)$ within the C-SMC algorithm to improve algorithm efficiency.

\section{Acknowledgement}
Xuhui Fan and Scott Anthony Sisson are supported by the Australian Research Council through the Australian Centre of Excellence in Mathematical and Statistical Frontiers (ACEMS, CE140100049), and Scott Anthony Sisson through the Discovery Project Scheme (DP160102544). Bin Li is supported by the Fudan University Startup Research Grant (SXH2301005) and Shanghai Municipal Science \& Technology Commission (16JC1420401).

\begin{figure*}[tp]
\centering
  \includegraphics[width = 0.17 \textwidth, viewport = 110 210 497 597, clip]{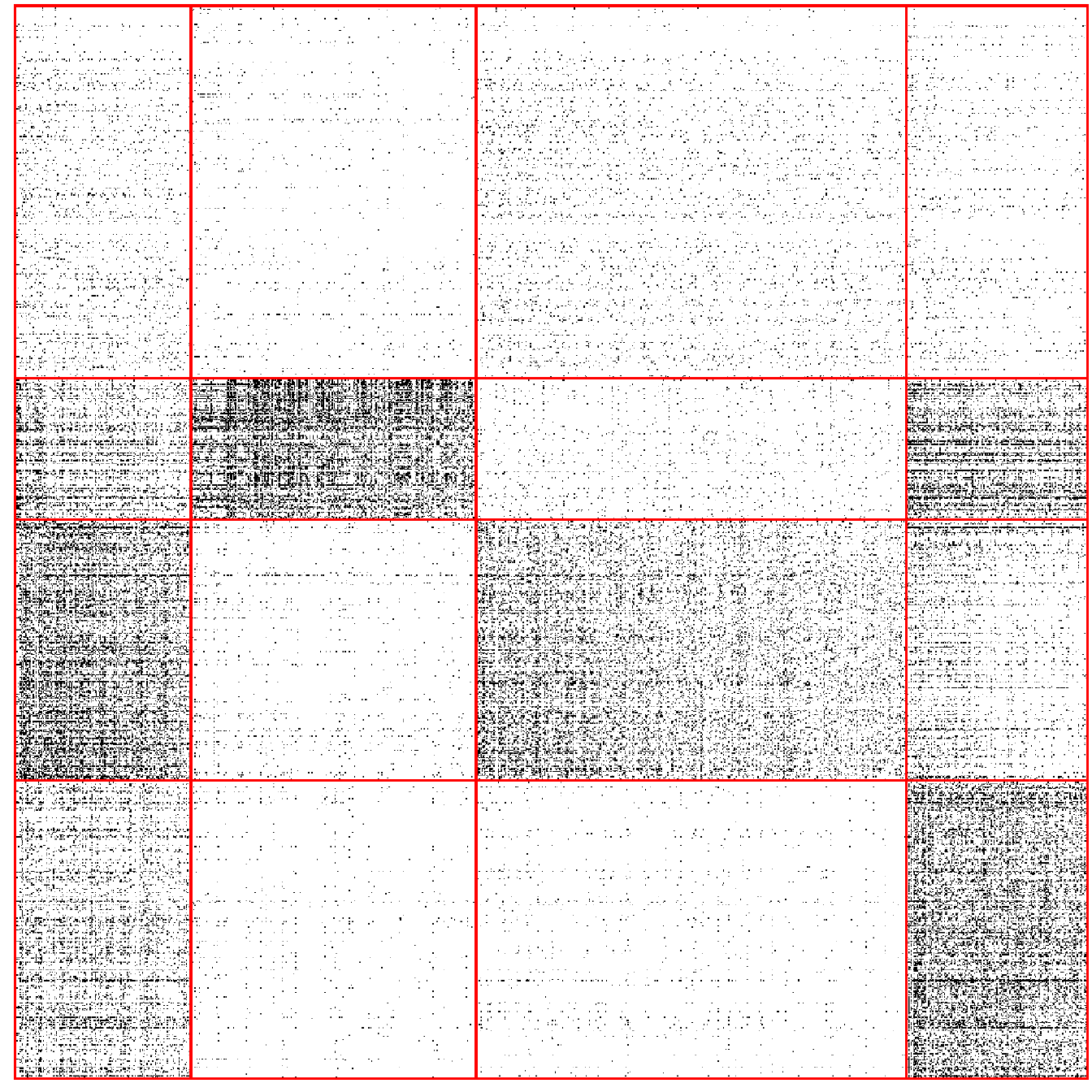}
  \includegraphics[width = 0.17 \textwidth, viewport = 110 210 497 597, clip]{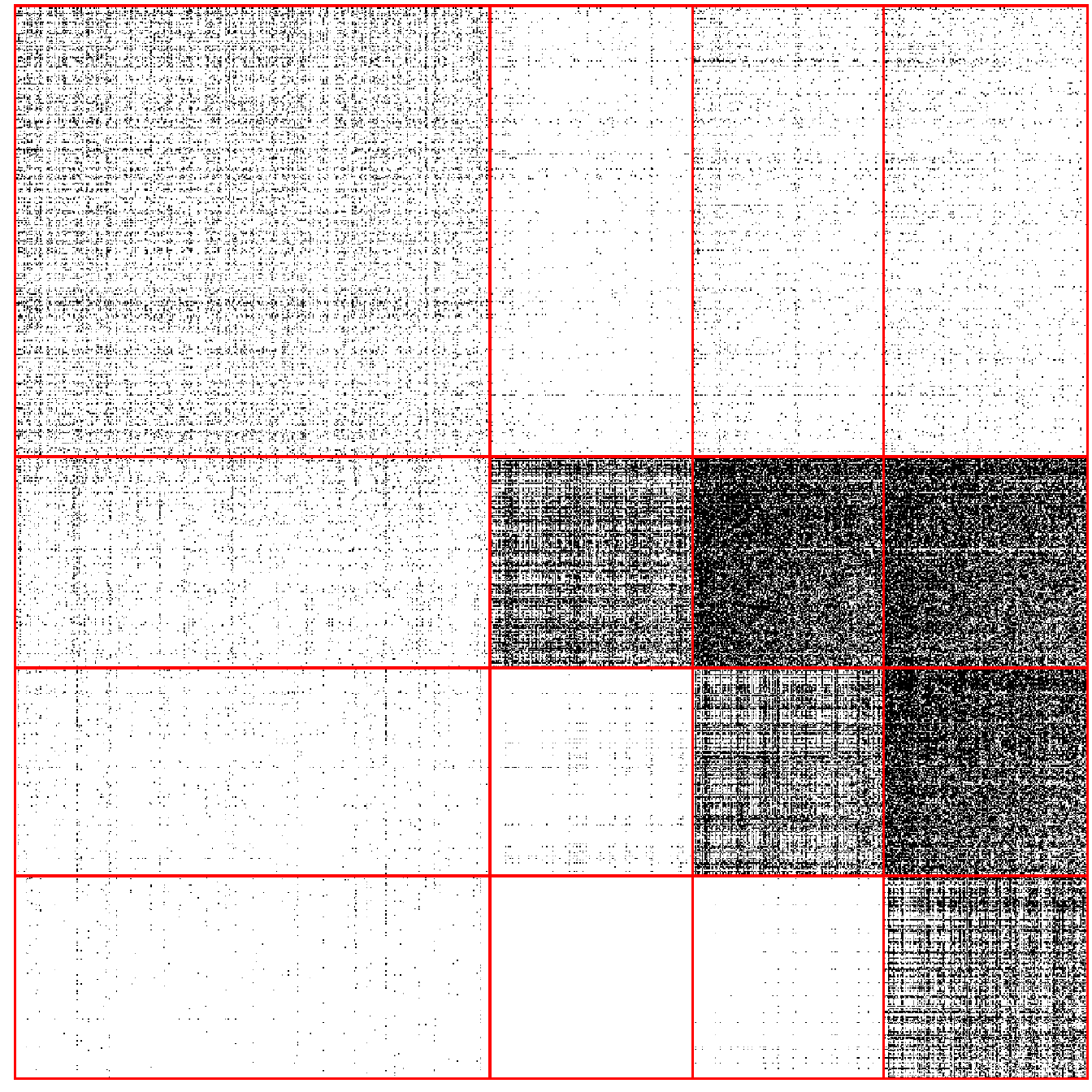}
  \includegraphics[width = 0.17 \textwidth, viewport = 110 210 497 597, clip]{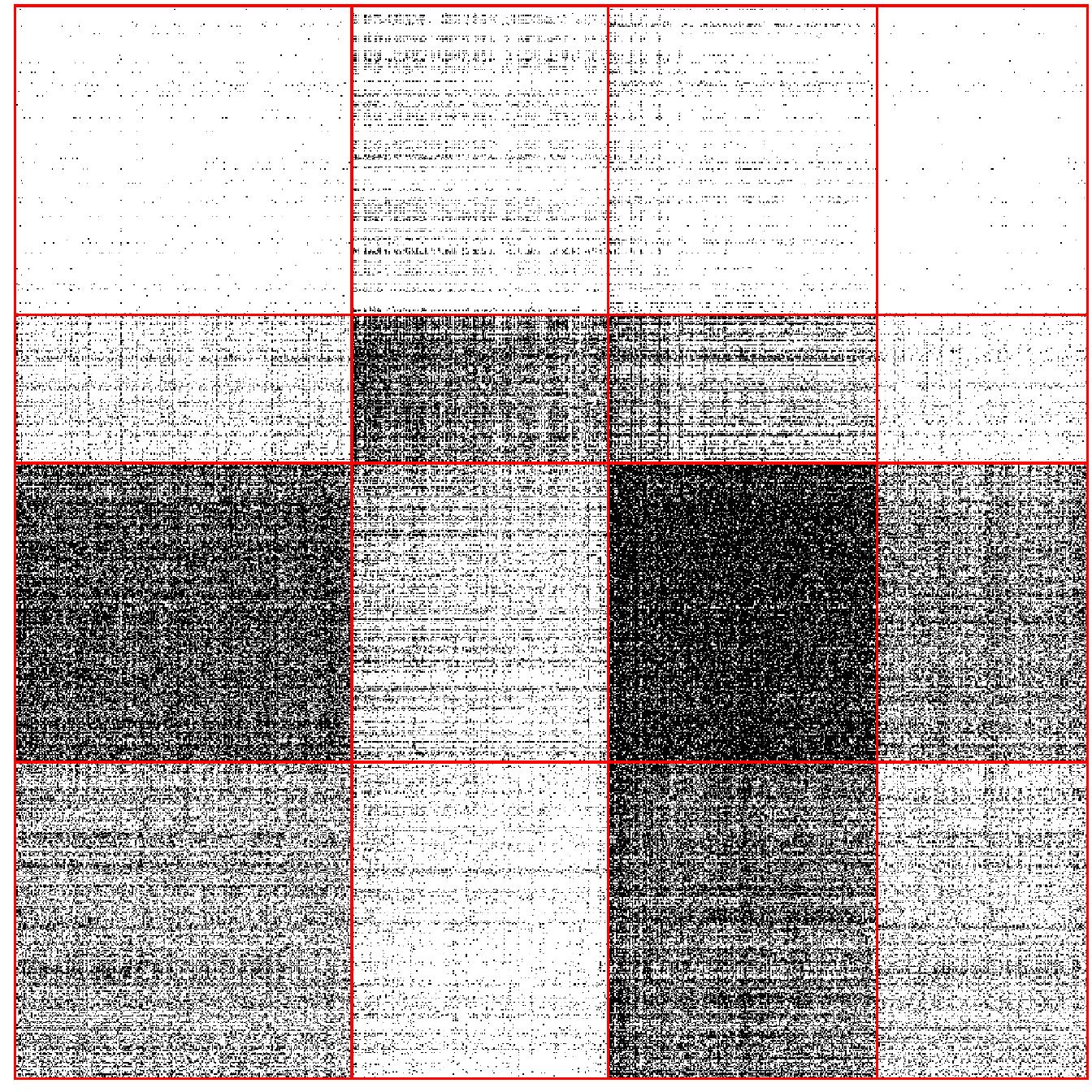}
  \includegraphics[width = 0.17 \textwidth, viewport = 110 210 497 597, clip]{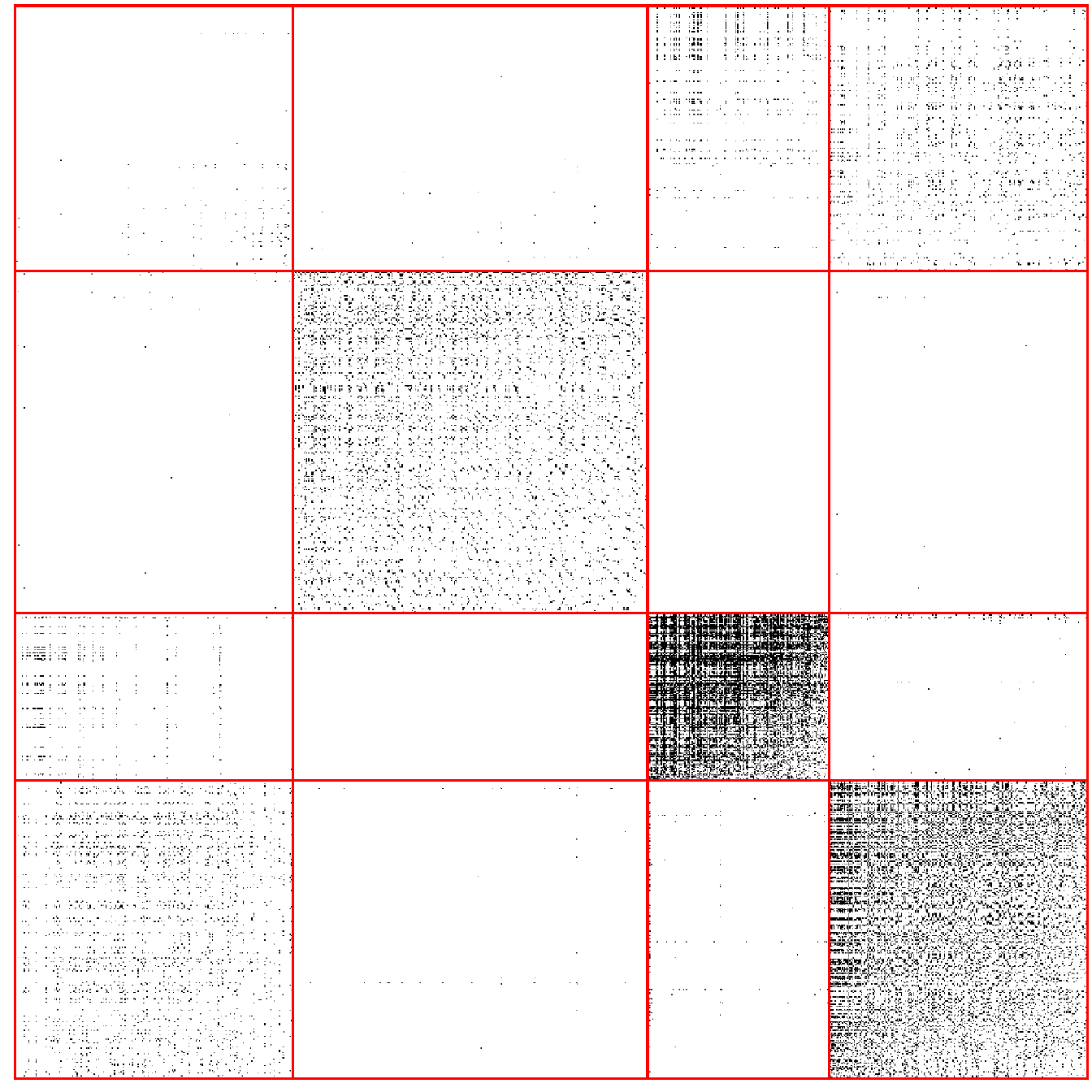}
  \includegraphics[width = 0.17 \textwidth, viewport = 110 210 497 597, clip]{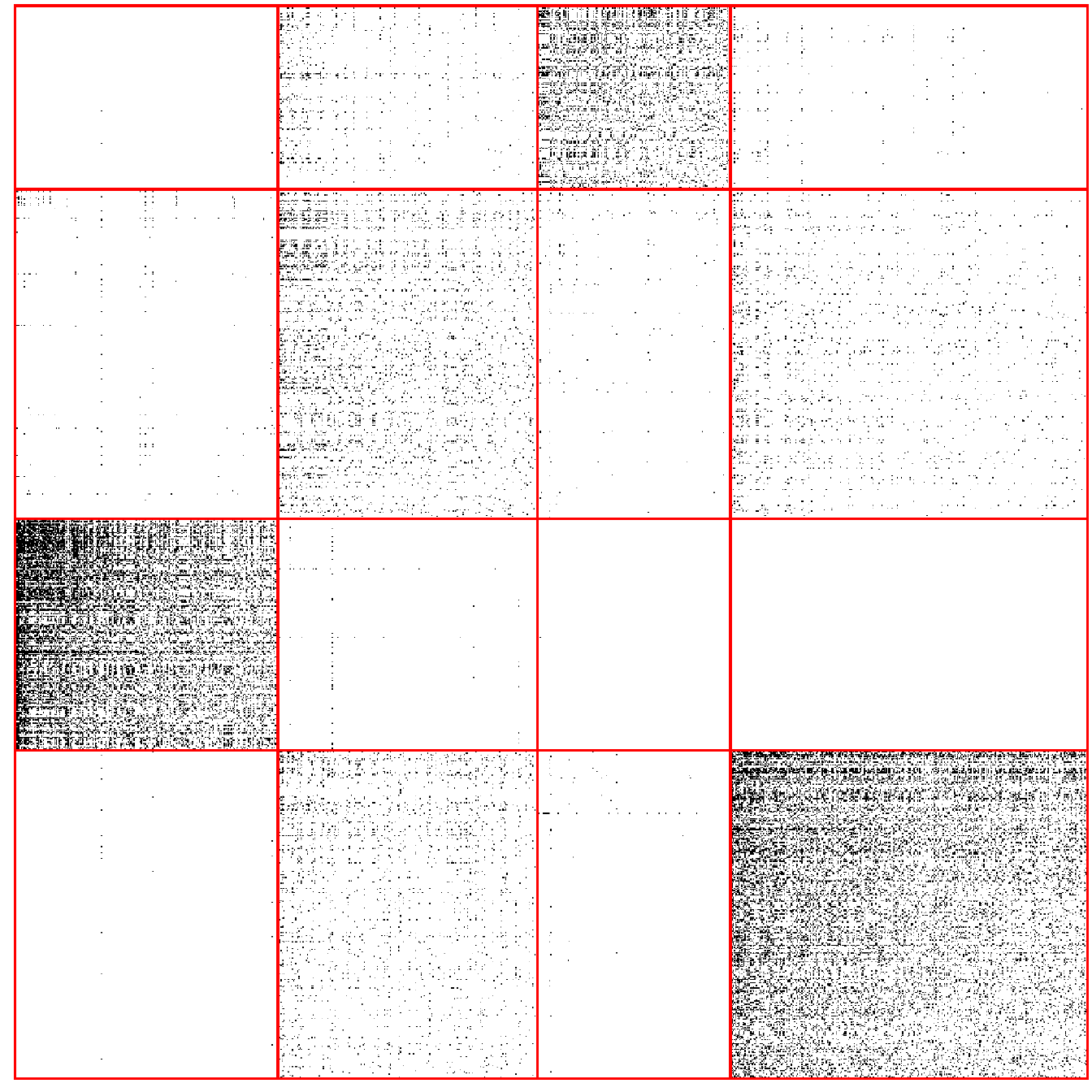}
  \includegraphics[width = 0.17 \textwidth, viewport = 110 210 497 597, clip]{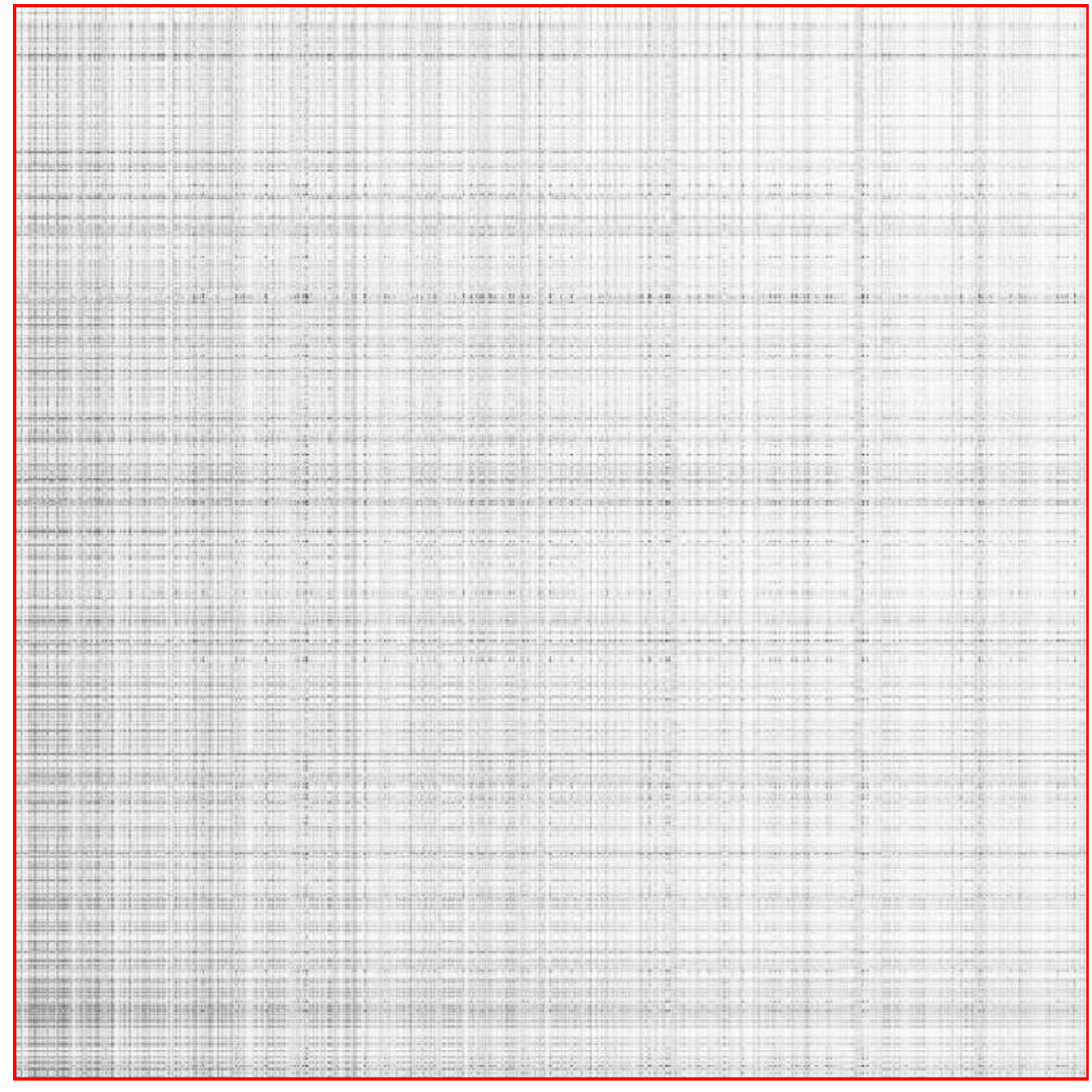}
  \includegraphics[width = 0.17 \textwidth, viewport = 110 210 497 597, clip]{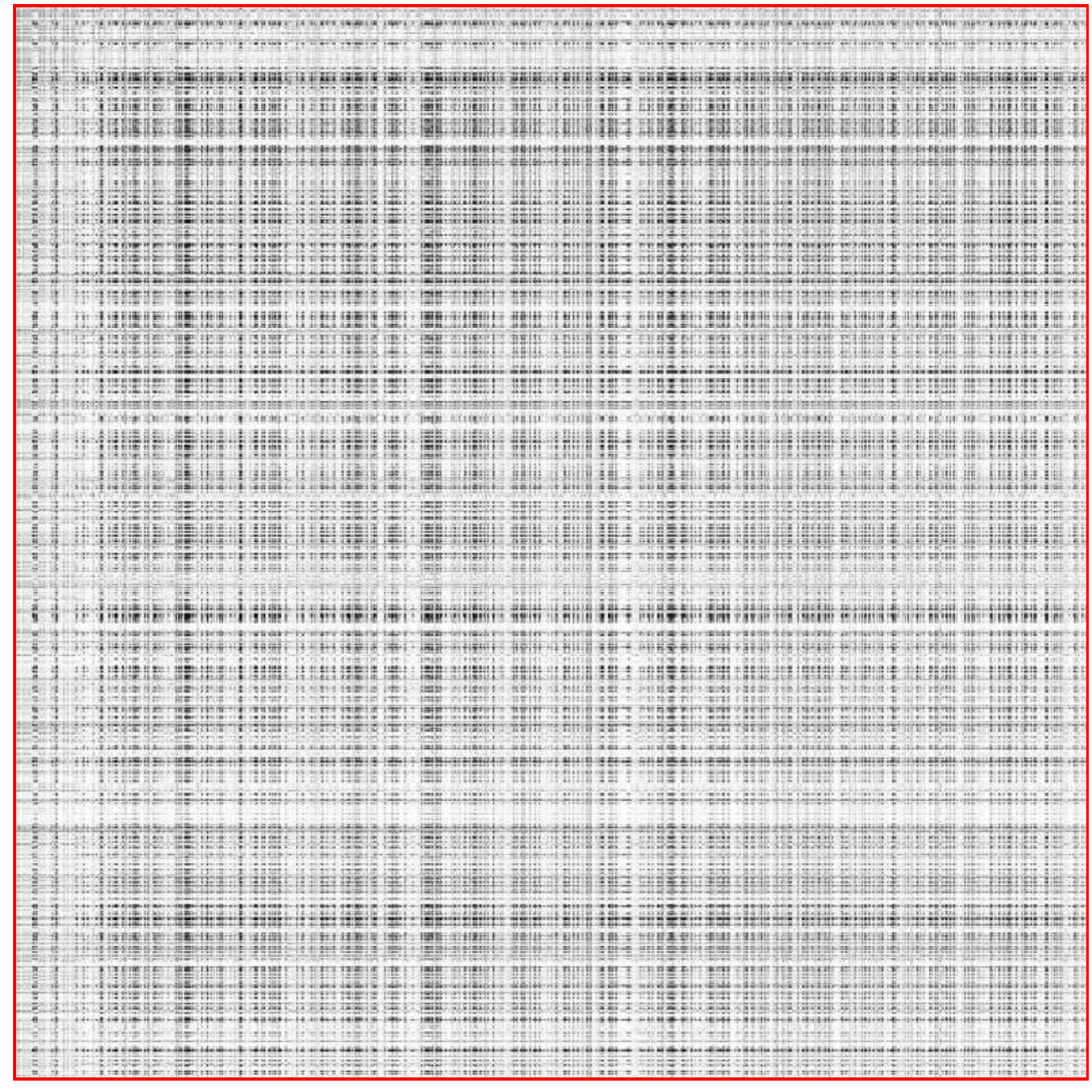}
  \includegraphics[width = 0.17 \textwidth, viewport = 110 210 497 597, clip]{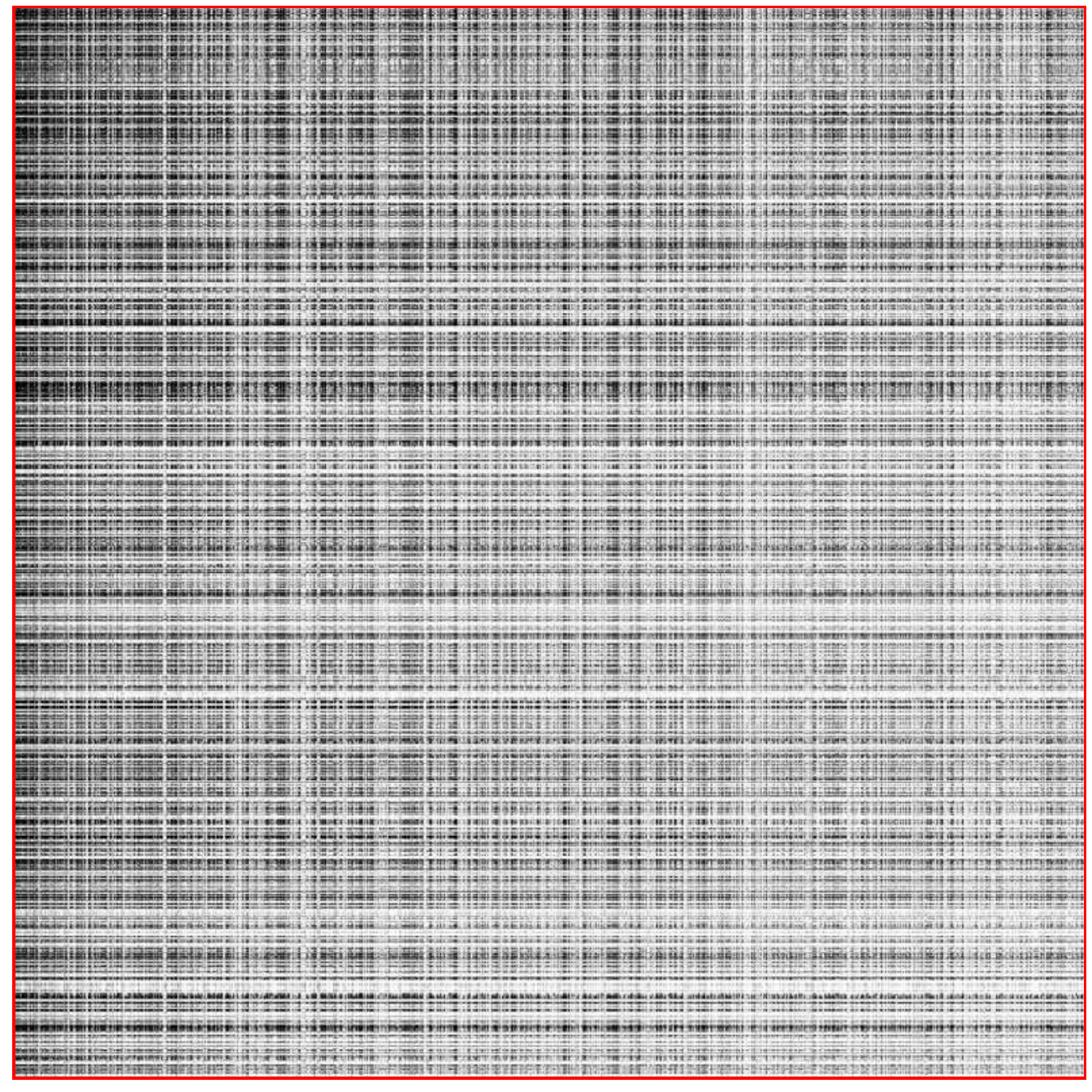}
  \includegraphics[width = 0.17 \textwidth, viewport = 110 210 497 597, clip]{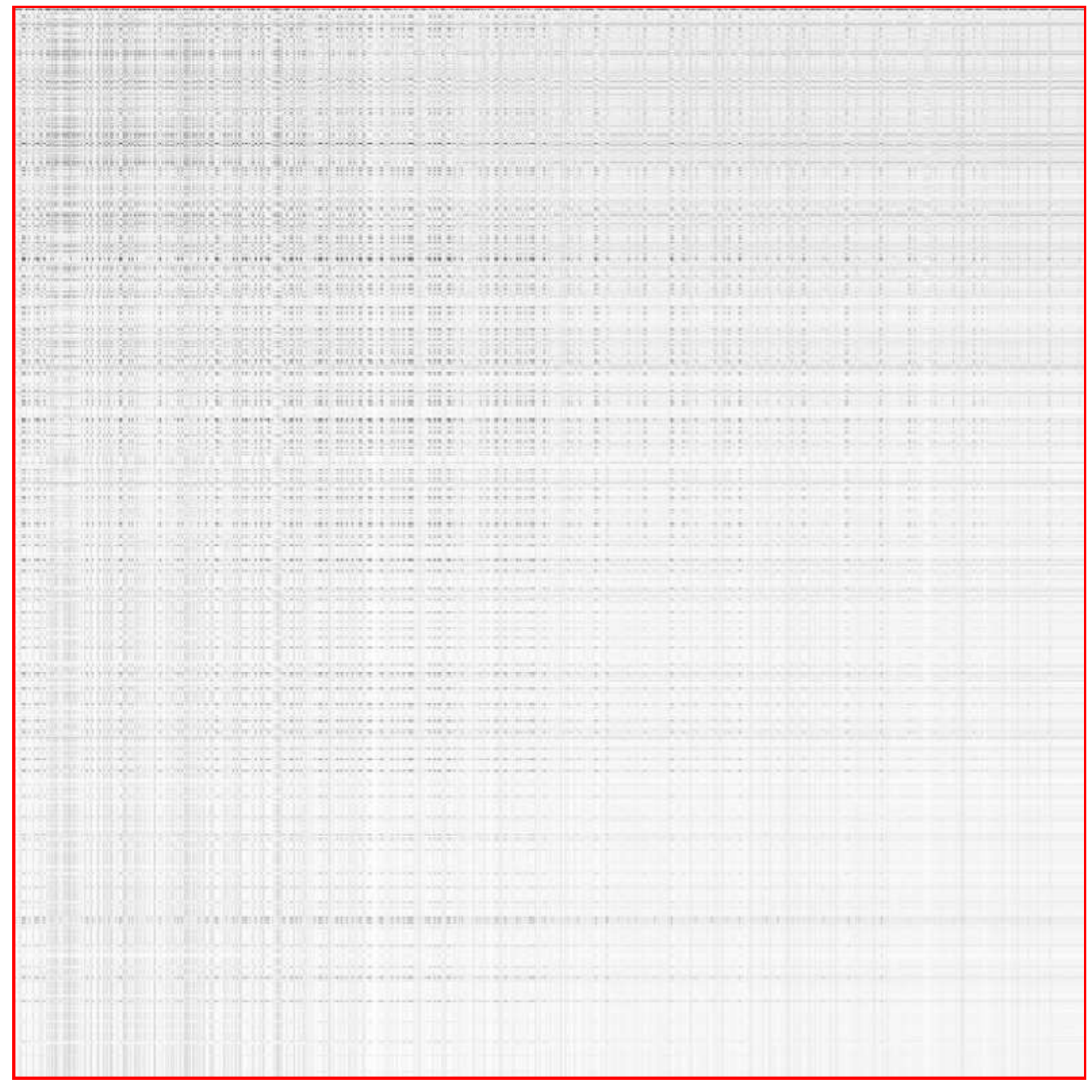}
  \includegraphics[width = 0.17 \textwidth, viewport = 110 210 497 597, clip]{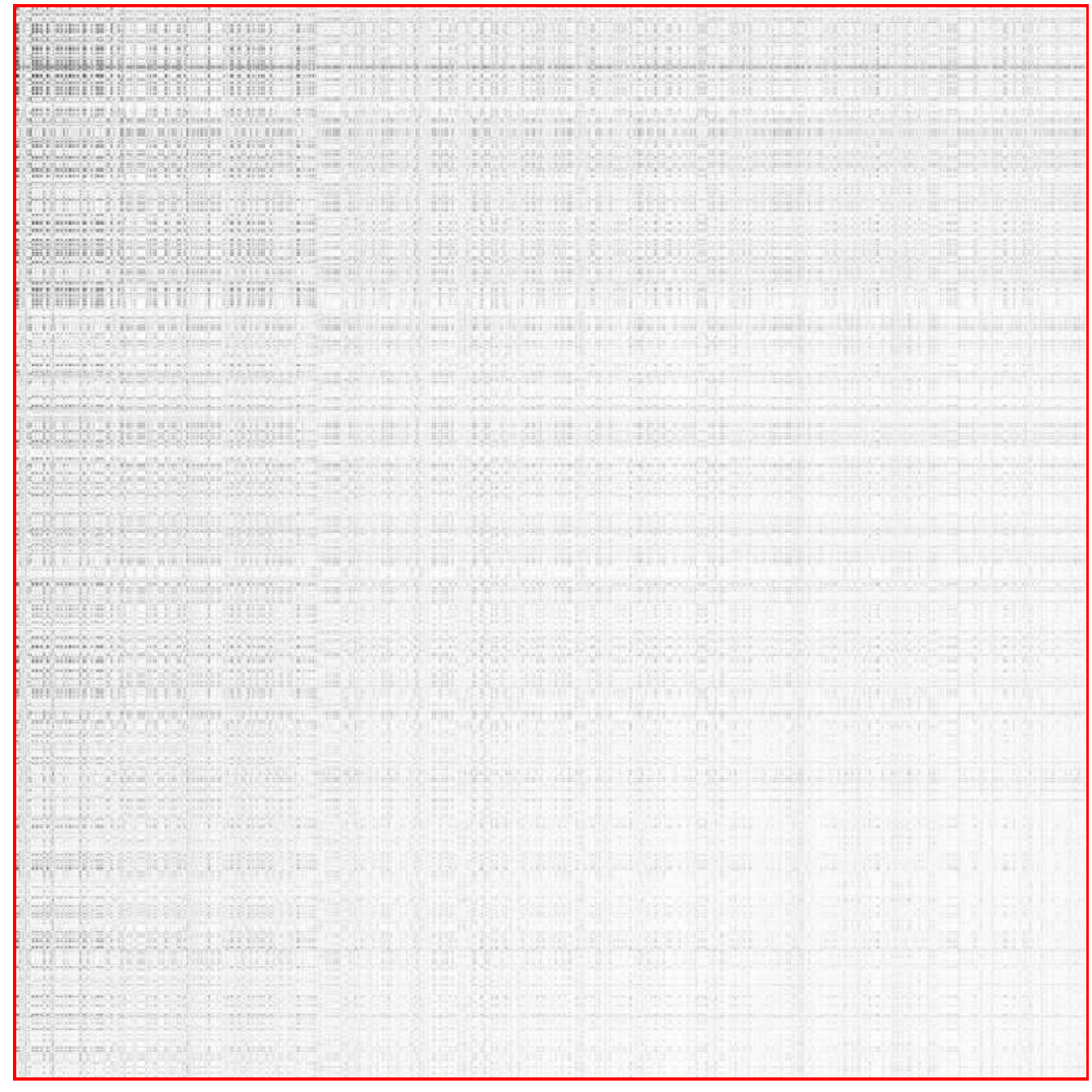}
  \includegraphics[width = 0.17 \textwidth, viewport = 110 210 497 597, clip]{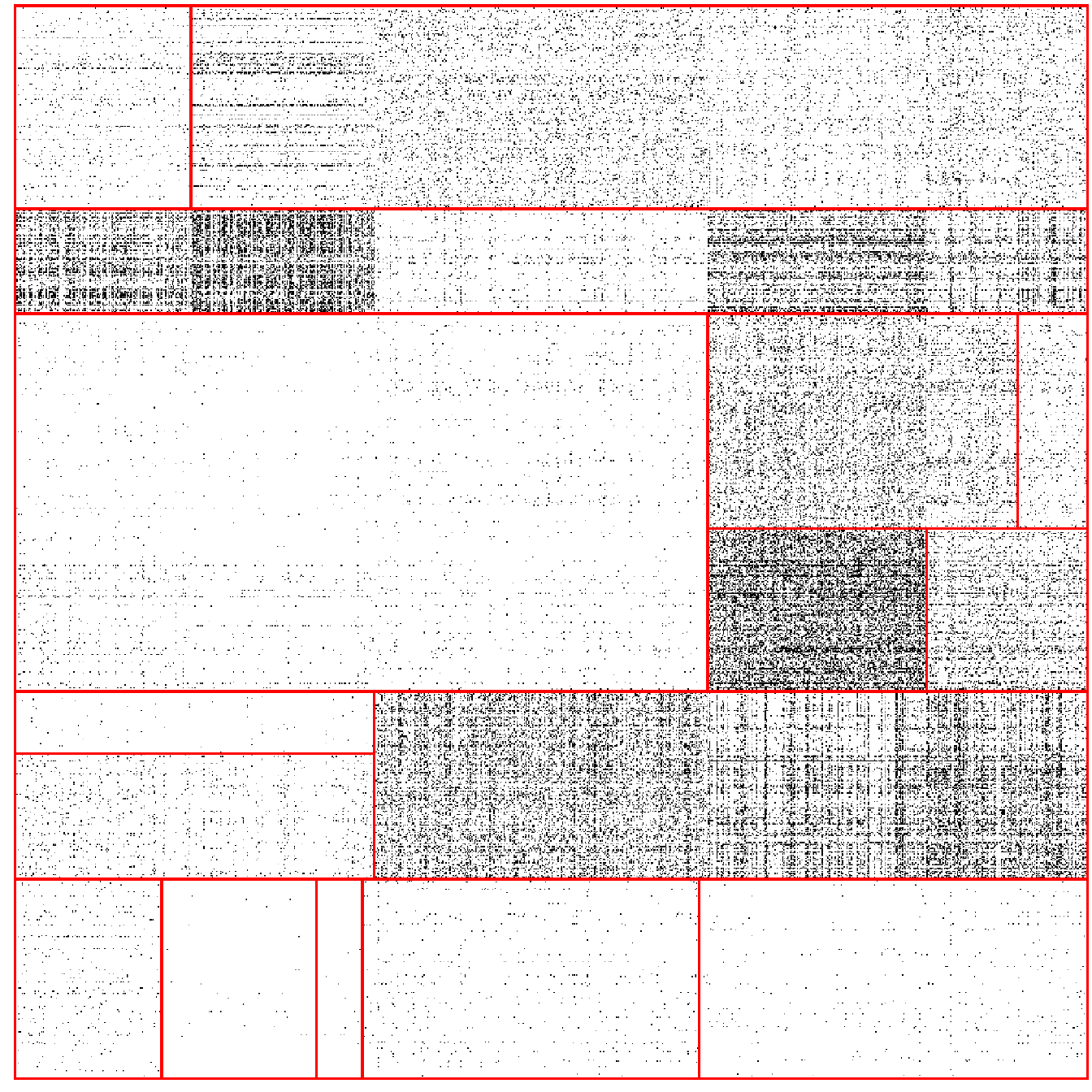}
  \includegraphics[width = 0.17 \textwidth, viewport = 110 210 497 597, clip]{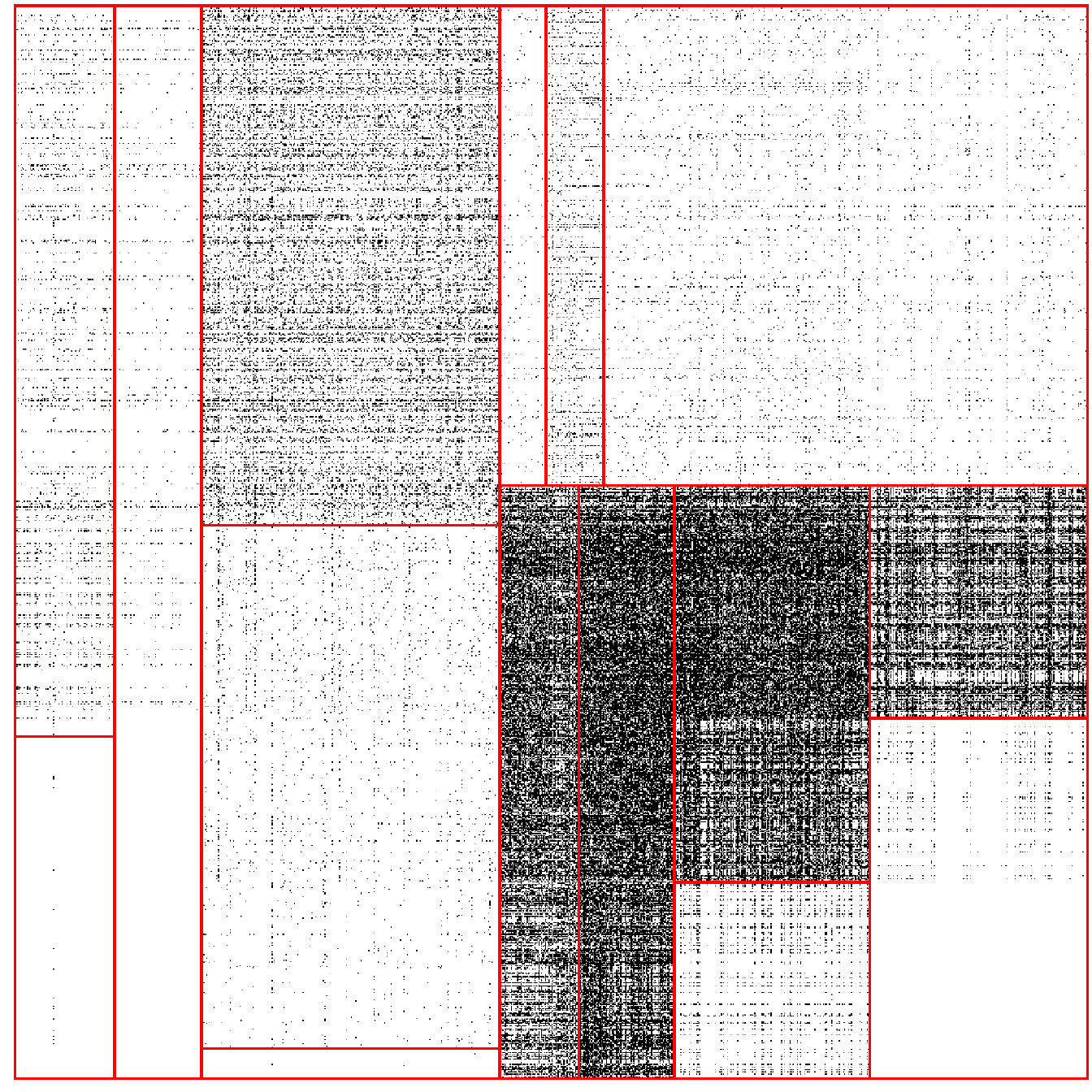}
  \includegraphics[width = 0.17 \textwidth, viewport = 110 210 497 597, clip]{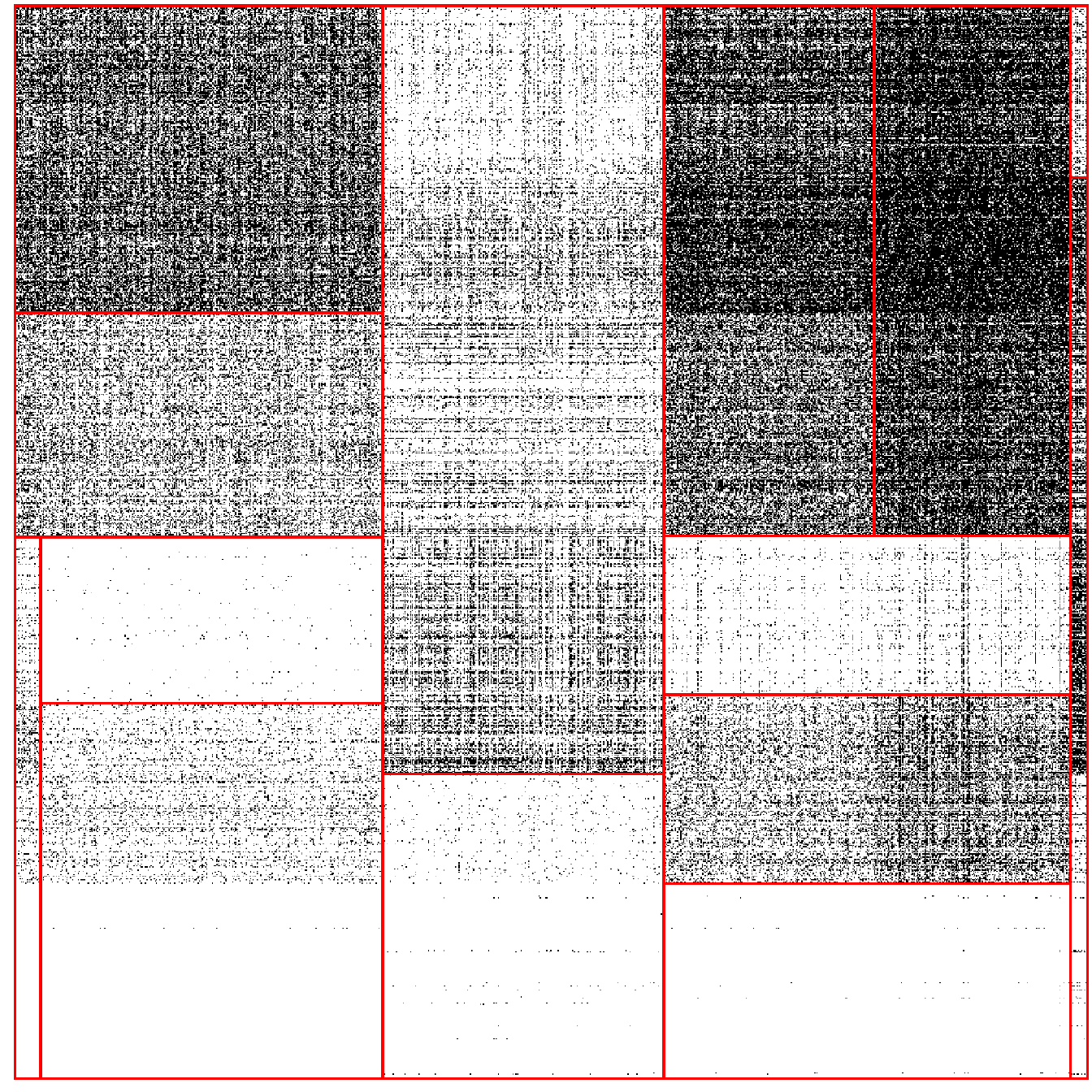}
  \includegraphics[width = 0.17 \textwidth, viewport = 110 210 497 597, clip]{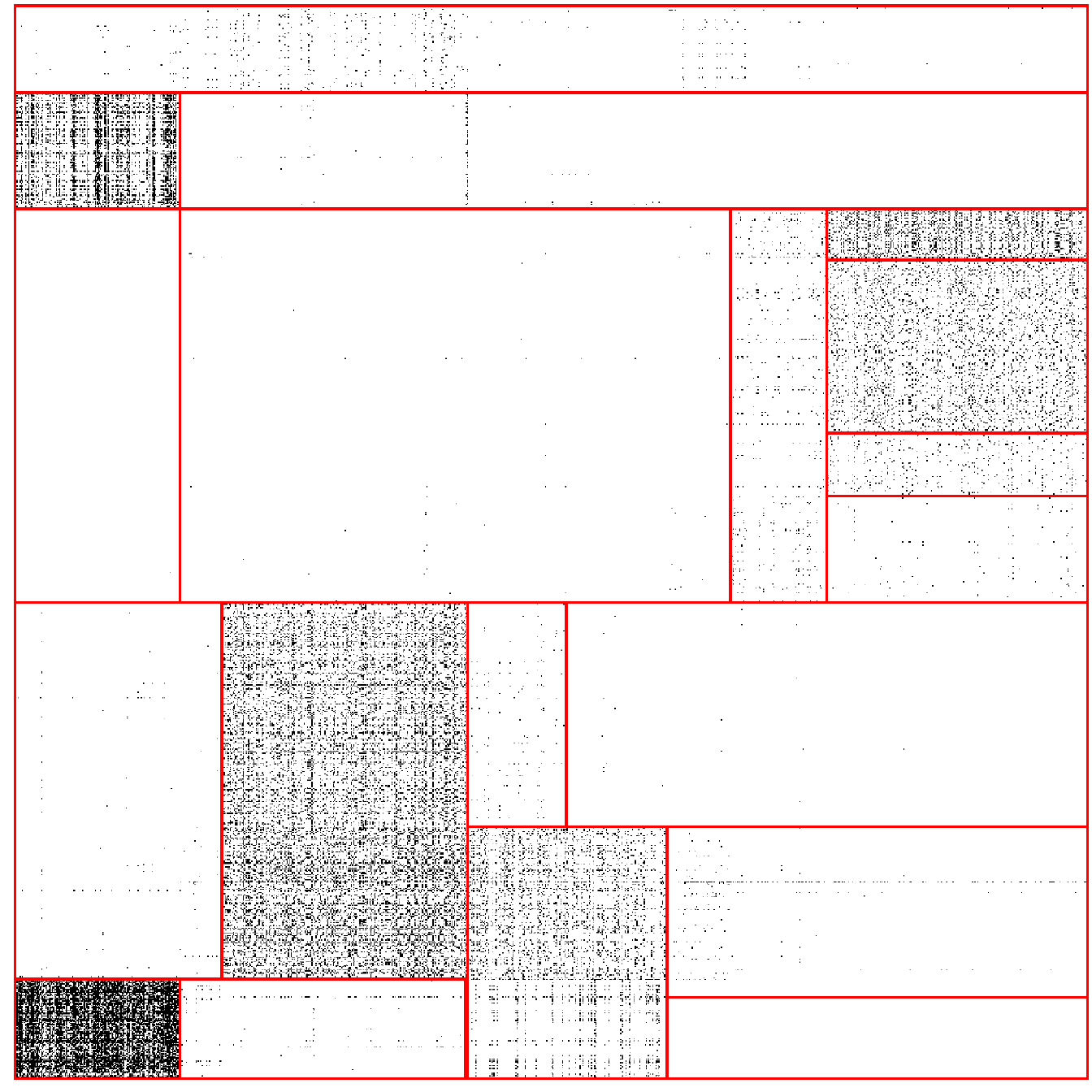}
  \includegraphics[width = 0.17 \textwidth, viewport = 110 210 497 597, clip]{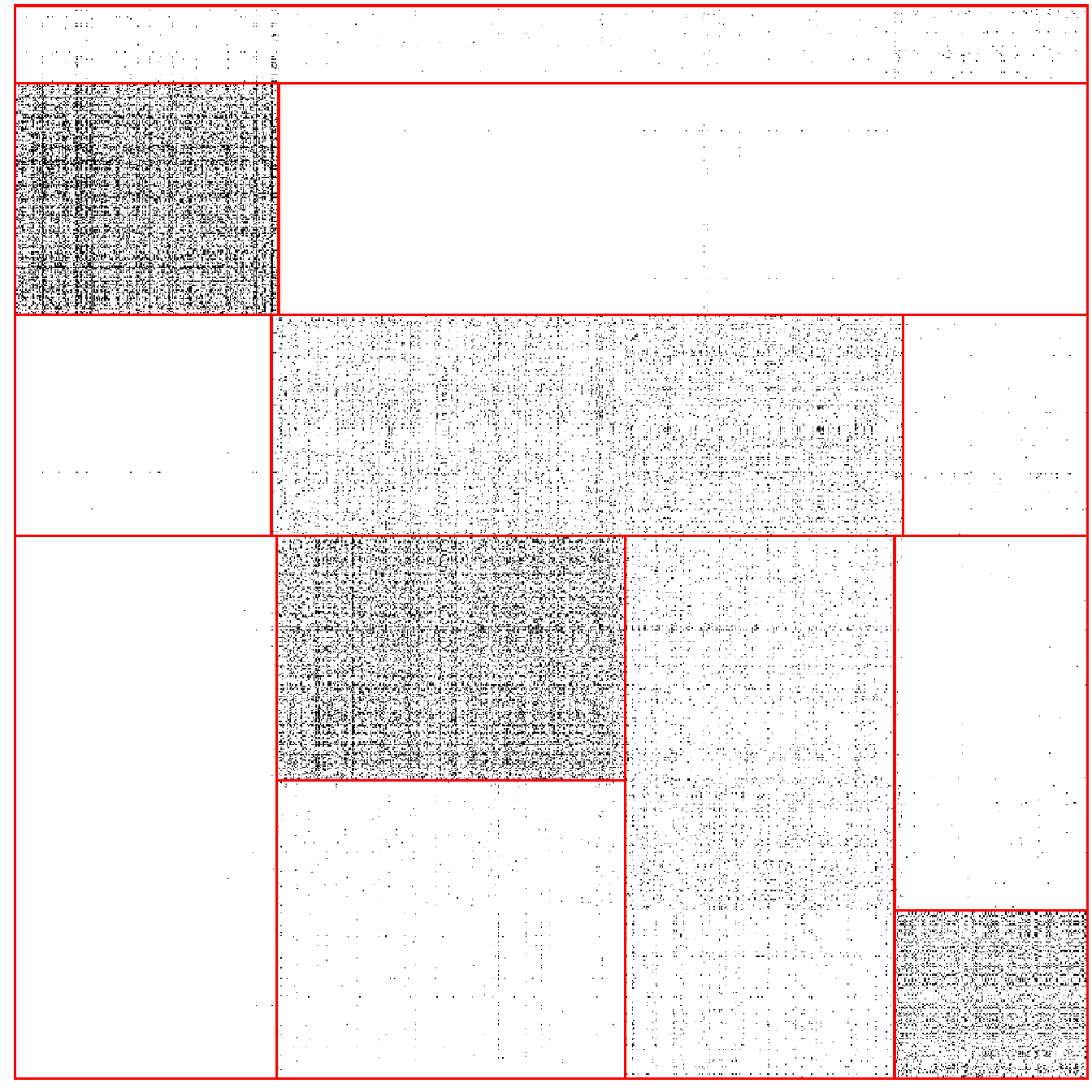}
  \includegraphics[width = 0.17 \textwidth, viewport = 110 210 497 597, clip]{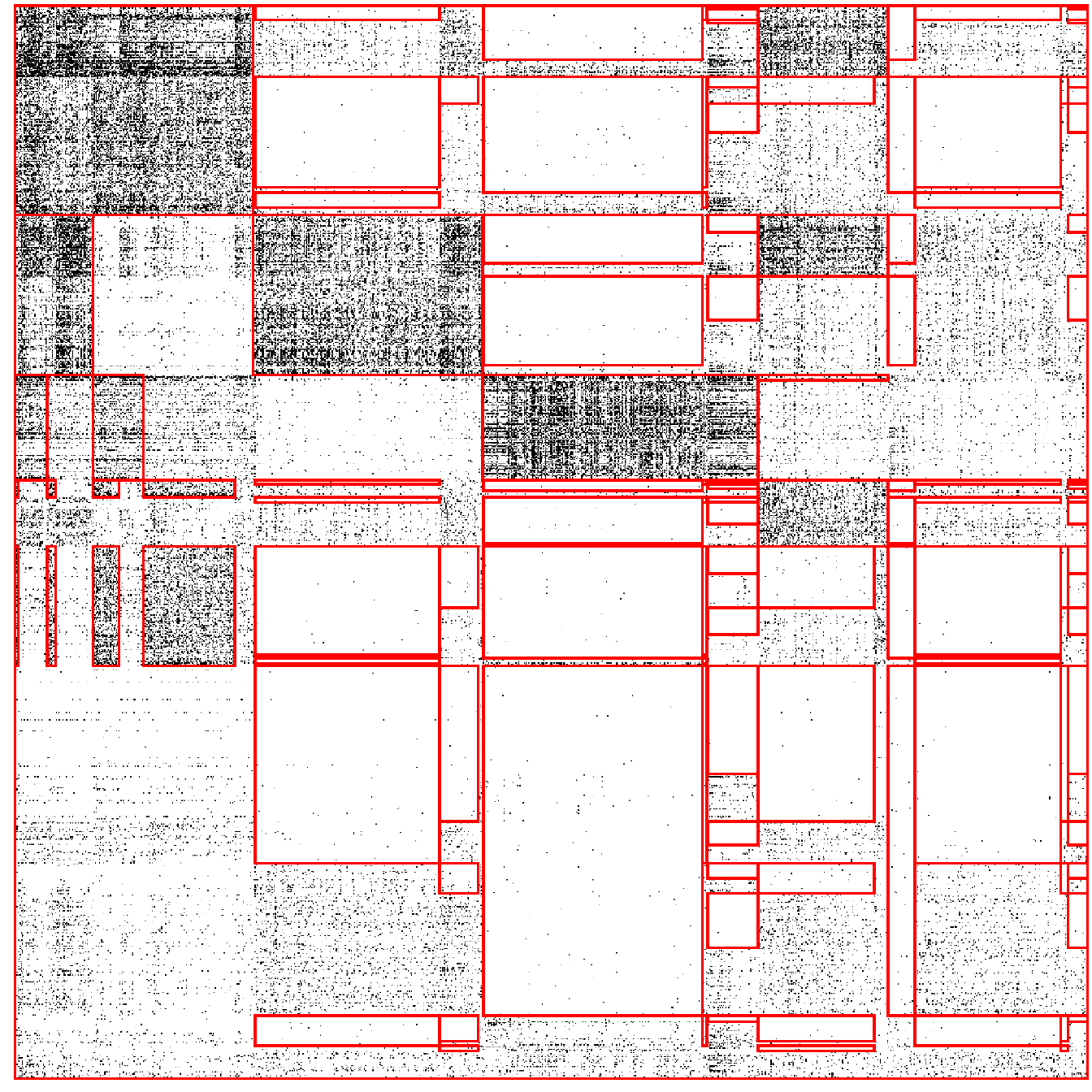}
  \includegraphics[width = 0.17 \textwidth, viewport = 110 210 497 597, clip]{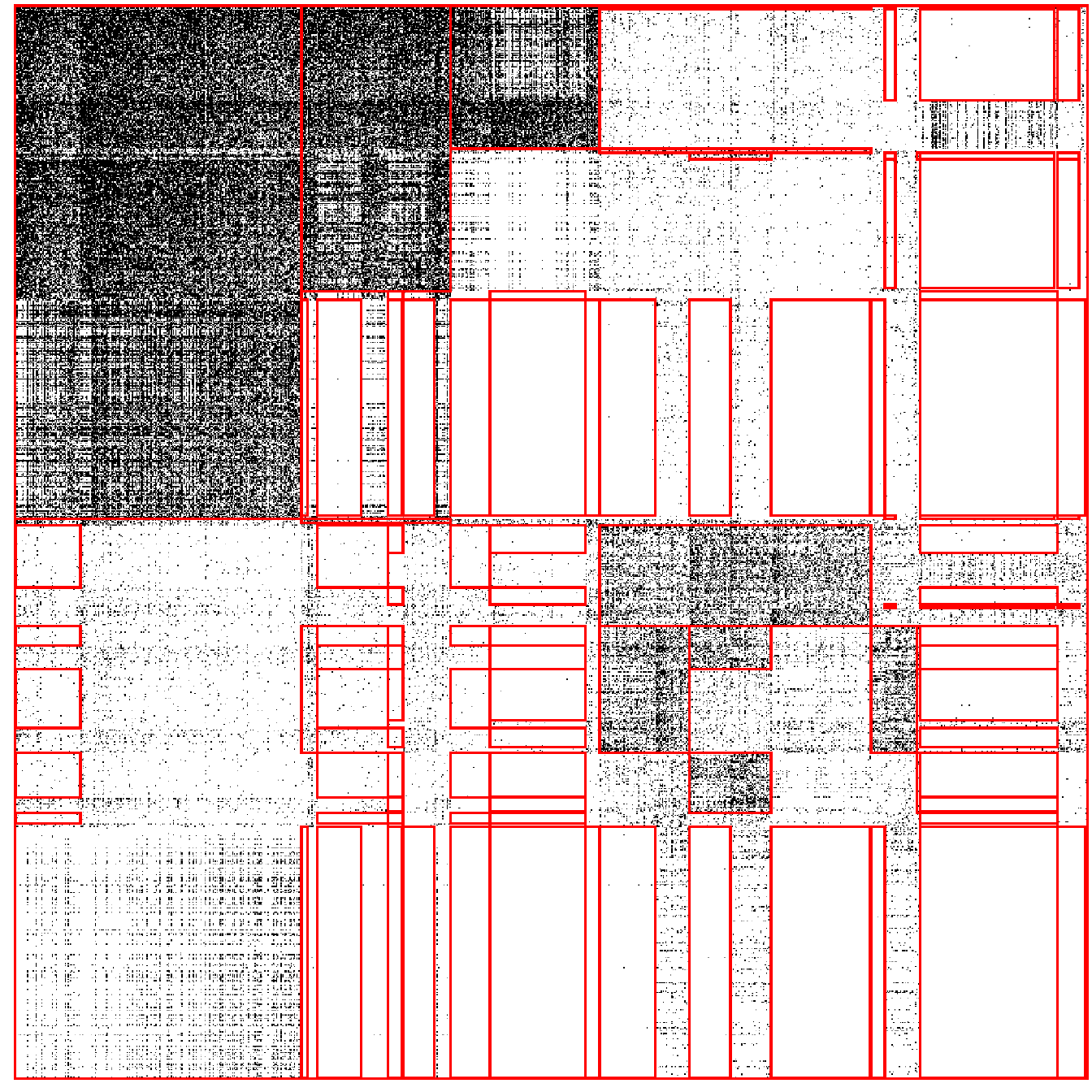}
  \includegraphics[width = 0.17 \textwidth, viewport = 110 210 497 597, clip]{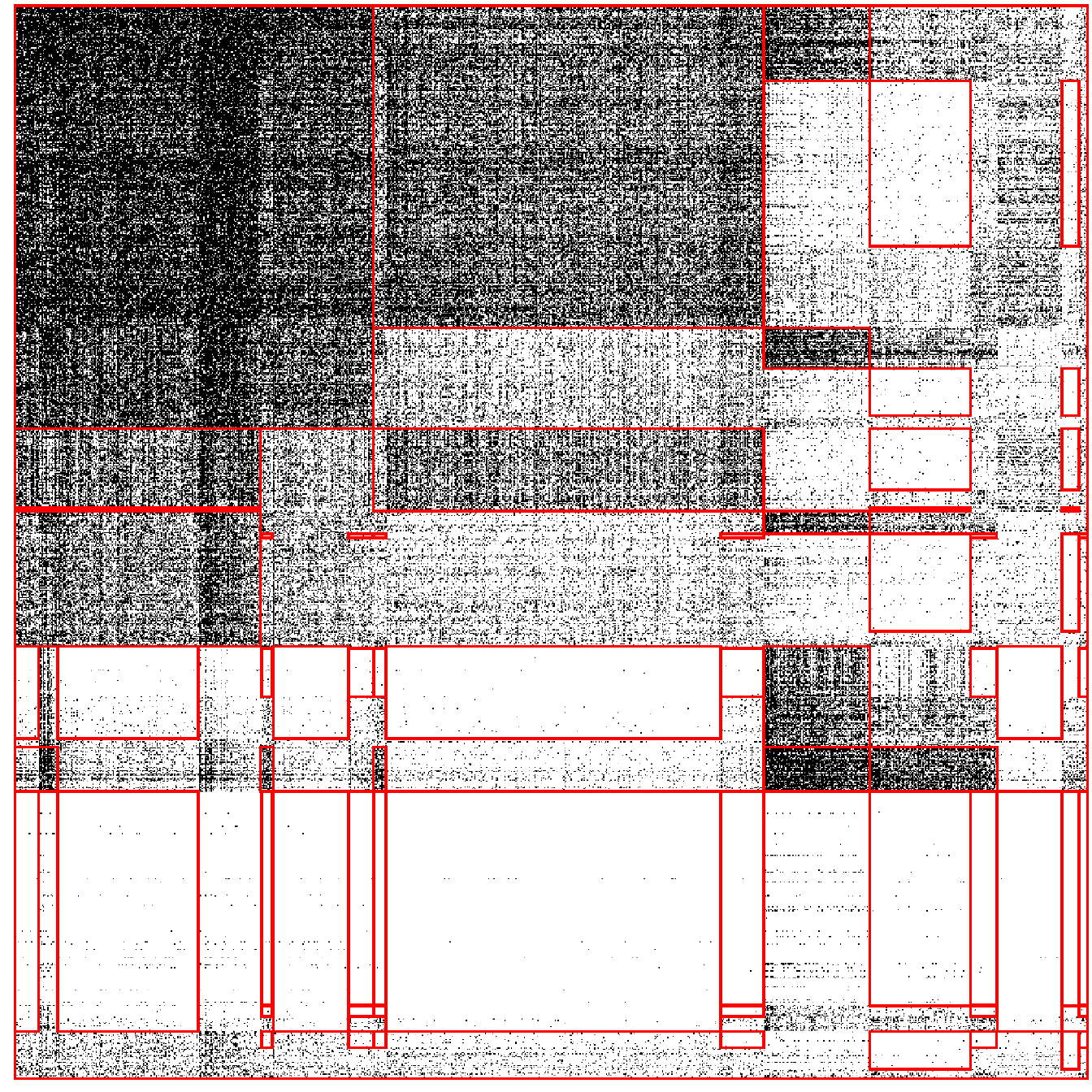}
  \includegraphics[width = 0.17 \textwidth, viewport = 110 210 497 597, clip]{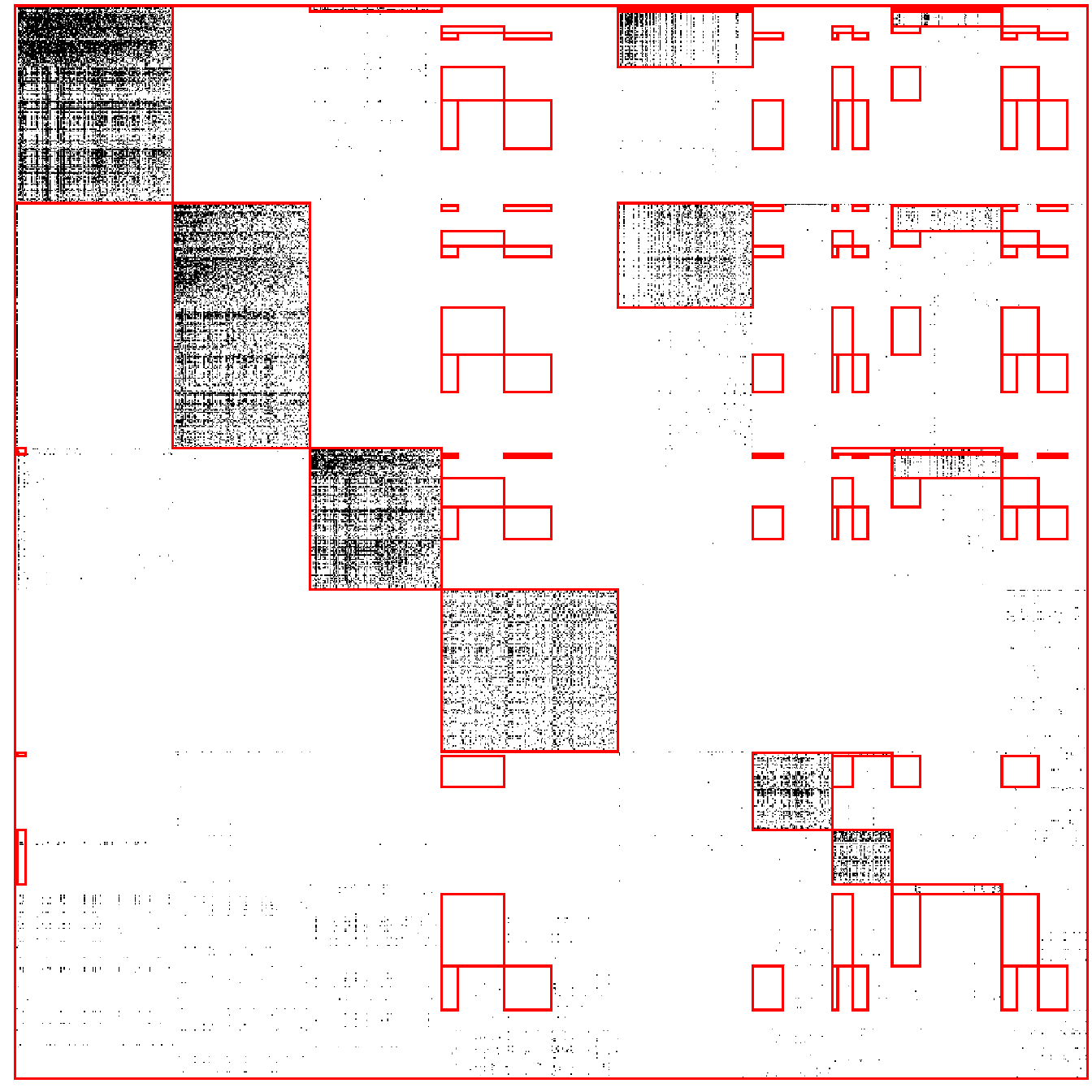}
  \includegraphics[width = 0.17 \textwidth, viewport = 110 210 497 597, clip]{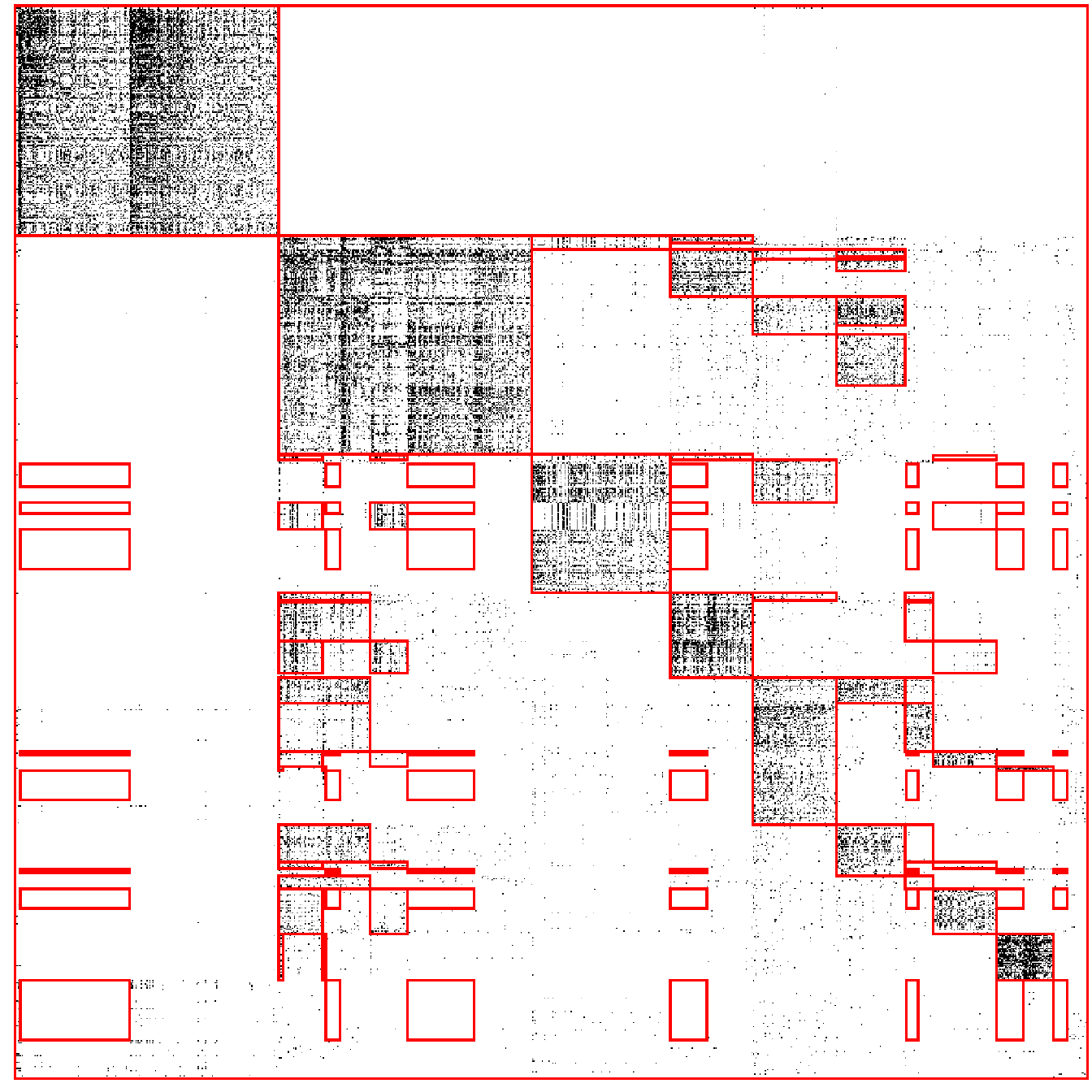}
  \includegraphics[width = 0.17 \textwidth, viewport = 30 20 300 300, clip]{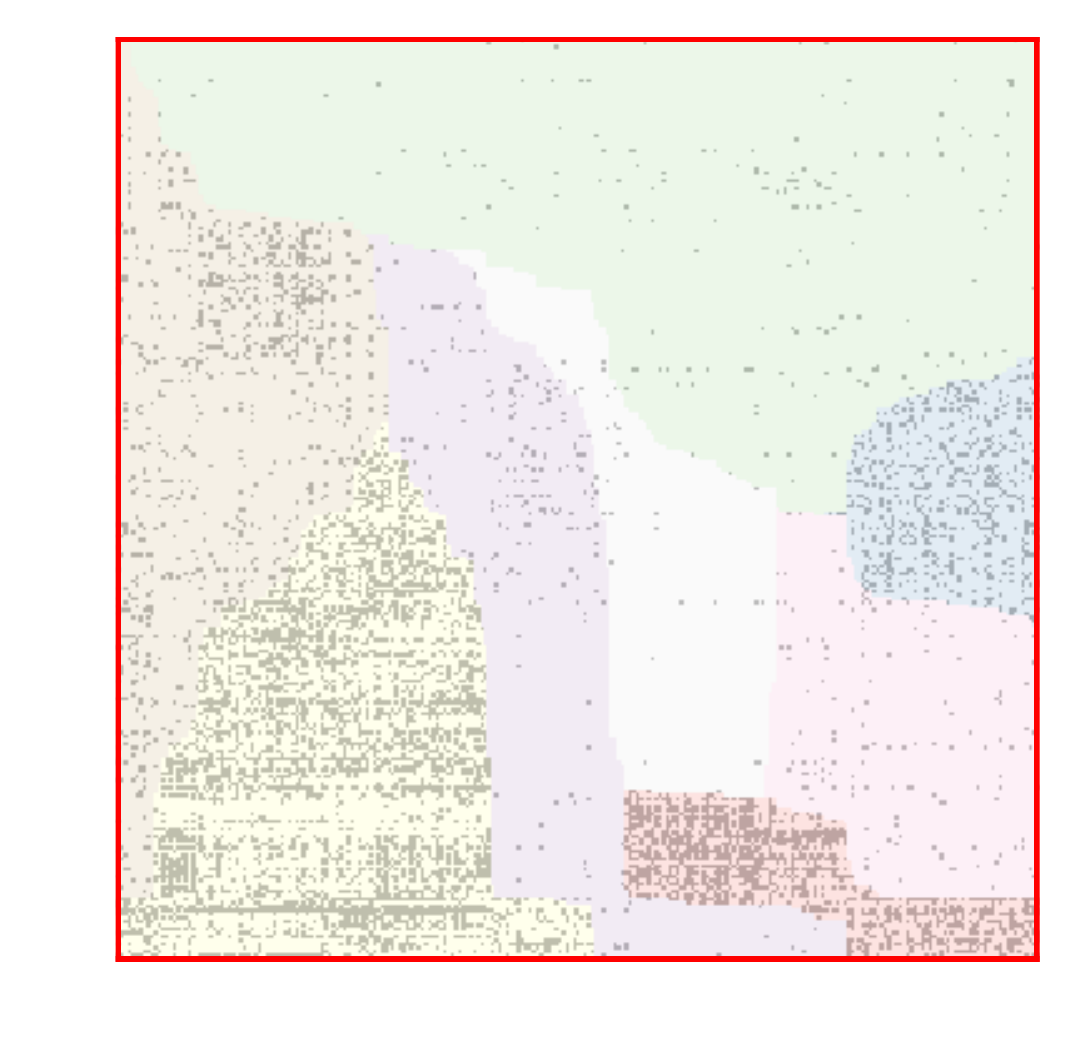}
  \includegraphics[width = 0.17 \textwidth, viewport = 30 20 300 300, clip]{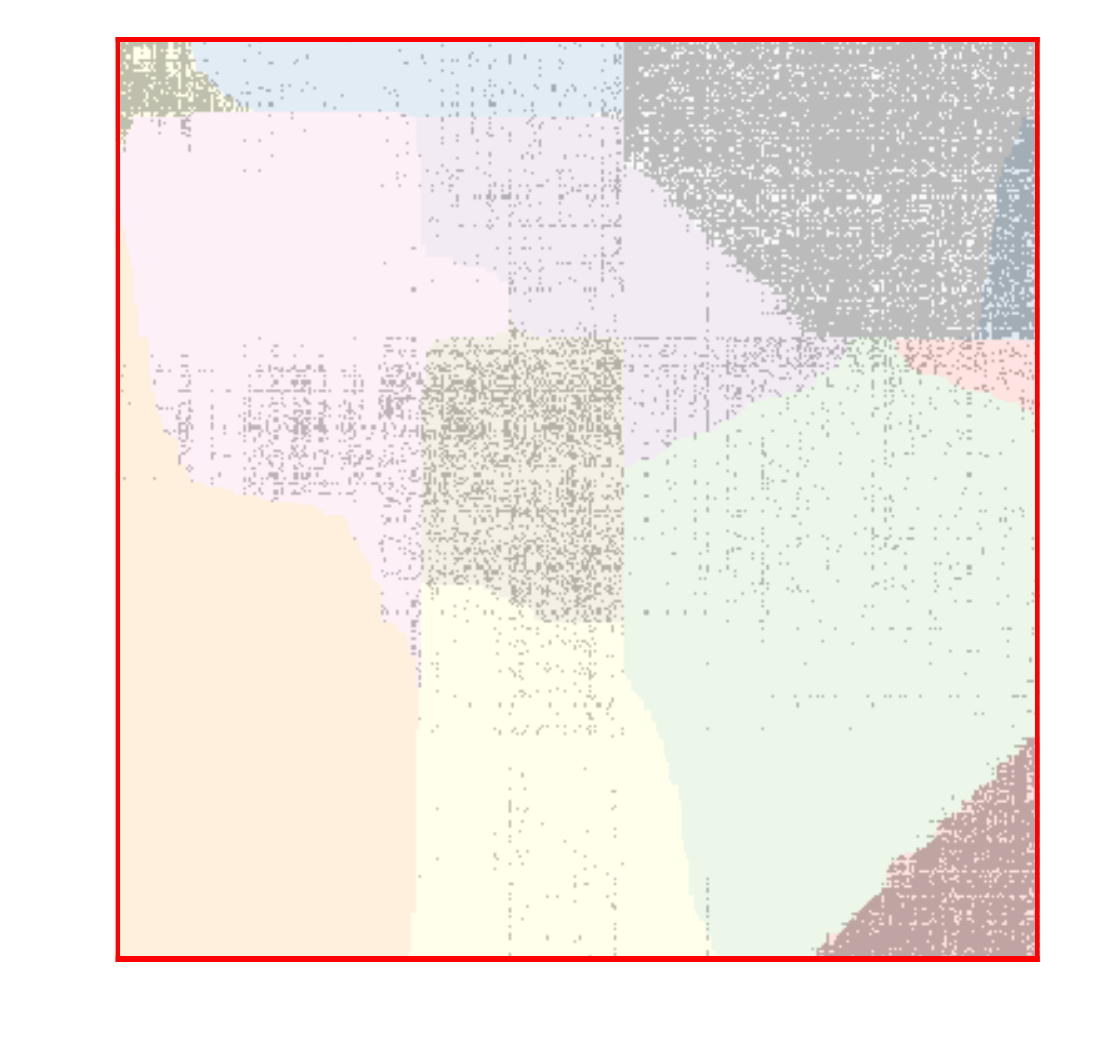}
  \includegraphics[width = 0.17 \textwidth, viewport = 30 20 300 300, clip]{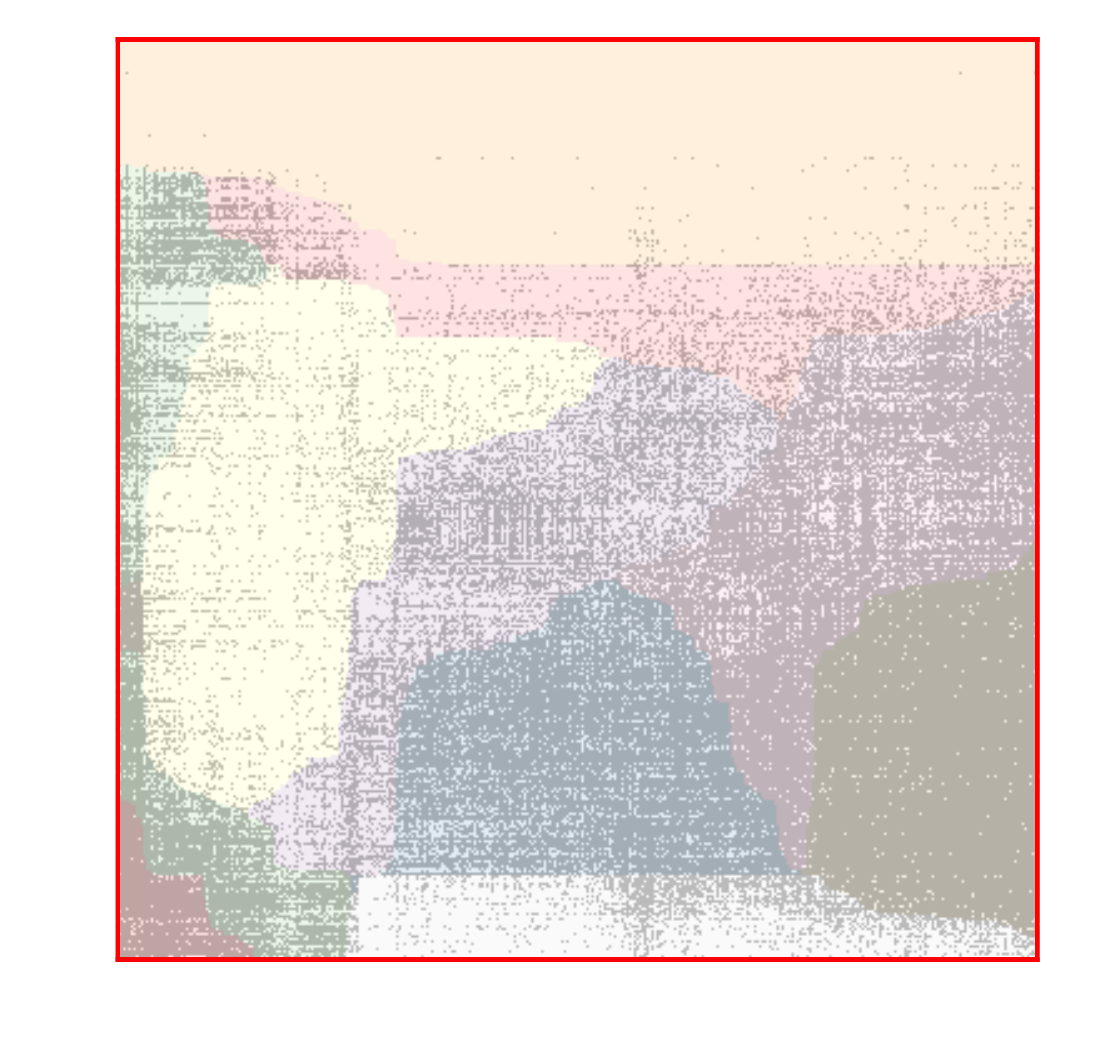}
  \includegraphics[width = 0.17 \textwidth, viewport = 30 20 300 300, clip]{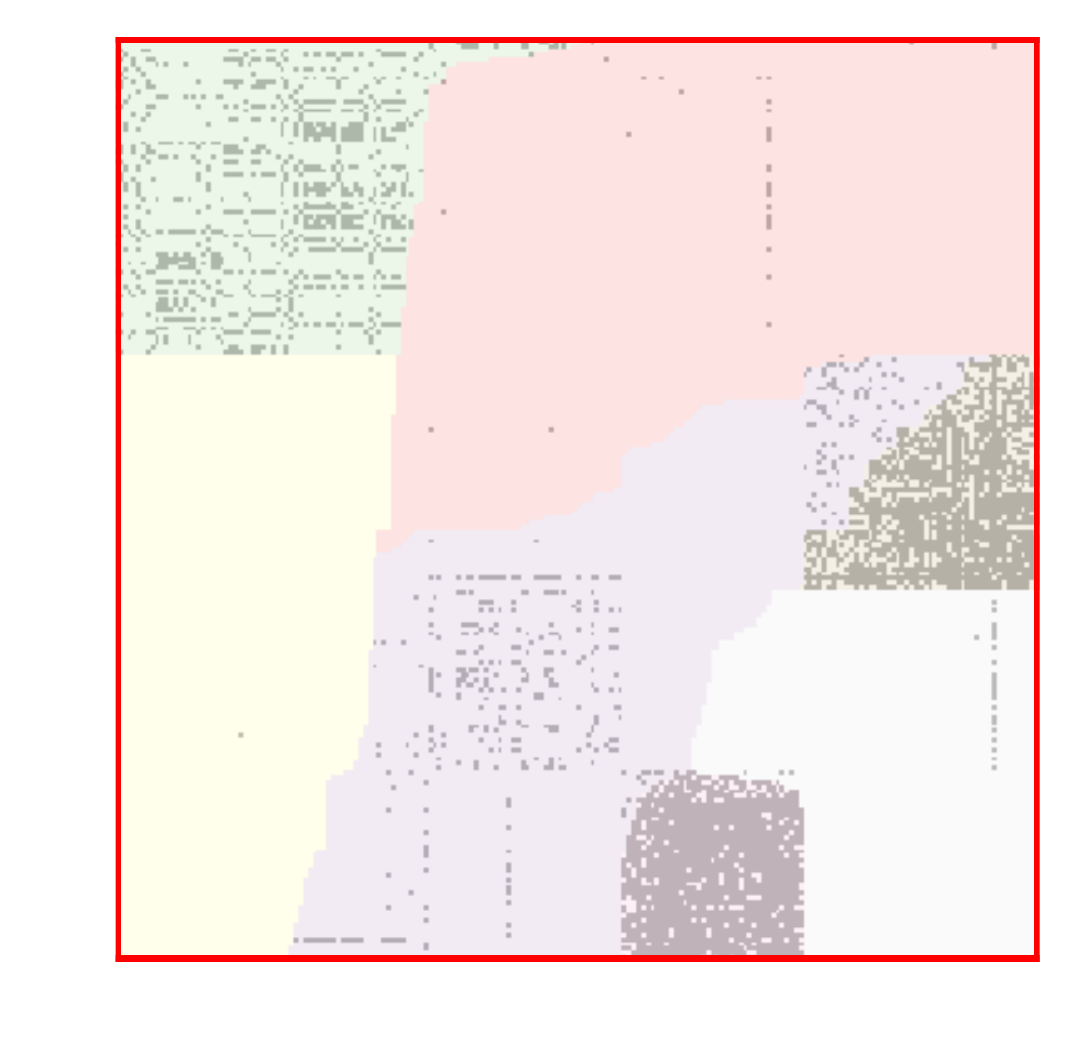}
  \includegraphics[width = 0.17 \textwidth, viewport = 30 20 300 300, clip]{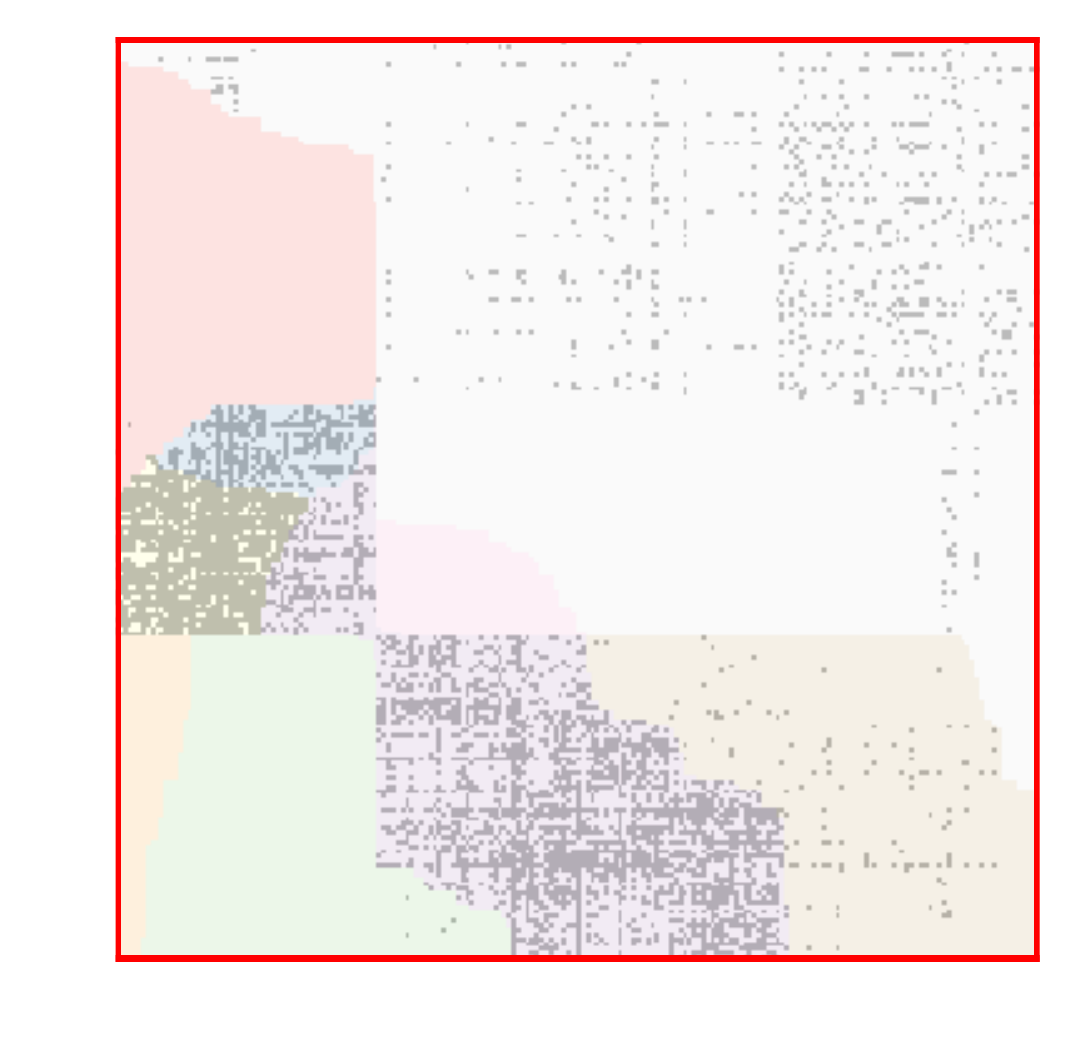}
  \includegraphics[width = 0.17 \textwidth, viewport = 115 275 479 569, clip]{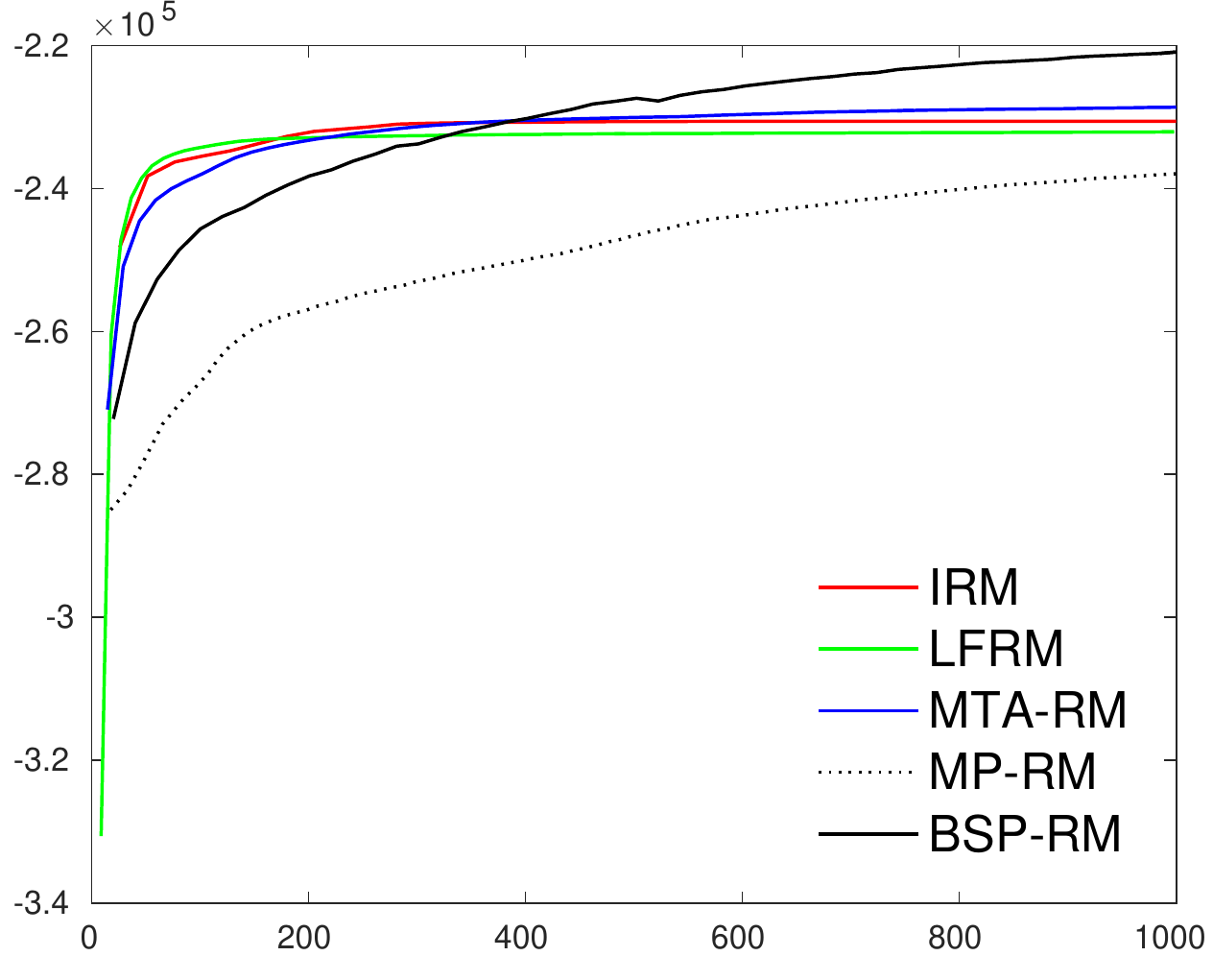}
  \includegraphics[width = 0.17 \textwidth, viewport = 115 275 479 569, clip]{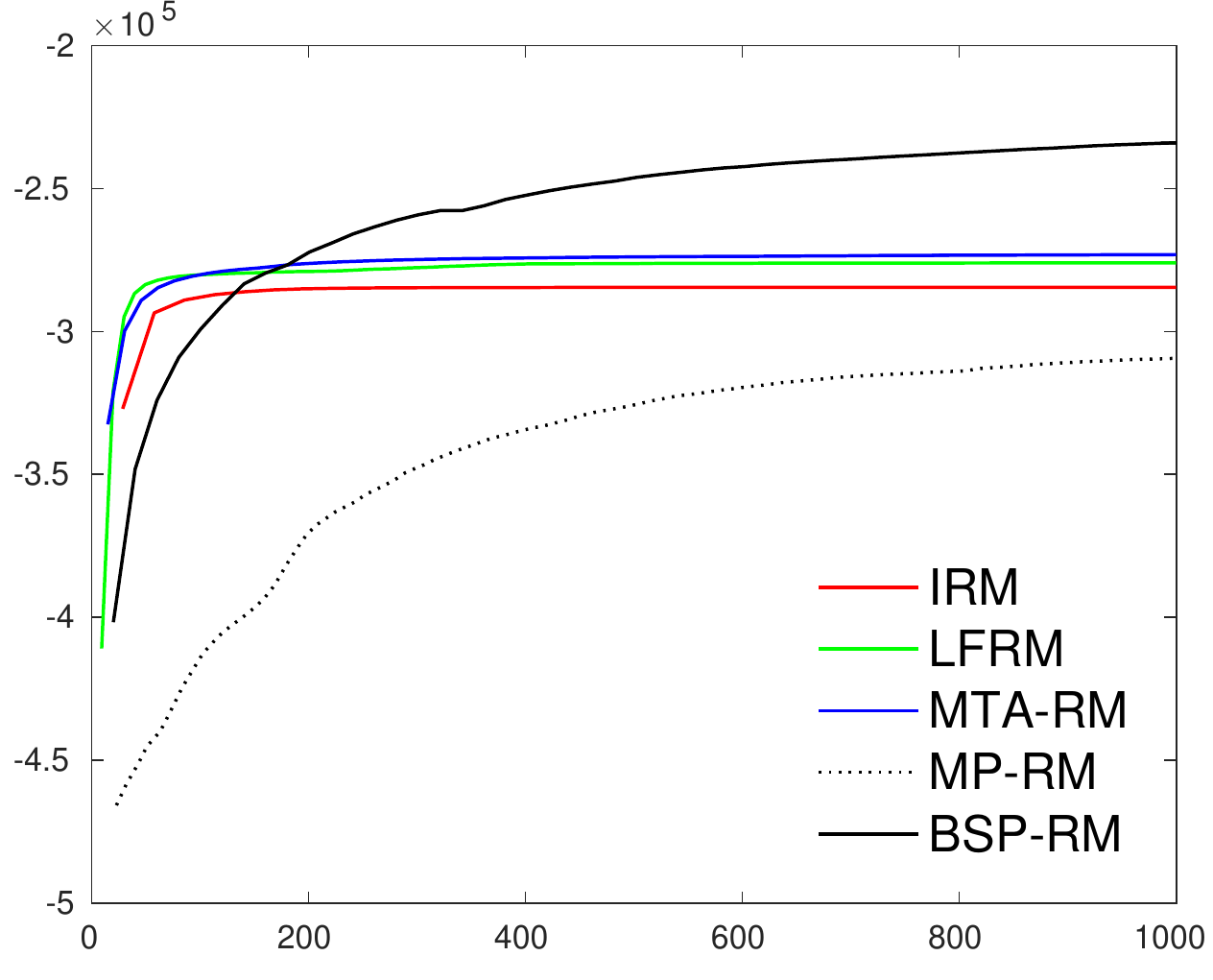}
  \includegraphics[width = 0.17 \textwidth, viewport = 115 275 479 569, clip]{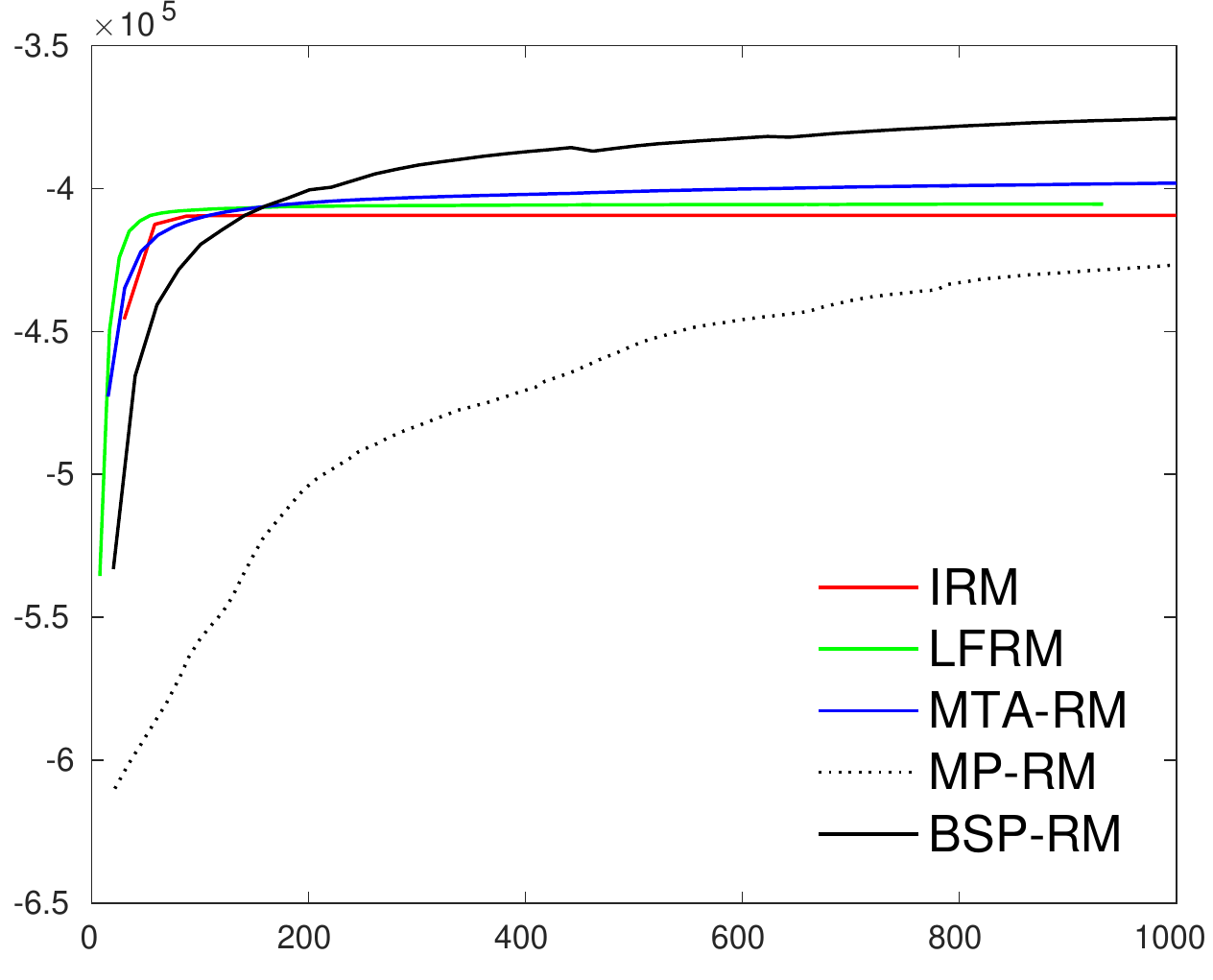}
  \includegraphics[width = 0.17 \textwidth, viewport = 115 275 479 569, clip]{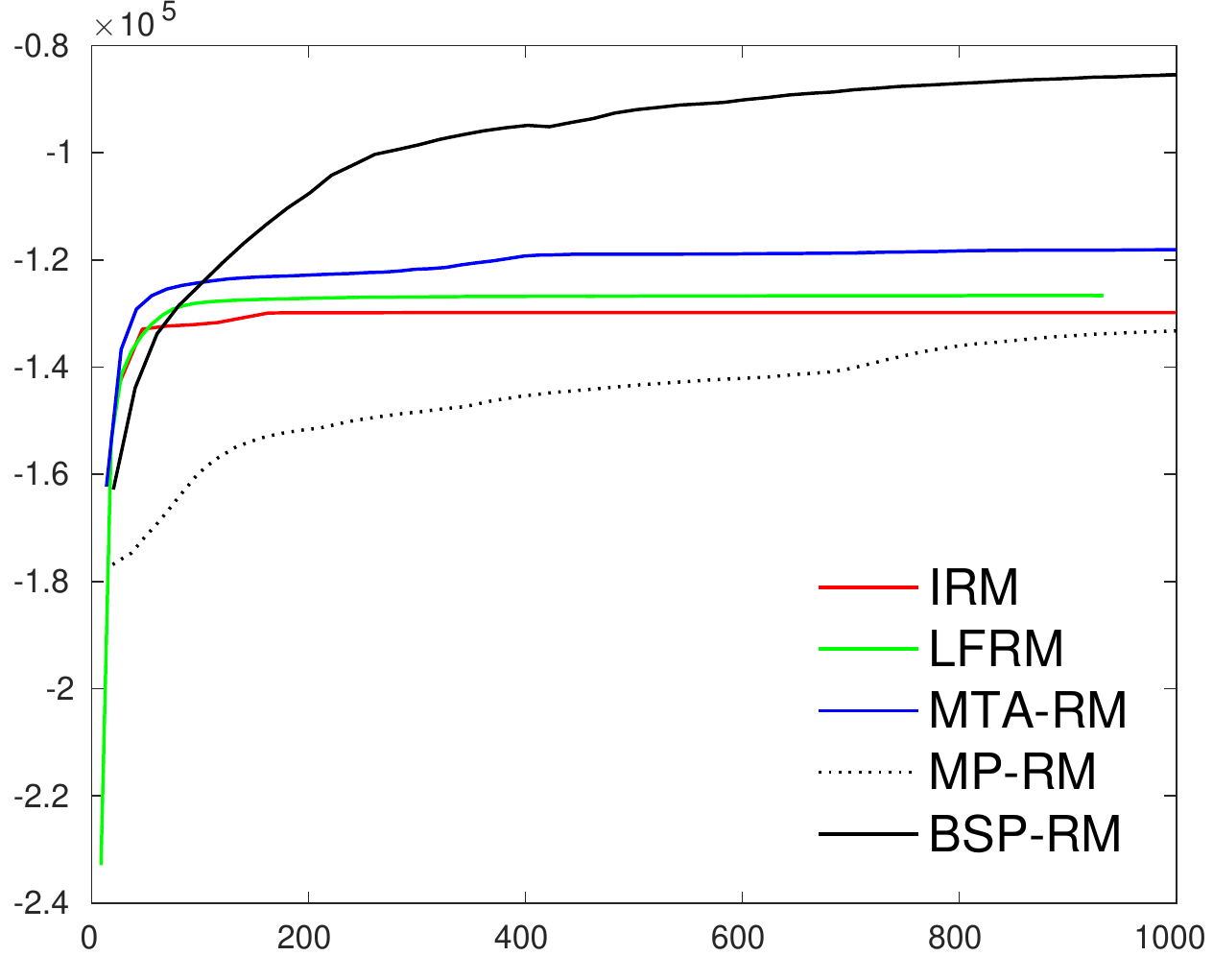}
  \includegraphics[width = 0.17 \textwidth, viewport = 115 275 479 569, clip]{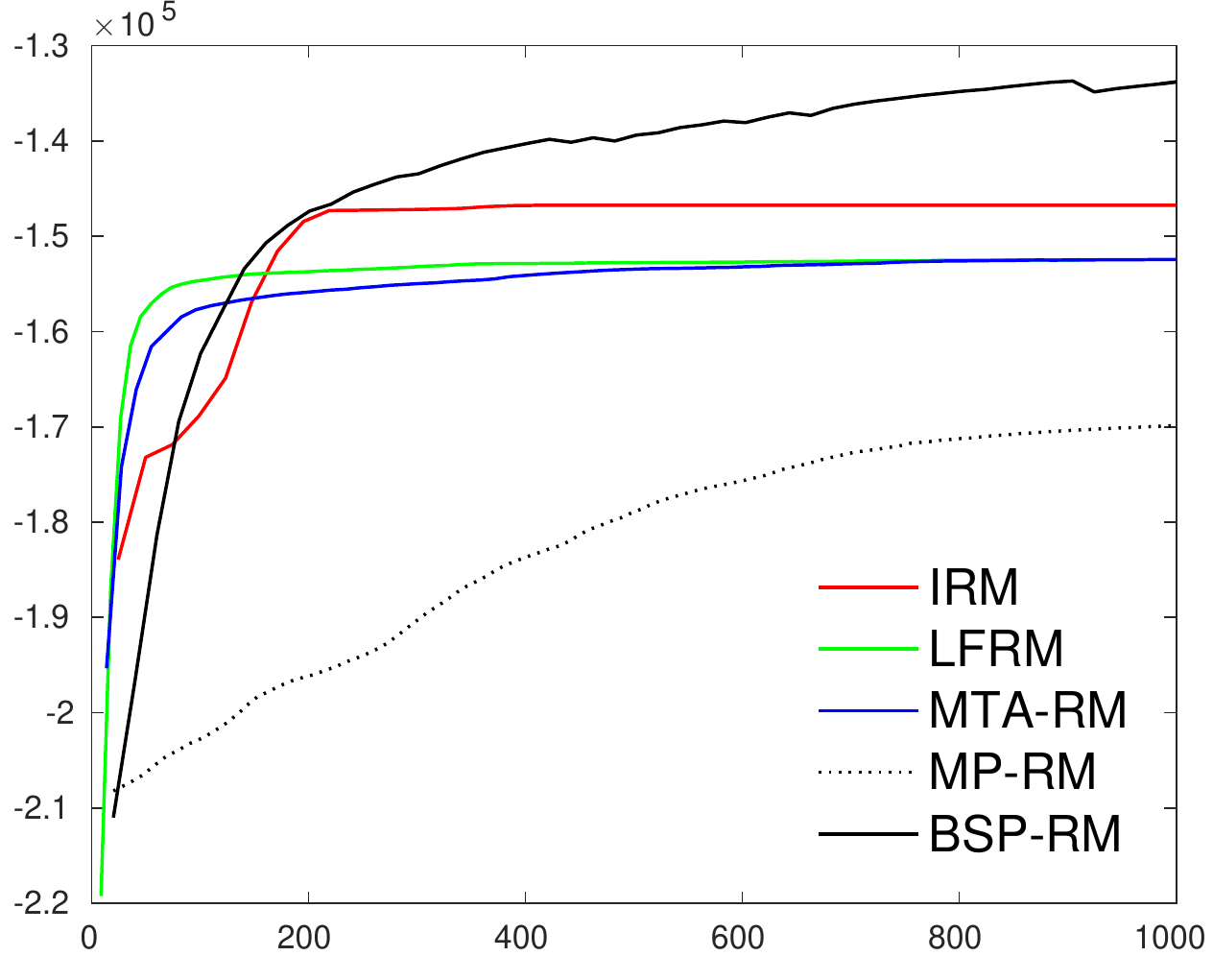}
  \includegraphics[width = 0.17 \textwidth, viewport = 115 275 479 562, clip]{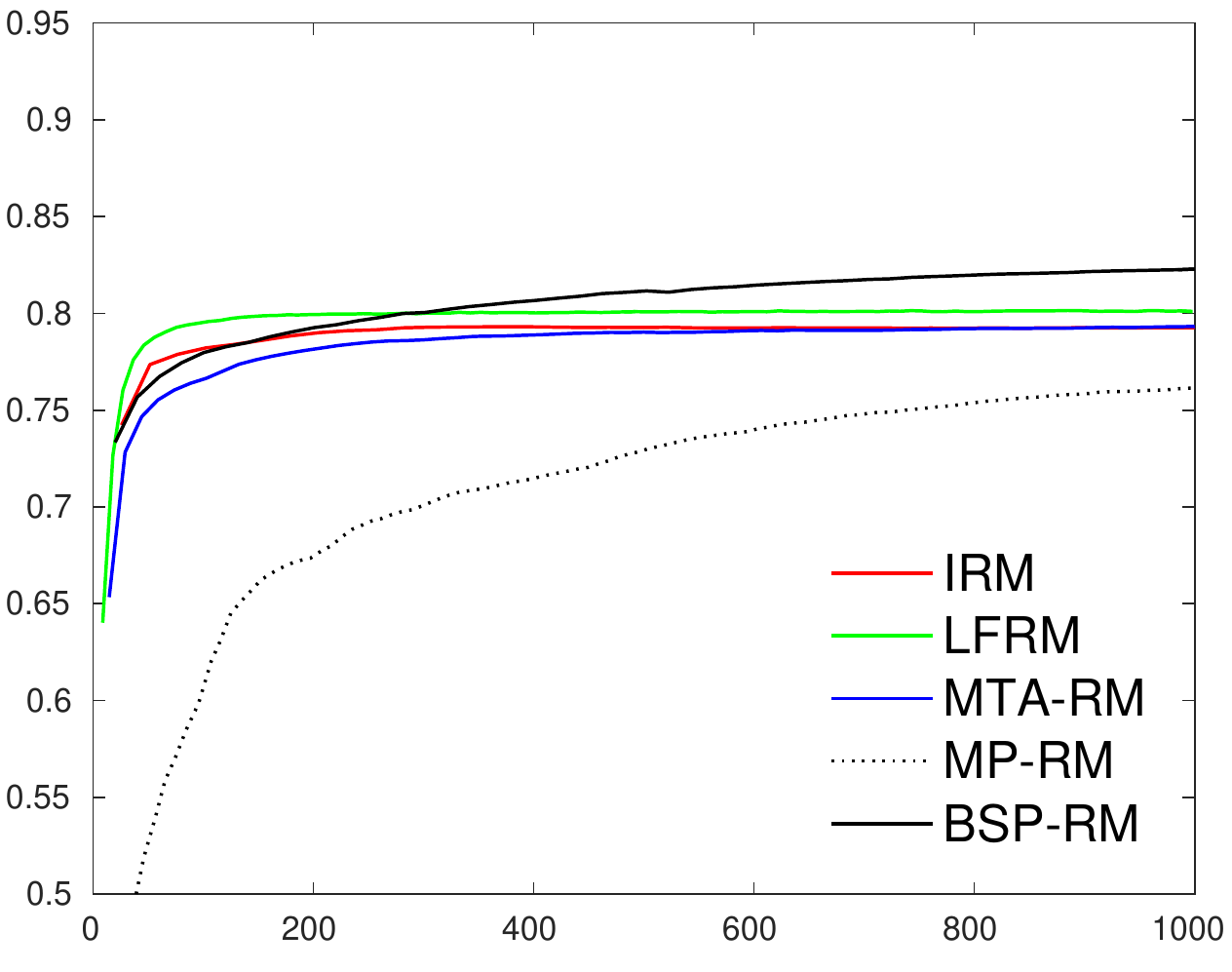}
  \includegraphics[width = 0.17 \textwidth, viewport = 115 275 479 562, clip]{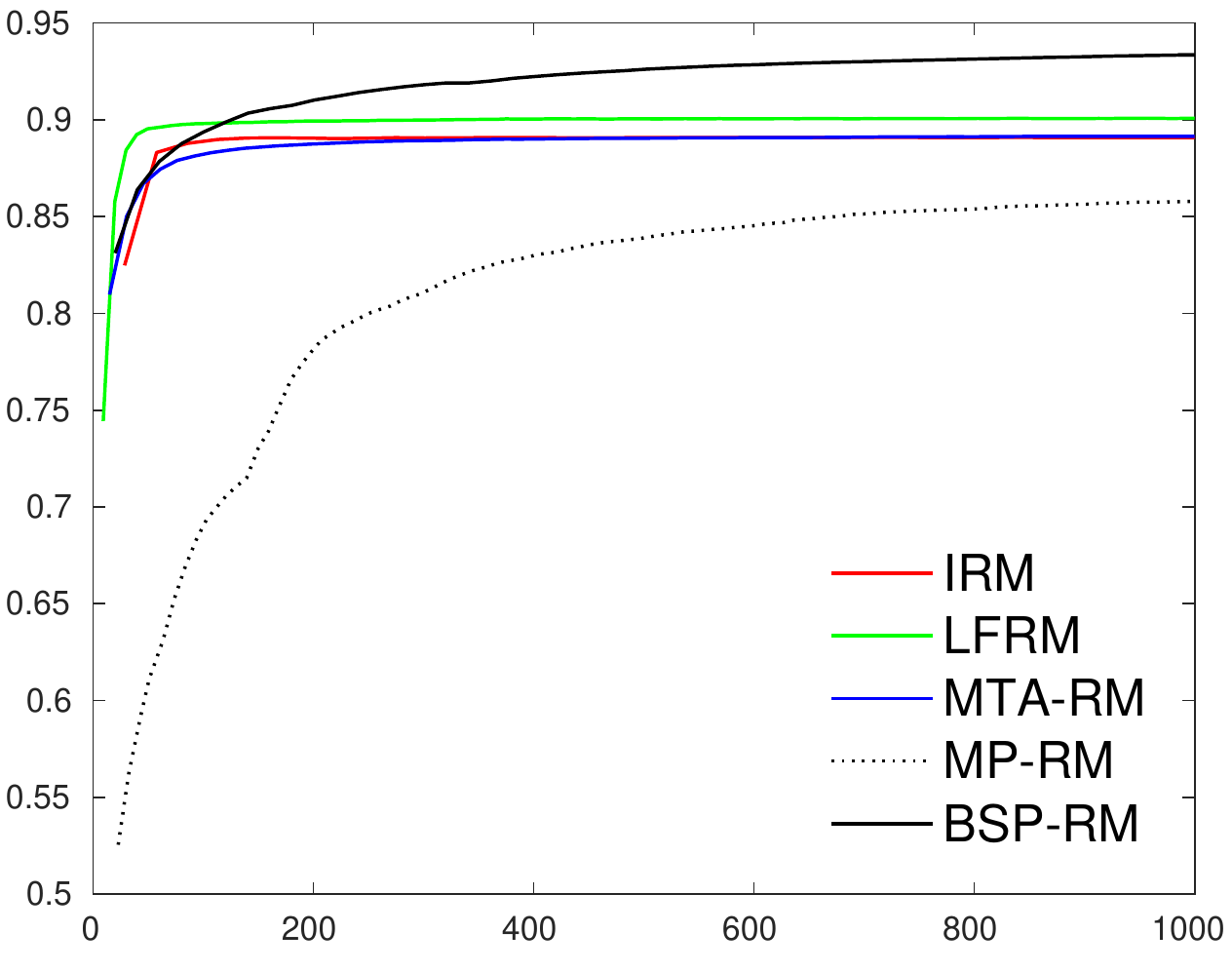}
  \includegraphics[width = 0.17 \textwidth, viewport = 115 275 479 562, clip]{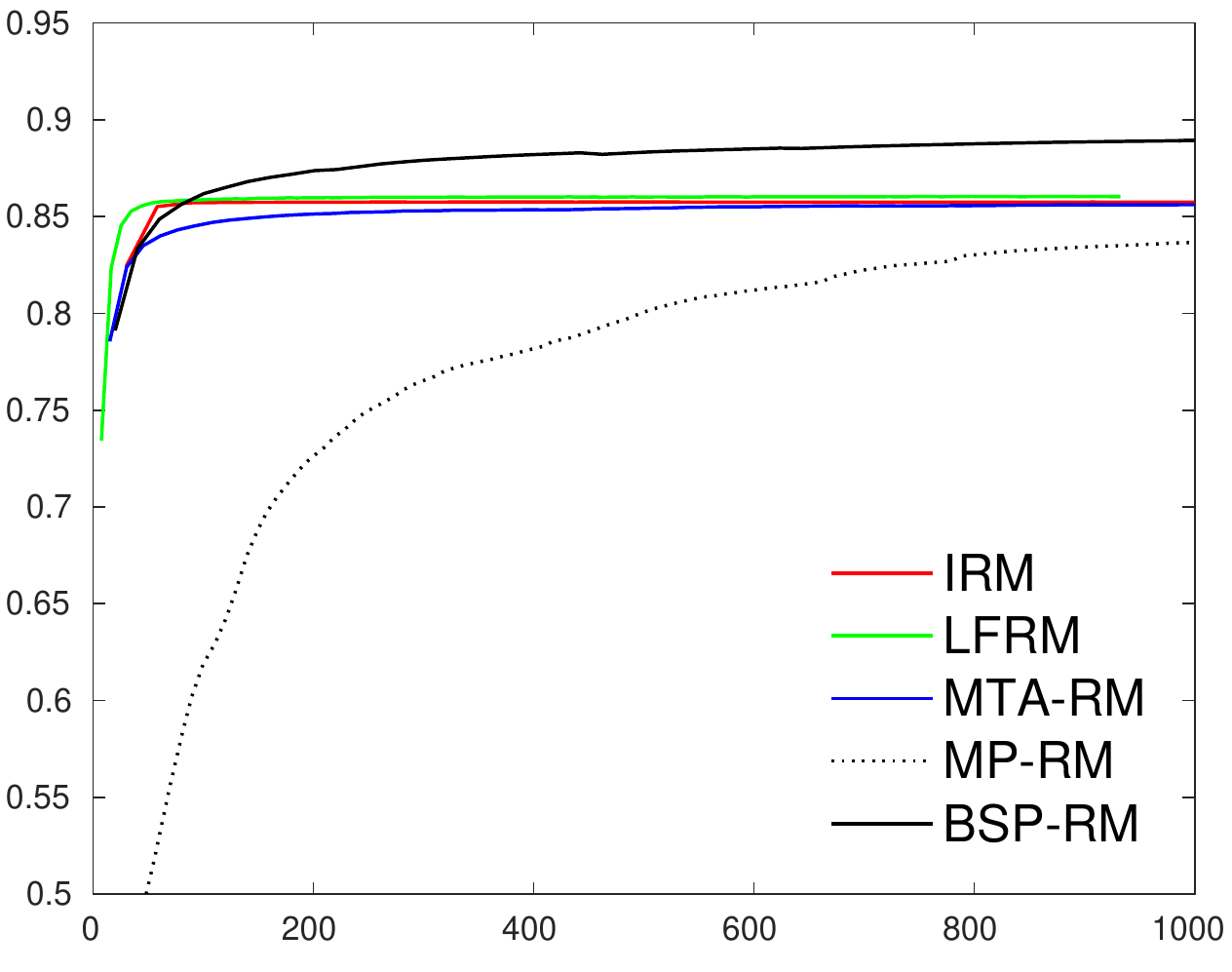}
  \includegraphics[width = 0.17 \textwidth, viewport = 115 275 479 562, clip]{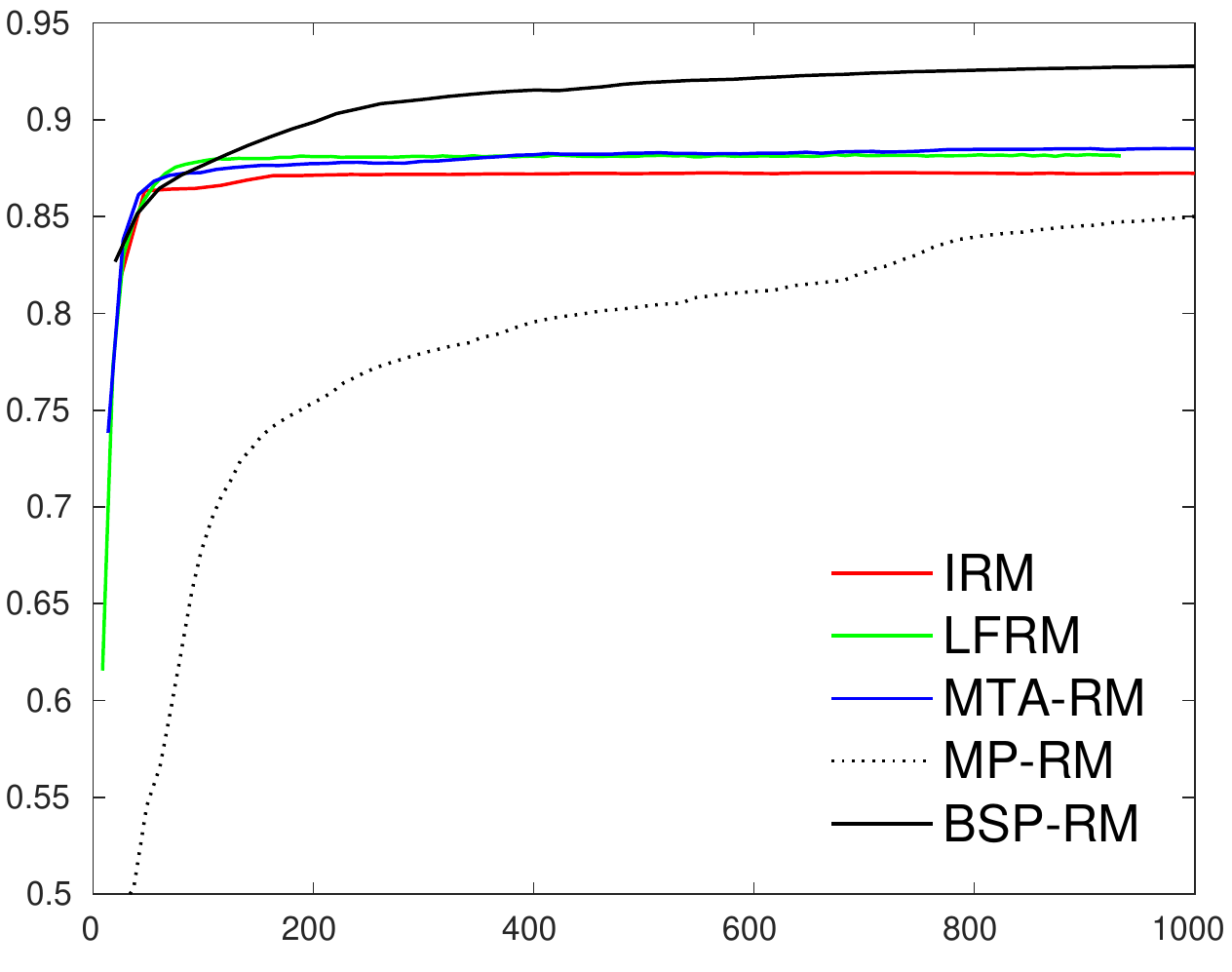}
  \includegraphics[width = 0.17 \textwidth, viewport = 115 275 479 562, clip]{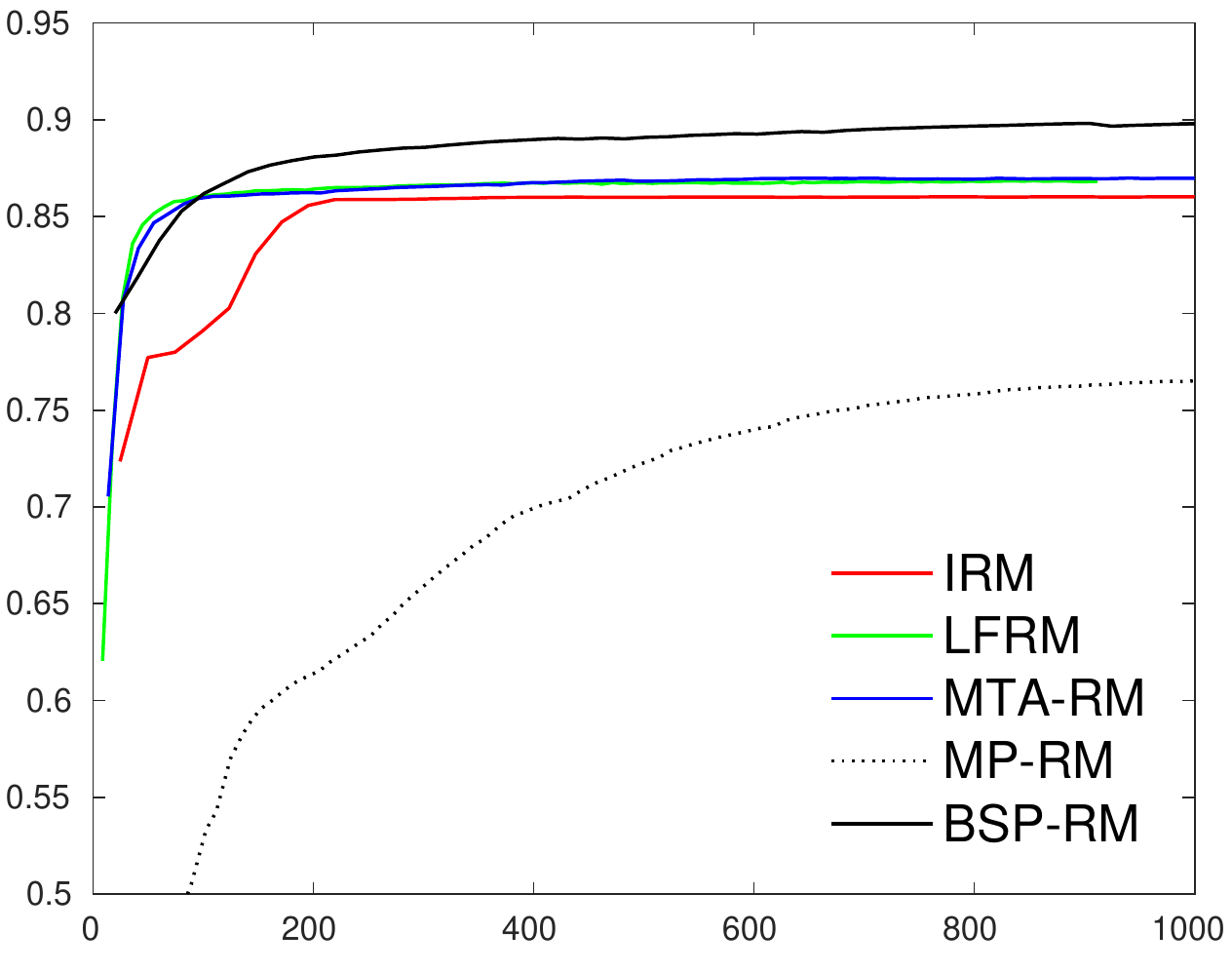}
  \caption{Partition structure visualization and performance comparison on the five data sets: (from left to right) Digg, Flickr, Gplus, Facebook and Twitter. The rows correspond to (from top to bottom) (1) IRM, (2) LFRM (refer to trained densities for each entry in the relational data), (3) MP-RM, (4) MTA-RM, (5) BSP-RM, (6) training log-likelihood vs.~wall-clock time (s) and (7) testing AUC vs.~wall-clock time (s). In BSP-RM, the colors behind the data points refer to the blocks and the cut lines are formed as curved lines rather than straight lines, since we are ranking the data points based on their coordinates and displaying this permutated relational matrix.}
\label{PartiionGraph}
\end{figure*}

\bibliography{StochasticPartitionProcessBase}
\bibliographystyle{plain}

\appendix

\section{Justification on the accumulated cut cost}
Let $D_{\Box}=\sup_{x_1, x_2\in\Box}\{|x_1-x_2|\}$ denotes the diameter of the polygon $\Box$. $\forall \theta$, it is obviously that the length of the newly generated cut line $L(\theta, \pmb{u})$ is smaller or equal to $D_{\Box}$, i.e., $|L(\theta, \pmb{u})|\le D_{\Box}$. Thus, we have the result for the sum of perimeters in the $l$-th partitioning result as:
\begin{align}
\sum_{k=1}^{l}PE(\Box_{\tau_l}^{(k)})\le PE(\Box)+2(l-1)D_{\Box}    
\end{align}

According to the Fatou's lemma, we get
\begin{align}
&\mathbb{E}\liminf_{l\to\infty}\frac{\sum_{k=1}^{l}PE(\Box_{\tau_l}^{(k)})}{l}\nonumber\\ 
\le& \liminf_{l\to\infty}\mathbb{E}\frac{\sum_{k=1}^{l}PE(\Box_{\tau_l}^{(k)})}{l}\nonumber\\
\le& \liminf_{l\to\infty}\frac{\mathbb{E}[PE(\Box)]+2(l-1)D_{\Box}}{l}\nonumber\\
< & \infty
\end{align}
which leads to $\liminf_{l\to\infty}\frac{\sum_{k=1}^{l}PE(\Box_{\tau_l}^{(k)})}{l}<\infty$ almost surely.
Since $\sum_{k=1}^{l}PE(\Box_{\tau_l}^{(k)})$ is increasing for $l$ almost surely, we get $\sum_{l=1}^{\infty}\left[\sum_{k=1}^{l}PE(\Box_{\tau_l}^{(k)})\right]^{-1}=\infty$ almost surely.

\section{Mathematical formulation of the three-restrictions on the measure invariance}
    \begin{enumerate}
    \item {translation $t$}: $\lambda_{\Box}(C_{\Box}^{\theta}) = \lambda_{t_{\pmb{v}}\Box}\circ t_{\pmb{v}}(C_{\Box}^{\theta})$, where $t_{\pmb{v}}(\pmb{x}) = \pmb{x}+\pmb{v}, \forall \pmb{v}\in \mathbb{R}^2$; 
    \item {rotation $r$}: $\lambda_{\Box} = \lambda_{r_{\theta'}\Box}\circ r_{\theta'}$, where $r_{\theta'}(\pmb{x})=\left[\begin{array}{cc}
        \cos\theta' & -\sin\theta' \\
        \sin\theta' & \cos\theta'
        \end{array}\right]\cdot\pmb{x}$ refers to rotate the point $\pmb{x}$ in an angle of $\theta'$; 
    \item {restriction $\psi$}: $\lambda_{\Box}(C_{\triangle}^{\theta}) = \lambda_{\psi_{\triangle}\Box}\circ\psi_{\triangle}(C_{\triangle}^{\theta})$, where $\triangle\subseteq \Box$ refers to a sub-domain of $\Box$; $\psi_\triangle\Box=\{\pmb{x}|\pmb{x}\in \triangle\subset\Box\}$, and {$C_{\triangle}^{\theta}$ refers to the set of cut lines for all the potential cuts crossing through $\triangle$}. 
    \end{enumerate}

\section{Proof of Proposition 1}
    \begin{lemma} \label{intersection_convex_polygons}
    Assume two convex polygons $\Box_1$ and $\Box_2$ have the same length on the line segment $\pmb{l}(\theta)$. There exists a set of countable divisions passing through $\pmb{l}(\theta)$ in the direction of $\theta+\frac{\pi}{2}$. Each line sub-segment of $\pmb{l}(\theta)$ cut by consecutive divisions is covered by the intersection of the $\Box_1, \Box_2$, while $\Box_1, \Box_2$ can move in the direction of $\theta+\pi/2$.
    \end{lemma}
    \begin{proof}
    The set of divisions (all in the direction of $\theta+\frac{\pi}{2}$) can be designed into two stages.
    \begin{figure} \label{measure_scaled_to_ratio}
    \centering
    \includegraphics[width =  0.45\textwidth]{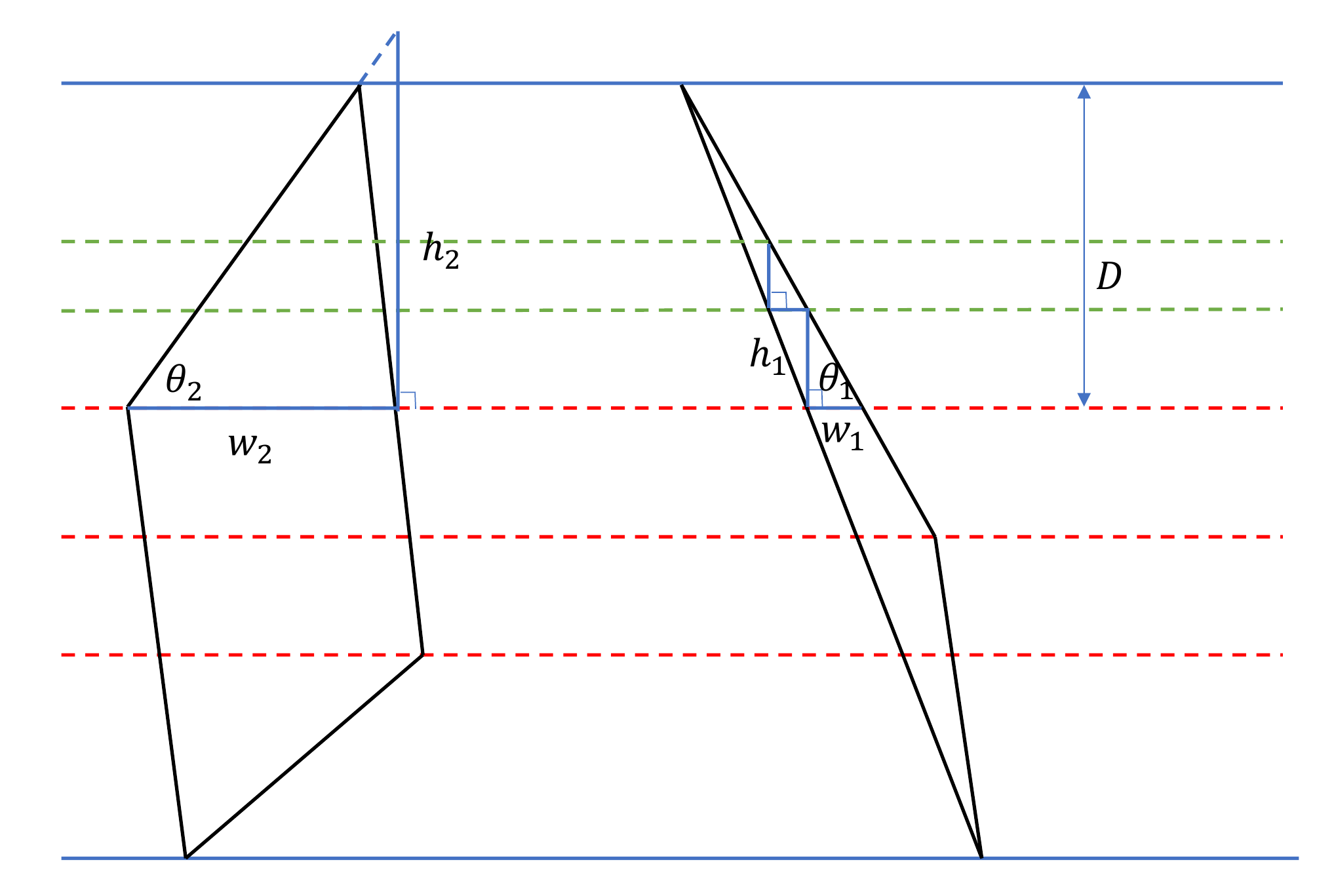}
    \caption{The design of a set of divisions in Lemma 1. Red dashed lines denote the ``rough'' divisions, while green dashed lines denote the ``grained'' divisions based on each consecutive rough divisions.}
    \label{polygon_division}
    \end{figure}

    \textbf{Stage $1$}, a set of ``rough'' divisions that pass through each vertices of the two polygons. 

    \textbf{Stage $2$}, sets of ``grained'' divisions based on the each consecutive rough divisions. Let $D$ denotes the distance between two selected consecutive divisions, $w_1$ ($w_2$) denotes the maximum width of polygon $\Box_1$ ($\Box_2$) between these two rough divisions and $\theta_1$ ($\theta_2$) is the smallest angle between the edge of polygon $\Box_1$ ($\Box_2$) and division. 
    
    We proceed the grained division in the following way. In the case of $\min\{w_1\tan\theta_1, w_2\tan\theta_2\}\ge D$, there is no need to do further grained division; otherwise, let $(\theta^*, d^*) = \arg_{\theta, d}\min \{d_1\cdot\tan\theta_1,d_2\cdot\tan\theta_2\}$, the first grained division is placed in the position of $d^*\tan\theta^*$ (see $h_1$ in Figure~\ref{polygon_division}). The second grained division would design based on the first one and proceed in a similar way. 
    
    Given the condition of $d\cdot\tan\theta<D$, we can do the partition at the positions of:
    \begin{equation}
    \left\{d\cdot\tan\theta\cdot\left(1-\frac{d\cdot\tan\theta}{D}\right)^l\right\}_{l=0}^{\infty}
    \end{equation}
    where $\sum_{l=0}^{\infty}d\cdot\tan\theta\cdot\left(1-\frac{d\cdot\tan\theta}{D}\right)^l=D$.

    \end{proof}

    \begin{proposition} \label{measure_definition_proof}
    The family of partition probability measure $\lambda_{\Box}(C_{\Box}^{\theta})$ keeps invariant under the operations of translation, rotation and restriction if and only if we have a constant $C$ such that $\lambda_{\Box}(C_{\Box}^{\theta})=C\cdot |\pmb{l}(\theta)|, \forall C\in \mathbb{R}^+$.
    \end{proposition}
    \begin{proof}
    The reverse case is fully discussed as in the main part of the paper. 

    On the other hand, assume we have two sets of cut lines $C_{\Box_1}^{\theta_1}, C_{\Box_2}^{\theta}$ with $|\pmb{l}_{\Box_1}(\theta_1)|=|\pmb{l}_{\Box_2}(\theta)|$.

    Given that the measure $\lambda_{\Box}(C_{\Box}^{\theta})$ is invariant under the operations of translation, rotation and restriction, we need to prove the following equity:
    \begin{equation} \label{measure_equ}
    \lambda_{\Box_1}(C_{\Box_1}^{\theta_1}) = \lambda_{\Box_2}(C_{\Box_2}^{\theta})
    \end{equation}

    To complete this, we first do rotation and translation operations on $\Box_1$, which is $\Box'_1: = r_{\theta'}\circ t_{\pmb{v}'}\circ\Box_1$, in a way that $\Box'_1$ and $\Box_2$ project into the same image $\pmb{l}(\theta)$. 
    
    Based on Lemma \ref{intersection_convex_polygons}, we divide $\Box'_1$ and $\Box_2$ into countable parts, where the intersection of these pair parts projects to the same images. That is:
    \begin{equation}
    \Box'_1 = \cup_{k}\Box^{',(k)}_1, \Box_2 = \cup_{k}\Box^{(k)}_2
    \end{equation}
    \begin{equation}
    Y^{(k)} = \Box^{',(k)}_1\cap\Box^{(k)}_2, \pmb{l}^{(k)}(\theta)\in Y^{(k)},\forall k\in N
    \end{equation}
    {\small
    \begin{equation}
    \begin{split}
    & C_{\Box^{',(k)}_1}^{\theta} = \{L(\theta, \pmb{u}) \mbox{ crossing } \Box^{',(k)}_1|\theta \mbox{ is fixed}, \pmb{u} \mbox{ lies on } \pmb{l}^{(k)}(\theta)\} \\
    & C_{\Box^{(k)}_2}^{\theta} = \{L(\theta, \pmb{u}) \mbox{ crossing } \Box^{(k)}_2|\theta \mbox{ is fixed}, , \pmb{u} \mbox{ lies on } \pmb{l}^{(k)}(\theta)\} \\
    \end{split}
    \end{equation}}

    The additivity of measures indicates that:
    \begin{align}
    \lambda_{\Box_1}(C_{\Box_1}^{\theta}) &= \sum_k\lambda_{\Box^{',(k)}_1}(C_{\Box^{',(k)}_1}^{\theta})\nonumber\\
    \lambda_{\Box_2}(C_{\Box_2}^{\theta}) &= \sum_k\lambda_{\Box_2^{(k)}}(C_{\Box^{(k)}_2}^{\theta})\nonumber
    \end{align}

    Eq. (\ref{measure_equ}) is correct if we can prove the following
    \begin{equation}
    \lambda_{\Box^{',(k)}_1}(C_{\Box^{',(k)}_1}^{\theta}) = \lambda_{\Box^{(k)}_2}(C_{\Box^{(k)}_2}^{\theta}),\forall k\in N
    \end{equation}

    From the measure invariance under rotation and translation, we get $\lambda_{\Box_1}(C_{\Box_1}^{\theta}) = \lambda_{t_{\pmb{v}}(r_{\theta'}({\Box_1}))}(t_{\pmb{v}}(r_{\theta'}(C_{\Box_1}^{\theta})))$.

    We also get
    \begin{equation} \label{eq_111}
    \Pi_{Y^{(k)}}\pi = \{Y^{(k)}\} = \Pi_{Y^{(k)}}t^{(k)}(\rho\pi)
    \end{equation}
    \begin{equation} \label{eq_112}
    \Pi_{Y^{(k)}}C_{\Box^{',(k)}_1}^{\theta} = \Pi_{Y^{(k)}}C_{\Box^{(k)}_2}^{\theta}
    \end{equation}
    Thus, we get
    \begin{equation}
    \begin{split}
     \lambda_{t_{\pmb{v}}r_{\theta'}\Box^{',(k)}_1}(C_{\Box^{',(k)}_1}^{\theta}) \overset{restriction}{=} & \lambda_{\Pi_{Y^{(k)}}t_{\pmb{v}}r_{\theta'}\Box^{',(k)}_1}(\Pi_{Y^{(k)}}C_{\Box^{',(k)}_1}^{\theta})\\
    \overset{Eq. (\ref{eq_111})}{=} & \lambda_{\Pi_{Y^{(k)}}\Box^{',(k)}_1}(\Pi_{Y^{(k)}}C_{\Box^{',(k)}_1}^{\theta})\\
    \overset{Eq. (\ref{eq_112})}{=} & \lambda_{\Pi_{Y^{(k)}}{\Box^{(k)}_2}}(\Pi_{Y^{(k)}}C_{\Box^{(k)}_2}^{\theta})\\
    \overset{restriction}{=} & \lambda_{\Box^{(k)}_2}(C_{\Box^{(k)}_2}^{\theta})\\
    \end{split}
    \end{equation}
    \end{proof}

\section{Proof of Proposition 2}

Our partition would result in convex polygon. We have the integration results for the convex polygon.
\begin{lemma}
  The integration of the intersection line in a triangle $\triangle$ over $[0, \pi]$ equals to the triangle's perimeter.
\end{lemma}
\begin{proof}
We first consider the acute triangle (Top row of {\it Figure \ref{acutetriangle}}) case.
\begin{figure}[h]
\centering
\includegraphics[width = 0.3 \textwidth]{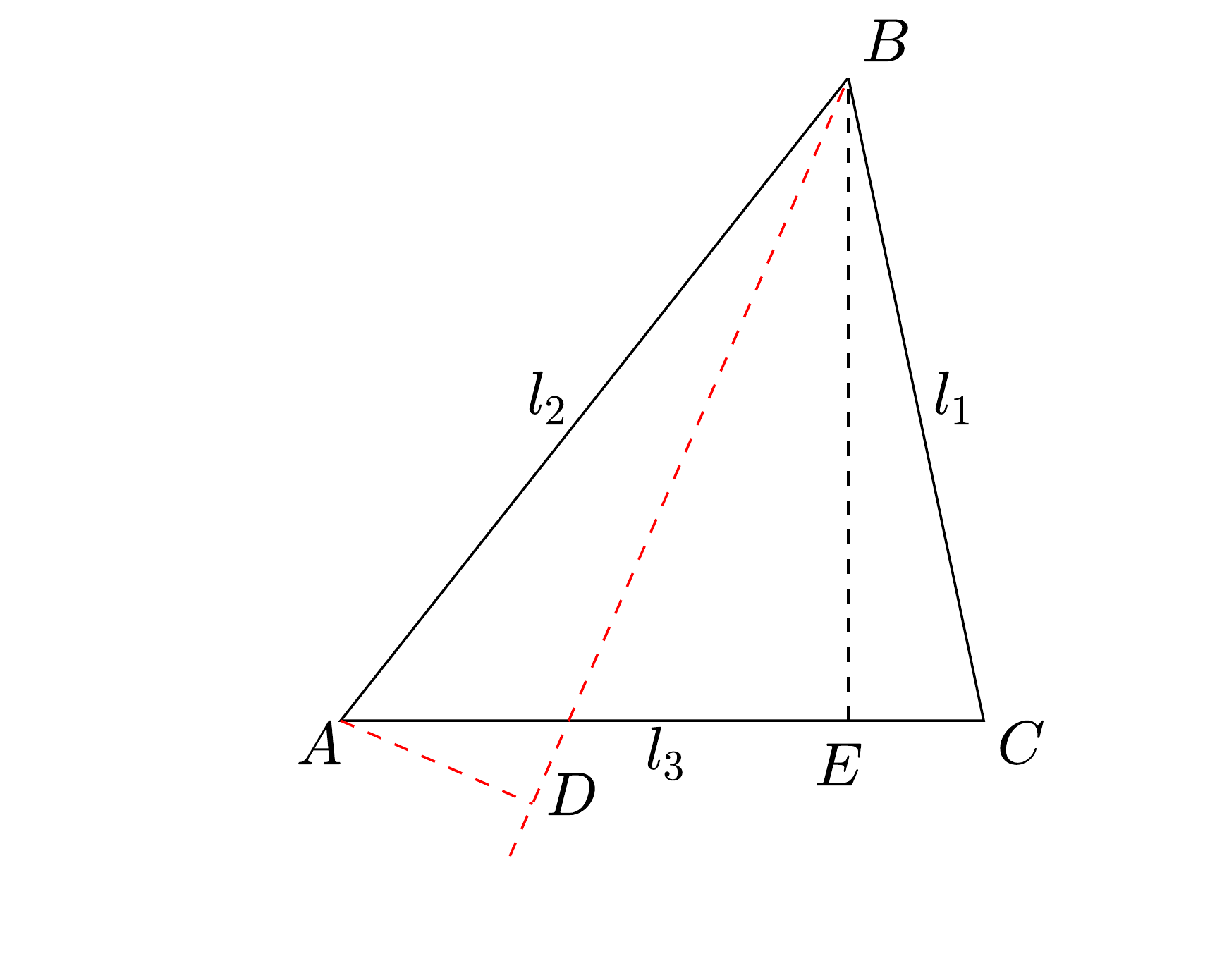}
\includegraphics[width = 0.3 \textwidth]{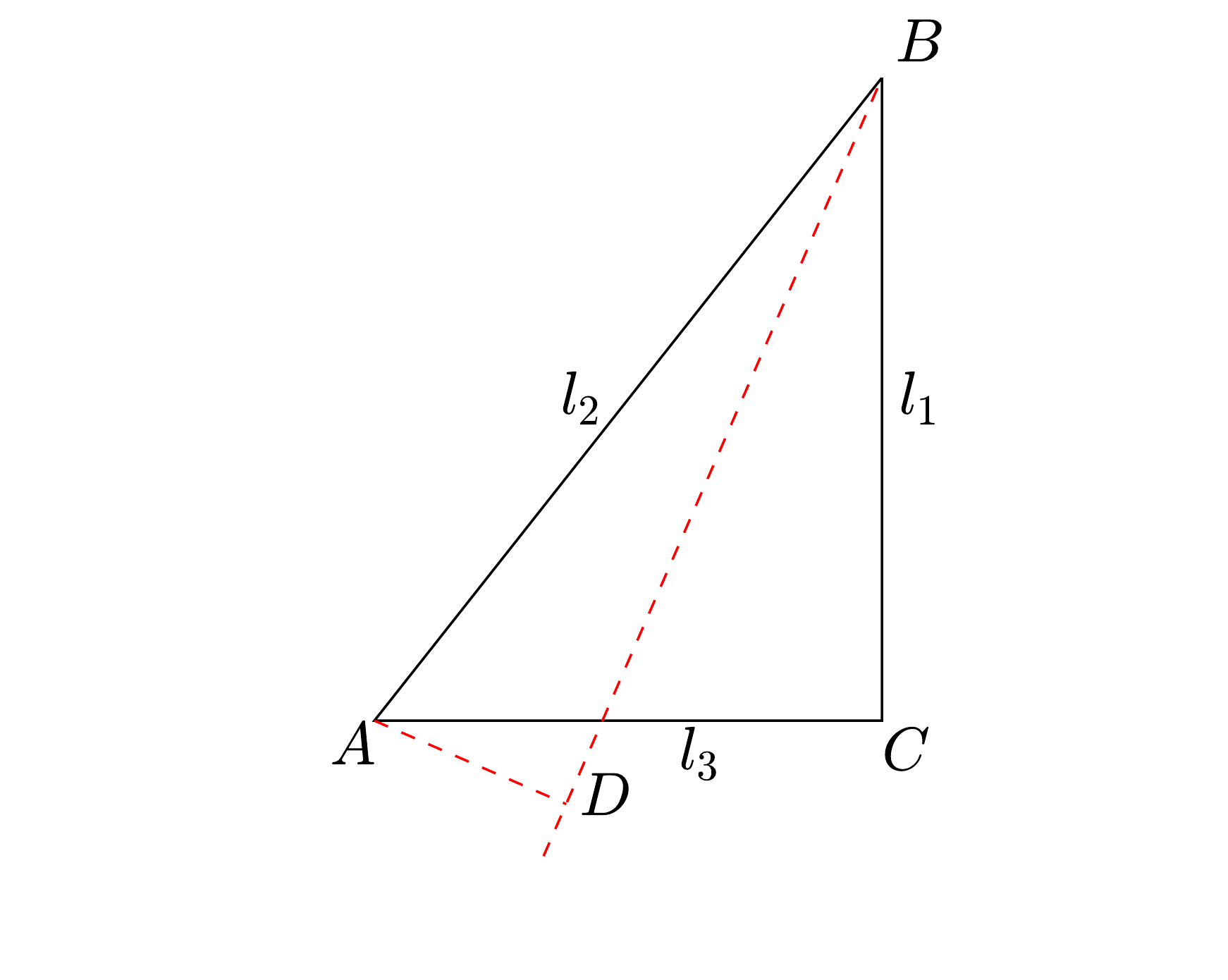}
\includegraphics[width = 0.3 \textwidth]{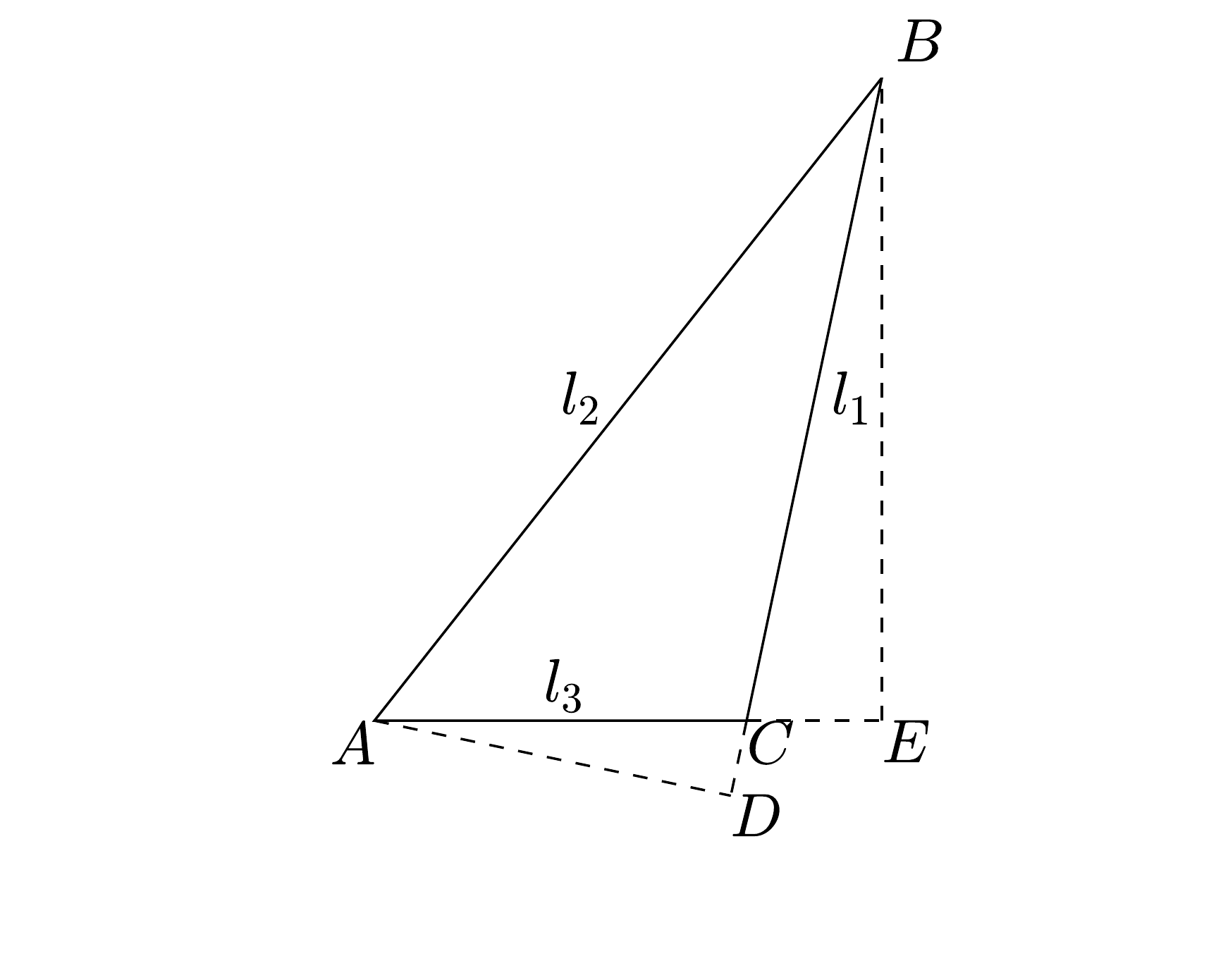}
\caption{Top: Acute Triangle; Middle: Right Triangle; Bottom: Obtuse Triangle.}
\label{acutetriangle}
\end{figure}

Let $\{l_1, l_2, l_3\}$ being the lengths of the triangle $\triangle$'s edges and $\{\angle BAC, \angle ABC, \angle ACB\}$ being the corresponding angles. According to the law of sines, we have
\begin{equation}
l_0=\frac{l_1}{\sin \angle BAC} = \frac{l_2}{\sin \angle ABC} = \frac{l_3}{\sin \angle ACB}
\end{equation}
where we use $l_0$ to denote the ratio between the length and its corresponding angle. 

W.l.o.g., we are cutting the block in the direction within $\angle ABC$. The projection scalar of $l_2$ is calculated as $|BD|=l_2\cos\theta(\theta=\angle ABD)$. While $\theta$ is ranging from $0$ to $\angle ABE$, the integration of $BD$ is
\begin{equation}
\begin{split}
\int_{0}^{\angle ABE} |BD|d\theta = & \int_{0}^{\angle ABE} l_2\cos\theta d\theta\\
=& l_2\sin\theta|_{0}^{\angle ABE} = l_2\cos(\angle BAC)
\end{split}
\end{equation}
By using the similar routines, we can get the integration of all the projection lines $I$ as:
\begin{equation}
\begin{split}
I =& l_2\cos\angle BAC+l_2\cos\angle ABC+l_1\cos\angle ACB\\
& +l_1\cos\angle ABC+l_3\cos\angle BAC+l_3\cos\angle ACB\\
= & l_0\sin\angle ACB\cos\angle BAC+l_0\sin\angle ACB\cos\angle ABC\\
& +l_0\sin\angle BAC\cos\angle ACB+l_0\sin\angle BAC\cos\angle ABC\\
&+l_0\sin\angle ABC\cos\angle BAC+l_0\sin\angle ABC\cos\angle BAC\\
= & l_0\sin(\angle ACB+\angle ABC)+l_0\sin(\angle BAC+\angle ACB)\\
&+l_0\sin(\angle ABC+\angle BAC)\\
= & l_0\sin\angle BAC+l_0\sin\angle ABC+l_0\sin\angle ACB\\
= & l_1+l_2+l_3=PE(\triangle)
\end{split}
\end{equation}
Here the $2^{\mbox{nd}}$ equation holds due to the law of Sines.

The case of right triangle (Middle row of {\it Figure \ref{acutetriangle}}) is straight forward.

We can get
\begin{equation}
\begin{split}
I =& l_2\cos\angle BAC+l_2\cos\angle ABC\\
& +l_1\cos\angle ABC+l_3\cos\angle BAC\\
= & l_1+l_2+l_3=PE(\triangle)
\end{split}
\end{equation}

On the case of obtuse triangle (Bottom row of {\it Figure \ref{acutetriangle}})
\begin{equation}
\begin{split}
I =& l_2\cos\angle BAD-l_3\cos\angle CAD +l_2\cos\angle ABE\\
& -l_1\cos\angle CBE +l_3\cos\angle BAC+l_1\cos\angle ABC\\
= & l_1+l_2+l_3=PE(\triangle)
\end{split}
\end{equation}

\end{proof}

\begin{lemma} \label{perimeter_of_convex_polygon}
  The integration of the length of the block's projected image in the direction of $\theta$ over $(0, \pi]$ equals to the perimeter of the block, which is $\int_{0}^{\pi}|\pmb{l}(\theta)|d\theta = PE(\Box)$. 
\end{lemma}
\begin{figure}[h]
\centering
\includegraphics[width = 0.4 \textwidth]{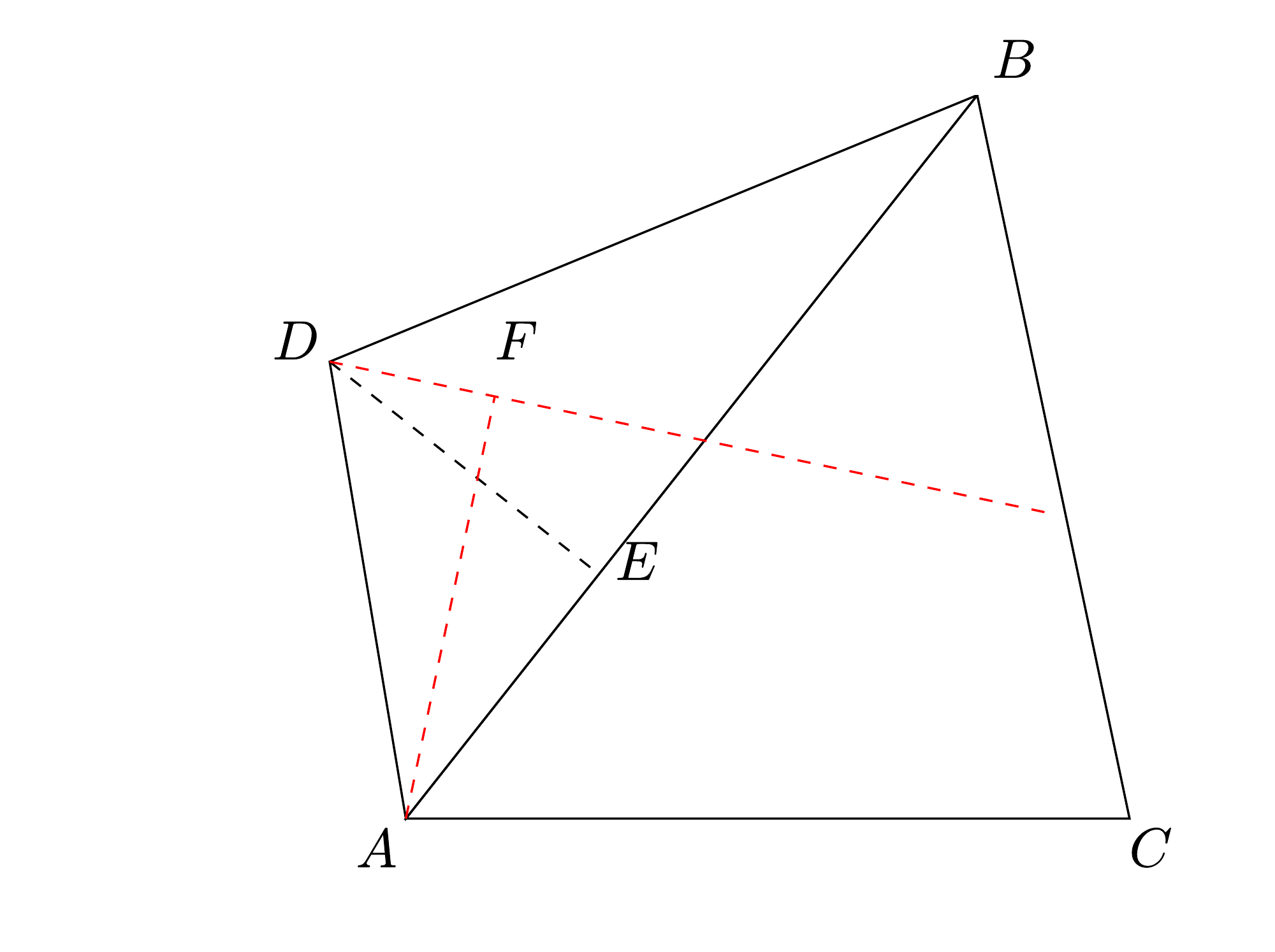}
\caption{From convex polygon with $n-1$ vertices to convex polygon with $n$ vertices.}
\label{quandrangle}
\end{figure}
\begin{proof}
Convex polygon with $n$ vertices can be divided into $n-2$ triangles. Since we have the result for the case of triangles, mathematical induction is used to get the conclusion for any convex polygons. 

Assume we have the result for convex polygon with $n-1$ vertices, the additional part for its transformation to convex polygon with $n$ vertices is the triangle $\triangle ABD$. Correspondingly, the increase in the scalar projection is composed of two parts:
\begin{align}
L_{\textrm{increase}}^1 = & \int_{0}^{\angle DAB} |AD|\sin\theta d\theta \nonumber \\
= & |AD| - |AD|\cos\angle DAB = |AD|-|AE|
\end{align}
where $\theta=\angle DAF$ and $L_{\textrm{increase}}^1$ refers to the integration of $|DF|$ in the angle of $\angle DAB$.
\begin{align}
L_{\textrm{increase}}^2 = & \int_{0}^{\angle ABD} |BD|\sin\theta d\theta \nonumber \\
= & |BD| - |BD|\cos\angle ABD = |BD|-|BE|
\end{align}

Thus, the total add amount is $L_{\textrm{increase}} = |BD|+|AD|-(|AE|+|BE|) = |BD|+|AD|-|AB|$. This is exactly the increase of perimeter from the convex polygon with $n-1$ vertices to convex polygon with $n$ vertices. Thus, we can get the result for all the convex polygons.
\end{proof}

Proposition 2 is a direct result of Lemma \ref{perimeter_of_convex_polygon}.

\section{Consistency}
    Some notations are firstly defined for convenient reference. We use $\Box$ and $\triangle$ to denote a domain and its subdomain, which is $\triangle\subseteq \Box$. $M_{\tau}$ and $N_{\tau}$ are individually defined as the BSP-Tree processes on $\Box$ and $\triangle$ respectively. The restriction is denoted as $\Pi$, in which we have $\Pi_{\triangle} M_{\tau}=N_{\tau}$. Also, we let $c(\Box)$ denote the measure over the block $\Box$, which is $c(\Box) = \int_{\theta}\omega(\theta)|\pmb{l}_{\Box}(\theta)|d\theta$ and let $\mathcal{O}_n^{\Box}$ denote the partition after $n$-th cut on the convex polygon $\Box$. 

    \paragraph{Extending partition from $\triangle$ to $\Box$} For the BSP-Tree process $N_{\tau}$, we let $Z$ and $\{\sigma_l\}_{l\in\mathbb{N}}$ denotes the related Markov chain and the corresponding time stops. For $t\ge 0$, define $m_t$ to be the index such that $t\in[\sigma_{m_t}, \sigma_{m_{t+1}}], N_t=Z_{m_t}$.
    
    To extend $N_{\tau}$ from $\triangle$ to $\Box$, let $\tau_0=0$ and $Y_0=\Box$. For $n\in\mathbb{N}$, we define $\tau_{n+1}$ and $Y_{n+1}$ inductively as:
    \begin{equation} \label{tau_n1_definition}
    \tau_{n+1}:=\min\{\sigma_{m_{\tau_n}+1}, \tau_n+\frac{\xi_n}{c(Y_n)-c(Z_{m_{\tau_n}})}\}
    \end{equation}
    where $\xi_n$ is generated from the exponential distribution with mean $1$.
    
    \begin{equation} \label{Y_n1_definition}
    Y_{n+1}=\left\{\begin{array}{ll}
    lift_{Y_n, \triangle}(Z_{m_{\tau_n}}), & \tau_{n+1}=\sigma_{m_{\tau_n}+1};\\
    gencut_{\triangle}(Y_n), & \mbox{otherwise}. \\
    \end{array}\right.
    \end{equation}

    where $lift_{Y_n, \triangle}(Z_{m_{\tau_n}})$ denotes extending the existing cut to the larger domain $\Box$ and $gencut_{\triangle}(Y_n)$ refers to the case that there will be a new cut generated in $\Box$ that does not cross into $\triangle$. 
    
    According to the results of Proposition V.16 of chapter VI in \cite{roy2011thesis},  
    the defined process $\{Y_n, \tau_n\}$ are well-defined. 

    % That is to say, we need to show that
    % \begin{equation}
    % \Pi_A(Y_n) = Z_{m_{\tau_n}}
    % \end{equation}

    % We use mathematical induction to infer the result.
    % Since $Y_0=X, \mathcal{O}^A_0=A$, according to the definition of the projection $\Pi_A$, we get $\Pi_A(Y_0) = \mathcal{O}^A_{0}=A$.

    % If the cut is outside $A$, we get $\Pi_A(Y_{n+1})=\Pi_A(Y_{n})=Z_{\tau}$.

    % If the cut crosses into $A$, by the mathematical induction, we get $\Pi_A(Y_n) = Z_{m_{\tau_n}}\prec_1Z_{m_{\tau_n}+1}$. Based on the definition of $lift_{Y_n, A}$, we get $\Pi_A(Y_{n+1})=\Pi_A(lift_{Y_n, A}(Z_{m_{\tau_n}+1}))=Z_{m_{\tau_n}+1}=Z_{m_{\tau_{n+1}}}$.  Thus, the conclusion holds.

    \paragraph{Prove the correctness}
    $\forall t>0$, the waiting time for the next cut in $X$ is:
    \begin{equation}
    \zeta_t = \tau_{n+1}-t
    \end{equation}
    According to $\tau_{n+1}$'s definition (Eq. (\ref{tau_n1_definition})),  $\zeta_t$ follows the exponential distribution with the rate being $c(Y_{n_t})$. What is more, the probability of the event $\tau_{n+1}=\sigma_{m_{\tau_n}+1}$ occurs with probability $c(\mathcal{O}^{\triangle}_n)/c(\mathcal{O}^{\Box}_n)$.

    For the newly extended case $\{Y_n\}_n$, while $Y_n$ crosses through $\triangle$, the probability measure on $C_{\triangle}^{\theta}$ is in proportion to $\omega(\theta)|\pmb{l}_{\triangle}(\theta)|$ and $\pmb{u}$ locates only on $\pmb{l}_{\triangle}(\theta)$. Thus, we get
    \begin{align} \label{xtoacase1}
        P = & \frac{c(\mathcal{O}^{\triangle}_n)}{c(\mathcal{O}^{\Box}_n)}\cdot \frac{\omega(\theta)|\pmb{l}_{\triangle}(\theta)|}{\int_{\theta}\omega(\theta)|\pmb{l}_{\triangle}(\theta)|d\theta}\cdot\frac{1}{|\pmb{l}_{\triangle}(\theta)|}\nonumber \\
        = & \frac{\omega(\theta)}{c(\mathcal{O}^{\Box}_n)}
    \end{align}
    while $Y_n$ does not cross through $\triangle$, , the probability measure on $C_{\Box\backslash\triangle}^{\theta}$ is in proportion to $\omega(\theta)(|\pmb{l}_{\Box}(\theta)|-|\pmb{l}_{\triangle}(\theta)|)$ and $\pmb{u}$ locates only on $\pmb{l}_{\Box\backslash\triangle}(\theta)$ (with the length $|\pmb{l}_{\Box}(\theta)|-|\pmb{l}_{\triangle}(\theta)|$). Thus, we get
    \begin{align} \label{xtoacase2}
        P = & \frac{c(\mathcal{O}^{\Box}_n)-c(\mathcal{O}^{\triangle}_n)}{c(\mathcal{O}^{\Box}_n)}\cdot \frac{\omega(\theta)(|\pmb{l}_{\Box}(\theta)|-|\pmb{l}_{\triangle}(\theta)|)}{\int_{\theta}\omega(\theta)(|\pmb{l}_{\Box}(\theta)|-|\pmb{l}_{\triangle}(\theta)|)d\theta}\nonumber \\
        &\cdot\frac{1}{|\pmb{l}_{\Box}(\theta)|-|\pmb{l}_{\triangle}(\theta)|}\nonumber \\
        = & \frac{\omega(\theta)}{c(\mathcal{O}^{\Box}_n)}
    \end{align}
    
    Eq. (\ref{xtoacase1}) and Eq. (\ref{xtoacase2}) show that the probability measure of $Y_n$ equals to the one that directly generated in the domain of $\Box$. Therefore, the partition constructed by Eq. (\ref{tau_n1_definition}) and Eq. (\ref{Y_n1_definition}) is a realization of BSP-Tree process in $\Box$. 

    According the transfer theorem~(Theorem V.13 of chapter VI in \cite{roy2011thesis}), the partition distribution is consistent from $\Box$ to $\triangle$.

\section{MCMC for the BSP-RM}
    Algorithm \ref{mcmc_total} displays an MCMC solution for the BSP-RM.
\setcounter{algorithm}{1}
\begin{algorithm}[H]
    \caption{MCMC for BSP-RM} \label{mcmc_total}
    \begin{algorithmic}[1]
        \REQUIRE Training data $X$, Budget $\tau$, Number of particles $C$
        \ENSURE A realization of the BSP-Tree process; coordinates $\{(\xi_i, \eta_i)\}_{i=1}^n$ of $X$
        \STATE Initialize the partition and nodes' coordinates
        \FOR{$t = 1:T$}
            \STATE Use C-SMC algorithm to update the partition structure, according to Algorithm {\bf 1};
            \STATE Update nodes' coordinates $\{\xi_i, \eta_i\}_{i=1}^n$ according to Eq. (\ref{eq_coordinate}).
        \ENDFOR
    \end{algorithmic}
    \end{algorithm}
\subsection{Updating nodes' coordinates $(\xi_i, \eta_i)_{i=1}^n$}
    $(\xi_i, \eta_i)$'s updating is implemented through the Metropolis-Hastings algorithm. We propose the new values of $\xi_i, \eta_i$ with the uniform distribution in $[0, 1]$ and the acceptance ratio $\min(1, \alpha)$ is as follows:
    \begin{equation} \label{eq_coordinate}
    \begin{split}
    & \alpha(\xi_i, \xi_i^0) = \frac{\prod_{j'=1}^n P(e_{ij'}|\xi_i, \xi_{\backslash i}, \eta_{j'}, \theta)}{\prod_{j'=1}^n P(e_{ij'}|\xi_i^0, \xi_{\backslash i}, \eta_{j'}, \theta)};\\
    & \alpha(\eta_j, \eta_j^0) = \frac{\prod_{i'=1}^n P(e_{i'j}|\eta_j, \eta_{\backslash j}, \xi_{i'}, \theta)}{\prod_{i'=1}^n P(e_{i'j}|\eta_j^0, \eta_{\backslash j}, \xi_{i'}, \theta)}\\
    \end{split}
    \end{equation}
    
\section{Visualization of Case 1}
Figure 4 shows the visualization of Case 1.
    \begin{figure}[t]
    \centering
    \includegraphics[width =  0.12 \textwidth]{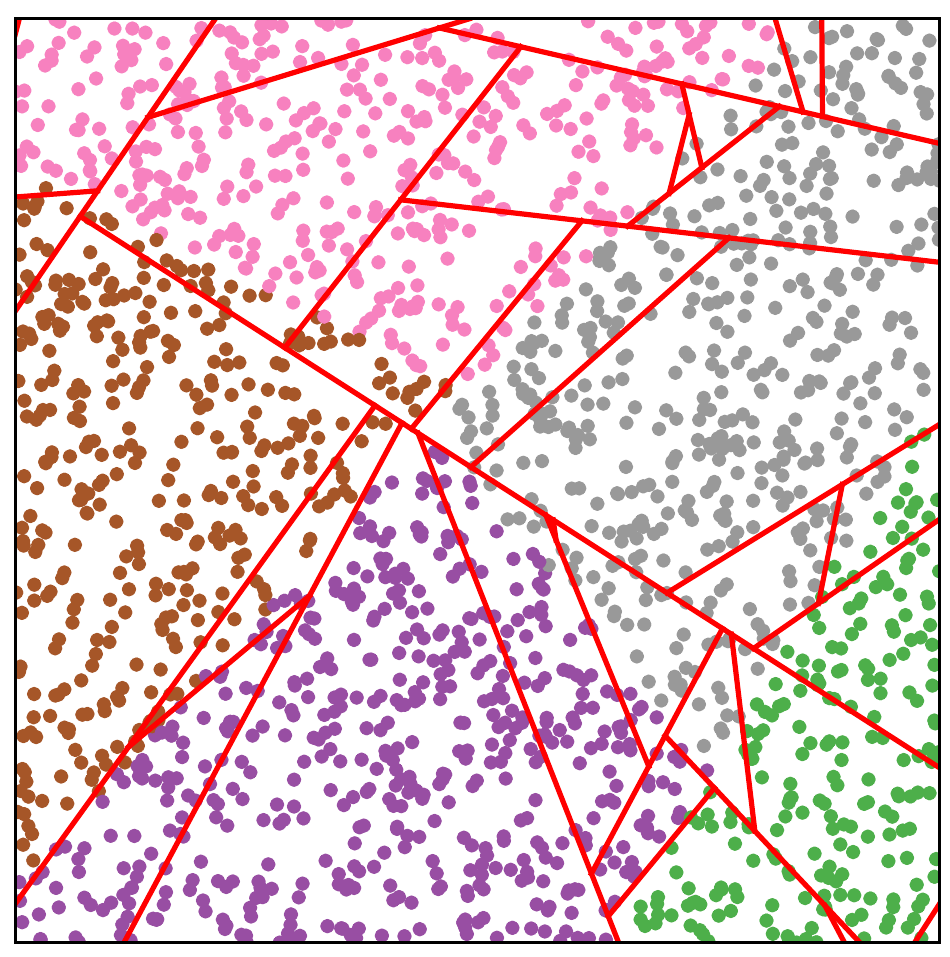}\qquad
    \includegraphics[width =  0.12 \textwidth]{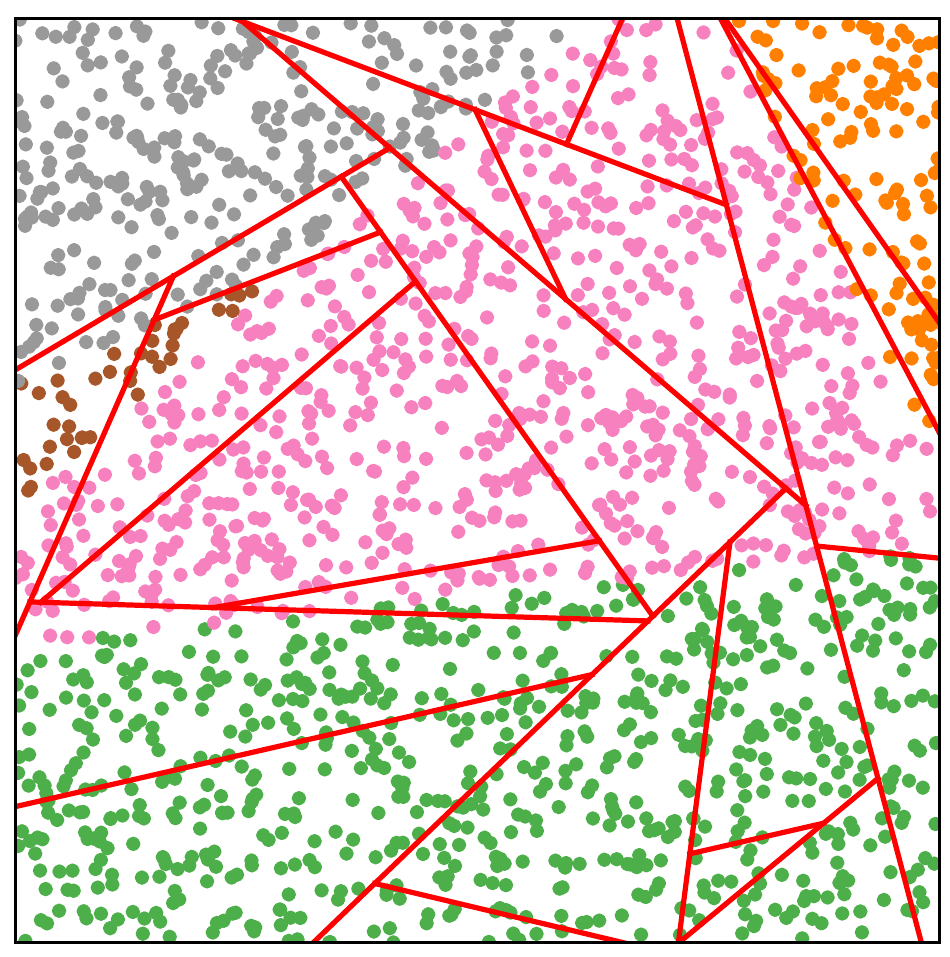}\qquad
    \caption{Toy Data Partition Visualization (Case 1).}
    \label{fig:toy_data_partition}
    \end{figure}

\end{document}